%% file: main.tex
\RequirePackage{fix-cm}
\documentclass[twocolumn]{svjour3}       
\smartqed  
\usepackage{graphicx}
\usepackage{xcolor}

\usepackage{amsmath}
\usepackage{lineno}
\usepackage{algorithm}
\usepackage{algpseudocode}
\usepackage{booktabs}    
\usepackage{float}       
\usepackage{array}       
\usepackage[table]{xcolor} 
\usepackage{xspace}
\usepackage{rotating}

\input{commands}
\input{acks}

\newcommand{\tool}{GEG\xspace}
\newcommand{\replpackage}{\footnote{\url{https://github.com/giordanoDaloisio/GEG}}}

\begin{document}




\title{A Generalised Exponentiated Gradient Approach to Enhance Fairness in Binary and Multi-class Classification Tasks}

\author{Maryam Boubekraoui \and Giordano d'Aloisio \and Antinisca Di Marco}

\institute{Maryam Boubekraoui \at
  Laboratoire d’Intelligence Artificielle et Systèmes Industriels (LIASI), HESTIM Engineering and Business School,
 Casablanca, Morocco
 \\
 Laboratory LAMAI, Faculty of Sciences and Technologies, Cadi
Ayyad University, Marrakech, Morocco 
\and Giordano d'Aloisio, Antinisca Di Marco \at 
 Department of Information Engineering, Computer Science and Mathematics, University of L'Aquila, L'Aquila, Italy
 }

\maketitle

\begin{abstract}

The widespread use of AI and ML models in sensitive areas raises significant concerns about fairness. While the research community has introduced various methods for bias mitigation in binary classification tasks, the issue remains underexplored in multi-class classification settings.
To address this limitation, in this paper, we first formulate the problem of fair learning in multi-class classification as a multi-objective problem between
effectiveness (i.e., prediction correctness) and multiple linear fairness constraints.
Next, we propose a Generalised Exponentiated Gradient (GEG) algorithm
to solve this task. \tool is an in-processing algorithm that enhances fairness in binary and multi-class classification settings under multiple fairness definitions. We conduct an extensive empirical evaluation of \tool against six baselines across seven multi-class and three binary datasets, using four widely adopted effectiveness metrics and three fairness definitions. \tool outperforms existing baselines, demonstrating that it can improve fairness without compromising the classifier's effectiveness.


\end{abstract}



\section{Introduction}

\input{sections/introduction}

\section{Background and Related Work}\label{sec:background}


\input{sections/rw}

\section{Methodology}\label{sec:methodology}

\input{sections/methodology}

\section{Evaluation}\label{sec:evaluation}

\input{sections/evaluation}

\section{Results}\label{sec:results}

\input{sections/results}

\section{Conclusion and Future Work}\label{sec:conclusion}

\input{sections/conclusion}

\begin{acknowledgements}
This work has been partially supported by \TAAck and by \SoBigDataITAck and by \FringeAck.
\end{acknowledgements}

\bibliographystyle{spmpsci} 
\bibliography{bibliography}

\clearpage

\appendix

\section{Derivation of Additional Linear Moment Conditions}

\subsection{Equalized Odds}\label{app:eq_odds}
For Equalized Odds, we impose that the classifier's 
prediction is independent of the sensitive attribute 
\( A \) conditionally on the true label \( Y = y \). 
This leads to defining separate moments for each group 
\( a \in \mathcal{A} \), each class $y_k \in \mathcal{Y}$,
and each true label \( y \in \mathcal{Y} \), of the form:
\[
\mu_{a,y}^{(k)}(h) = 
\mathbb{E}[\mathbf{1}_{\{h(X) = y_k\}} \mid A = a, Y = y],
\]
which represent the group-wise prediction rates for 
class $y_k$, conditioned on the true label $Y = y$. 
We also define the corresponding average moment across 
all groups:
\[
\mu_{*,y}^{(k)}(h) = 
\mathbb{E}[\mathbf{1}_{\{h(X) = y_k\}} \mid Y = y],
\]
which captures the overall prediction rate for class 
$y_k$ conditioned on the true label \( Y = y \). 
General Label Equalized Odds is satisfied when
\[
\mu_{a,y}^{(k)}(h) = \mu_{*,y}^{(k)}(h) 
\quad \forall a \in \mathcal{A},~ \forall y, y_k \in \mathcal{Y}.
\]
As before, each equality can be expressed as two 
inequalities:
\[
\mu_{a,y}^{(k)}(h) - \mu_{*,y}^{(k)}(h) \leq 0,
\]
\[
\mu_{*,y}^{(k)}(h) - \mu_{a,y}^{(k)}(h) \leq 0.
\]
In the binary sensitive attribute case \( A \in \{0,1\} \), 
we identify $A = 0$ as the \emph{unprivileged} group 
and $A = 1$ as the \emph{privileged} group. For each 
true label $y \in \mathcal{Y}$ and each class 
$y_k \in \mathcal{Y}$, we obtain:
\[
\mu_{0,y}^{(k)}(h) - \mu_{*,y}^{(k)}(h) \leq 0,
\]
\[
\mu_{*,y}^{(k)}(h) - \mu_{0,y}^{(k)}(h) \leq 0,
\]
\[
\mu_{1,y}^{(k)}(h) - \mu_{*,y}^{(k)}(h) \leq 0,
\]
\[
\mu_{*,y}^{(k)}(h) - \mu_{1,y}^{(k)}(h) \leq 0.
\]
To enforce General Label Equalized Odds across all 
classes and true labels simultaneously, we build a 
block-diagonal constraint system 
$M\,\mu(h) \leq \epsilon$, with 
$\epsilon = \mathbf{0} \in \mathbb{R}^{4|\mathcal{Y}|^2}$,
\[
\mu(h) =
\begin{bmatrix}
\mu_{0,y_0}^{(0)}(h) \\ \mu_{1,y_0}^{(0)}(h) \\ \mu_{*,y_0}^{(0)}(h) \\
\mu_{0,y_0}^{(1)}(h) \\ \mu_{1,y_0}^{(1)}(h) \\ \mu_{*,y_0}^{(1)}(h) \\
\vdots \\
\mu_{0,y_K}^{(K)}(h) \\ \mu_{1,y_K}^{(K)}(h) \\ \mu_{*,y_K}^{(K)}(h)
\end{bmatrix}
\in \mathbb{R}^{3|\mathcal{Y}|^2},
\]
\[
\begin{aligned}
M &= \mathrm{diag}(\underbrace{M_0, \dots, M_0}_{|\mathcal{Y}|^2})
\in \mathbb{R}^{4|\mathcal{Y}|^2 \times 3|\mathcal{Y}|^2},
\\
M_0 &= 
\begin{bmatrix}
\phantom{-}1 & \phantom{-}0 & -1 \\
-1 & \phantom{-}0 & \phantom{-}1 \\
\phantom{-}0 & \phantom{-}1 & -1 \\
\phantom{-}0 & -1 & \phantom{-}1
\end{bmatrix}.
\end{aligned}
\]

\subsection{Combined Parity}\label{app:combined_parity}
The Combined Parity constraint combines both 
Demographic Parity and Equalized Odds as joint 
conditions on the classifier. This results in the 
following pair of fairness constraints:
\[
\begin{aligned}
\mu_a^{(k)}(h) &= \mu_*^{(k)}(h), 
\qquad \text{and} \\
\mu_{a,y}^{(k)}(h) &= \mu_{*,y}^{(k)}(h), 
\quad \forall a \in \mathcal{A},~ \forall y, y_k \in \mathcal{Y}.
\end{aligned}
\]
Therefore, it can be expressed as inequalities:
\[
\mu_{a}^{(k)}(h) - \mu_{*}^{(k)}(h) \leq 0,
\]
\[
\mu_{*}^{(k)}(h) - \mu_{a}^{(k)}(h) \leq 0,
\]
\[
\mu_{a,y}^{(k)}(h) - \mu_{*,y}^{(k)}(h) \leq 0,
\]
\[
\mu_{*,y}^{(k)}(h) - \mu_{a,y}^{(k)}(h) \leq 0.
\]
When the binary sensitive attribute is in the form 
$A \in \{0,1\}$, for each class $y_k \in \mathcal{Y}$ 
and each true label $y \in \mathcal{Y}$, this gives:
\[
\mu_0^{(k)}(h) - \mu_*^{(k)}(h) \leq 0,
\]
\[
\mu_*^{(k)}(h) - \mu_0^{(k)}(h) \leq 0,
\]
\[
\mu_1^{(k)}(h) - \mu_*^{(k)}(h) \leq 0,
\]
\[
\mu_*^{(k)}(h) - \mu_1^{(k)}(h) \leq 0,
\]
\[
\mu_{0,y}^{(k)}(h) - \mu_{*,y}^{(k)}(h) \leq 0,
\]
\[
\mu_{*,y}^{(k)}(h) - \mu_{0,y}^{(k)}(h) \leq 0,
\]
\[
\mu_{1,y}^{(k)}(h) - \mu_{*,y}^{(k)}(h) \leq 0,
\]
\[
\mu_{*,y}^{(k)}(h) - \mu_{1,y}^{(k)}(h) \leq 0.
\]
We can compactly express this constraint system as 
$M\,\mu(h) \leq \epsilon$, with 
$\epsilon = \mathbf{0} \in 
\mathbb{R}^{4|\mathcal{Y}|+4|\mathcal{Y}|^2}$,
\[
\mu(h) =
\begin{bmatrix}
\mu_0^{(0)}(h) \\ \mu_1^{(0)}(h) \\ \mu_*^{(0)}(h) \\
\vdots \\
\mu_0^{(K)}(h) \\ \mu_1^{(K)}(h) \\ \mu_*^{(K)}(h) \\[6pt]
\mu_{0,y_0}^{(0)}(h) \\ \mu_{1,y_0}^{(0)}(h) \\ \mu_{*,y_0}^{(0)}(h) \\
\mu_{0,y_0}^{(1)}(h) \\ \mu_{1,y_0}^{(1)}(h) \\ \mu_{*,y_0}^{(1)}(h) \\
\vdots \\
\mu_{0,y_K}^{(K)}(h) \\ \mu_{1,y_K}^{(K)}(h) \\ \mu_{*,y_K}^{(K)}(h)
\end{bmatrix}
\in \mathbb{R}^{3|\mathcal{Y}|+3|\mathcal{Y}|^2},
\]
\[
\begin{aligned}
M &= \mathrm{diag}(\underbrace{M_0, \dots, M_0}_{|\mathcal{Y}|+|\mathcal{Y}|^2})
\in \mathbb{R}^{(4|\mathcal{Y}|+4|\mathcal{Y}|^2) 
\times (3|\mathcal{Y}|+3|\mathcal{Y}|^2)},
\\
M_0 &= 
\begin{bmatrix}
\phantom{-}1 & \phantom{-}0 & -1 \\
-1 & \phantom{-}0 & \phantom{-}1 \\
\phantom{-}0 & \phantom{-}1 & -1 \\
\phantom{-}0 & -1 & \phantom{-}1
\end{bmatrix}.
\end{aligned}
\]
\section{Proof of Theorem~\ref{thm:convergence}}
\label{app:proof_convergence}

We prove Theorem~\ref{thm:convergence} for the proposed Generalized
Exponentiated Gradient (\tool) algorithm.
The argument follows the reduction-based framework for fair classification and
online convex optimization introduced in~\cite{agarwal_reductions_2018}, and we
verify that the assumptions required for convergence remain valid in our
generalized setting with multiple fairness constraints and general label moments.

For any randomized classifier $Q \in \Delta(\mathcal{H})$, the empirical risk and
fairness moments satisfy
\[
\widehat{\mathcal{R}}(Q)
=
\mathbb{E}_{h\sim Q}\!\left[\widehat{\mathcal{R}}(h)\right],
\qquad
\widehat{\gamma}_i(Q)
=
\mathbb{E}_{h\sim Q}\!\left[\widehat{\gamma}_i(h)\right],
\]
and are therefore affine functions of $Q$.
Consequently, the empirical Lagrangian
\[
\mathcal{L}(Q,\boldsymbol{\lambda})
=
\widehat{\mathcal{R}}(Q)
+
\sum_{i=1}^{n} \lambda_i \bigl(\widehat{\gamma}_i(Q)-\widehat{\epsilon}_i\bigr)
\]
is convex in $Q$ and linear (hence concave) in $\boldsymbol{\lambda}$.

By construction of the empirical loss and of the general label fairness moments
using indicator functions, we have
$\widehat{\mathcal{R}}(h)\in[0,1]$ and
$\widehat{\gamma}_i(h)\in[-1,1]$ for all $h\in\mathcal{H}$ and all constraints $i$.
Assuming $\widehat{\epsilon}_i \in [0,1]$, it follows that
$|\widehat{\gamma}_i(h)-\widehat{\epsilon}_i|\le 2$.
Since the dual variable is constrained to the $\ell_1$-ball
$\Lambda=\{\boldsymbol{\lambda}\ge 0:\|\boldsymbol{\lambda}\|_1\le B\}$, it follows
that
\[
\bigl|\mathcal{L}(h,\boldsymbol{\lambda})\bigr|
\le
|\widehat{\mathcal{R}}(h)|
+
\sum_{i=1}^{n} \lambda_i
|\widehat{\gamma}_i(h)-\widehat{\epsilon}_i|
\le
1 + 2\|\boldsymbol{\lambda}\|_1
\le
1 + 2B ,
\]
and hence the payoffs of the associated zero-sum game are uniformly bounded.

To handle the constraint $\|\boldsymbol{\lambda}\|_1\le B$, we adopt the standard
reparameterization introduced in~\cite{agarwal_reductions_2018} and represent the
dual variable by a vector
$\widetilde{\boldsymbol{\lambda}} \in \Delta_{n+1}$ on the $(n+1)$-simplex, where
the first $n$ coordinates correspond to the fairness constraints and the remaining
coordinate acts as a slack variable.
The effective dual iterate is given by
$\boldsymbol{\lambda}=B\,\widetilde{\boldsymbol{\lambda}}_{1:n}$.

The dual player updates $\widetilde{\boldsymbol{\lambda}}^{(t)}$ using the
Exponentiated Gradient algorithm.
Since the per-round losses are linear and uniformly bounded by a constant of order $B$,
standard regret bounds for EG over the simplex
(e.g.,~\cite{kivinen1997exponentiated,freund1997decision})
imply that, after $T$ iterations, the cumulative regret of the dual sequence
satisfies
\[
\mathrm{Regret}_T
=
O\!\left(B\sqrt{T\log(n+1)}\right).
\]

On the learner side, we assume access to a cost-sensitive classification oracle
that returns, at each iteration $t$, a classifier $h_t$ satisfying
\[
\mathcal{L}(h_t,\boldsymbol{\lambda}^{(t)})
\le
\min_{h\in\mathcal{H}}
\mathcal{L}(h,\boldsymbol{\lambda}^{(t)}) + \tau ,
\]
for some fixed oracle accuracy $\tau \ge 0$.

Let $Q_T$ denote the uniform mixture over the classifiers
$\{h_t\}_{t=1}^T$ produced by the learner, and let
$\bar{\boldsymbol{\lambda}}_T=\frac{1}{T}\sum_{t=1}^T \boldsymbol{\lambda}^{(t)}$
denote the averaged dual iterate.
Combining the regret bound of the dual updates with the $\tau$-approximate
best-response property of the learner, and applying the standard
online-to-batch conversion for convex--concave zero-sum games with learning rate
$\eta=\Theta(1/\sqrt{T})$, yields
\[
\max_{\boldsymbol{\lambda}\in\Lambda}
\mathcal{L}(Q_T,\boldsymbol{\lambda})
-
\min_{Q\in\Delta(\mathcal{H})}
\mathcal{L}(Q,\bar{\boldsymbol{\lambda}}_T)
\;\le\;
O\!\left(\frac{B\sqrt{\log(n+1)}}{\sqrt{T}}\right) + \tau .
\]
This establishes the claim of Theorem~\ref{thm:convergence}.
\qed

\section{\revised{Detailed Results}}\label{app:detailed}

In the following, we report detailed metrics obtained by each approach analysed. In each table, we report in \textbf{bold} the cases with Wilcoxon $p$-value $\ll 0.05$ and \textit{large} $\hat{A}_{12}$ effect size with respect to the baseline(s). We highlight in \textit{italic} the cases with Wilcoxon $p$-value $\ll 0.05$ but no \textit{large} effect size.

\subsection{\revised{RQ$_1$ Detailed Results}}

\revised{Table \ref{tab:rq1_result} reports the results of the comparison of the fairness and effectiveness obtained by the three versions of \tool and the baseline LR model. Statistically significant differences in fairness scores with \textit{large} effect size are highlighted in \textbf{bold}.} 

\begin{table*}
\caption{RQ1: Fairness and effectiveness scores of \tool against a base LR model for multi-class classification.}
\label{tab:rq1_result}
\resizebox{\textwidth}{!}{\begin{tabular}{ll|ccc|cccc}
\toprule
& \textbf{Approach} & \textbf{SP} & \textbf{EO} & \textbf{AO} & \textbf{Accuracy} & \textbf{Precision} & \textbf{Recall} & \textbf{F1} \\
\midrule\midrule
\multirow{4}{*}{\begin{sideways}CMC\end{sideways}} & Baseline & 0.117 $\pm$ 0.061 & 0.217 $\pm$ 0.098 & 0.129 $\pm$ 0.059 & 0.606 $\pm$ 0.033 & 0.583 $\pm$ 0.037 & 0.560 $\pm$ 0.032 & 0.553 $\pm$ 0.033 \\\cmidrule{2-9}
 & GEG-CP & \textbf{0.053 $\pm$ 0.040} & \textbf{0.126 $\pm$ 0.063} & \textbf{0.078 $\pm$ 0.062} & 0.610 $\pm$ 0.030 & 0.591 $\pm$ 0.037 & 0.565 $\pm$ 0.023 & 0.558 $\pm$ 0.029 \\
 & GEG-SP & \textit{0.075 $\pm$ 0.072} & 0.190 $\pm$ 0.062 & 0.125 $\pm$ 0.058 & 0.609 $\pm$ 0.033 & 0.584 $\pm$ 0.044 & 0.561 $\pm$ 0.033 & 0.554 $\pm$ 0.037 \\
 & GEG-EO & \textbf{0.048 $\pm$ 0.032} & \textbf{0.154 $\pm$ 0.070} & 0.082 $\pm$ 0.040 & 0.605 $\pm$ 0.031 & 0.584 $\pm$ 0.040 & 0.561 $\pm$ 0.025 & 0.555 $\pm$ 0.027 \\
\midrule
\multirow{4}{*}{\begin{sideways}Crime\end{sideways}} & Baseline & 0.382 $\pm$ 0.062 & 0.519 $\pm$ 0.212 & 0.344 $\pm$ 0.105 & 0.474 $\pm$ 0.037 & 0.434 $\pm$ 0.046 & 0.452 $\pm$ 0.032 & 0.427 $\pm$ 0.037 \\\cmidrule{2-9}
 & GEG-CP & \textbf{0.233 $\pm$ 0.118} & \textbf{0.318 $\pm$ 0.157} & \textbf{0.190 $\pm$ 0.095} & 0.426 $\pm$ 0.044 & 0.406 $\pm$ 0.038 & 0.402 $\pm$ 0.048 & 0.384 $\pm$ 0.050 \\
 & GEG-SP & \textbf{0.193 $\pm$ 0.132} & 0.422 $\pm$ 0.196 & \textbf{0.242 $\pm$ 0.106} & 0.343 $\pm$ 0.070 & 0.332 $\pm$ 0.093 & 0.319 $\pm$ 0.058 & 0.281 $\pm$ 0.078 \\
 & GEG-EO & \textbf{0.296 $\pm$ 0.063} & \textit{0.393 $\pm$ 0.175} & \textbf{0.244 $\pm$ 0.096} & 0.447 $\pm$ 0.026 & 0.409 $\pm$ 0.043 & 0.422 $\pm$ 0.025 & 0.392 $\pm$ 0.027 \\
\midrule
\multirow{4}{*}{\begin{sideways}Drug\end{sideways}} & Baseline & 0.213 $\pm$ 0.125 & 0.359 $\pm$ 0.136 & 0.217 $\pm$ 0.090 & 0.687 $\pm$ 0.023 & 0.618 $\pm$ 0.033 & 0.614 $\pm$ 0.018 & 0.611 $\pm$ 0.025 \\\cmidrule{2-9}
 & GEG-CP & \textit{0.128 $\pm$ 0.097} & \textbf{0.280 $\pm$ 0.083} & \textbf{0.157 $\pm$ 0.049} & 0.682 $\pm$ 0.025 & 0.612 $\pm$ 0.027 & 0.607 $\pm$ 0.016 & 0.602 $\pm$ 0.019 \\
 & GEG-SP & \textbf{0.082 $\pm$ 0.073} & 0.300 $\pm$ 0.140 & \textbf{0.143 $\pm$ 0.065} & 0.682 $\pm$ 0.030 & 0.613 $\pm$ 0.037 & 0.605 $\pm$ 0.027 & 0.606 $\pm$ 0.031 \\
 & GEG-EO & 0.194 $\pm$ 0.139 & 0.324 $\pm$ 0.132 & \textit{0.191 $\pm$ 0.087} & 0.692 $\pm$ 0.023 & 0.626 $\pm$ 0.035 & 0.619 $\pm$ 0.022 & 0.616 $\pm$ 0.028 \\
\midrule
\multirow{4}{*}{\begin{sideways}Law\end{sideways}} & Baseline & 0.083 $\pm$ 0.019 & 0.105 $\pm$ 0.029 & 0.073 $\pm$ 0.018 & 0.666 $\pm$ 0.015 & 0.640 $\pm$ 0.016 & 0.652 $\pm$ 0.014 & 0.644 $\pm$ 0.015 \\\cmidrule{2-9}
 & GEG-CP & \textbf{0.046 $\pm$ 0.031} & 0.103 $\pm$ 0.031 & 0.059 $\pm$ 0.029 & \textbf{0.683 $\pm$ 0.010} & \textbf{0.657 $\pm$ 0.010} & \textbf{0.671 $\pm$ 0.010} & \textbf{0.660 $\pm$ 0.010} \\
 & GEG-SP & \textbf{0.018 $\pm$ 0.016} & 0.097 $\pm$ 0.039 & 0.063 $\pm$ 0.019 & 0.673 $\pm$ 0.015 & 0.649 $\pm$ 0.014 & 0.661 $\pm$ 0.014 & 0.652 $\pm$ 0.014 \\
 & GEG-EO & \textbf{0.047 $\pm$ 0.025} & 0.078 $\pm$ 0.040 & 0.051 $\pm$ 0.026 & \textbf{0.688 $\pm$ 0.012} & \textbf{0.661 $\pm$ 0.011} & \textbf{0.675 $\pm$ 0.011} & \textbf{0.661 $\pm$ 0.010} \\
\midrule
\multirow{4}{*}{\begin{sideways}Obesity\end{sideways}} & Baseline & 0.050 $\pm$ 0.042 & 0.573 $\pm$ 0.205 & 0.316 $\pm$ 0.093 & 0.668 $\pm$ 0.044 & 0.654 $\pm$ 0.046 & 0.665 $\pm$ 0.033 & 0.651 $\pm$ 0.038 \\\cmidrule{2-9}
 & GEG-CP & 0.061 $\pm$ 0.046 & \textbf{0.316 $\pm$ 0.138} & \textbf{0.195 $\pm$ 0.060} & 0.623 $\pm$ 0.105 & 0.613 $\pm$ 0.092 & 0.616 $\pm$ 0.111 & 0.601 $\pm$ 0.115 \\
 & GEG-SP & 0.050 $\pm$ 0.042 & 0.447 $\pm$ 0.095 & \textbf{0.217 $\pm$ 0.052} & \textit{0.687 $\pm$ 0.035} & \textit{0.685 $\pm$ 0.027} & \textit{0.686 $\pm$ 0.024} & \textit{0.668 $\pm$ 0.028} \\
 & GEG-EO & 0.063 $\pm$ 0.048 & \textbf{0.368 $\pm$ 0.118} & \textbf{0.209 $\pm$ 0.051} & 0.648 $\pm$ 0.025 & 0.623 $\pm$ 0.031 & 0.642 $\pm$ 0.031 & 0.623 $\pm$ 0.025 \\
\midrule
\multirow{4}{*}{\begin{sideways}Park\end{sideways}} & Baseline & 0.214 $\pm$ 0.072 & 0.352 $\pm$ 0.100 & 0.239 $\pm$ 0.078 & 0.473 $\pm$ 0.020 & 0.330 $\pm$ 0.039 & 0.402 $\pm$ 0.014 & 0.351 $\pm$ 0.017 \\\cmidrule{2-9}
 & GEG-CP & \textbf{0.083 $\pm$ 0.064} & \textbf{0.139 $\pm$ 0.084} & \textbf{0.093 $\pm$ 0.061} & 0.479 $\pm$ 0.018 & 0.329 $\pm$ 0.044 & 0.402 $\pm$ 0.015 & 0.348 $\pm$ 0.019 \\
 & GEG-SP & \textbf{0.053 $\pm$ 0.027} & \textbf{0.137 $\pm$ 0.060} & \textbf{0.068 $\pm$ 0.025} & \textit{0.480 $\pm$ 0.017} & 0.317 $\pm$ 0.011 & 0.400 $\pm$ 0.012 & 0.343 $\pm$ 0.012 \\
 & GEG-EO & \textbf{0.131 $\pm$ 0.051} & \textbf{0.156 $\pm$ 0.062} & \textbf{0.160 $\pm$ 0.064} & 0.479 $\pm$ 0.014 & \textit{0.384 $\pm$ 0.075} & 0.407 $\pm$ 0.016 & 0.365 $\pm$ 0.033 \\
\midrule
\multirow{4}{*}{\begin{sideways}Wine\end{sideways}} & Baseline & 0.115 $\pm$ 0.046 & 0.155 $\pm$ 0.049 & 0.125 $\pm$ 0.047 & 0.454 $\pm$ 0.019 & 0.246 $\pm$ 0.062 & 0.259 $\pm$ 0.006 & 0.201 $\pm$ 0.012 \\\cmidrule{2-9}
 & GEG-CP & \textbf{0.053 $\pm$ 0.038} & \textbf{0.083 $\pm$ 0.049} & \textbf{0.059 $\pm$ 0.040} & 0.457 $\pm$ 0.016 & 0.269 $\pm$ 0.082 & 0.259 $\pm$ 0.003 & 0.195 $\pm$ 0.007 \\
 & GEG-SP & \textbf{0.057 $\pm$ 0.023} & \textbf{0.087 $\pm$ 0.029} & \textbf{0.065 $\pm$ 0.023} & 0.459 $\pm$ 0.014 & 0.260 $\pm$ 0.052 & 0.261 $\pm$ 0.006 & 0.199 $\pm$ 0.014 \\
 & GEG-EO & \textbf{0.024 $\pm$ 0.022} & \textbf{0.058 $\pm$ 0.024} & \textbf{0.034 $\pm$ 0.023} & \textit{0.464 $\pm$ 0.015} & \textbf{0.279 $\pm$ 0.059} & \textbf{0.266 $\pm$ 0.007} & 0.208 $\pm$ 0.017 \\
\bottomrule
\end{tabular}}
\end{table*}

\subsection{\revised{RQ$_2$ Detailed Results}}

\begin{table*}[ht]
    \centering
        \caption{\revised{RQ$_2$: Results for Binary Classification against an LR model and the base EG approach from Agarwal et al. 
        }}
    \label{tab:rq2_binary}
    \revised{\resizebox{\textwidth}{!}{\begin{tabular}{ll|ccc|cccc}
\toprule
  & \textbf{Approach} & \textbf{SP} & \textbf{EO} & \textbf{AO} & \textbf{Accuracy} & \textbf{Precision} & \textbf{Recall} & \textbf{F1-score} \\
\midrule
\midrule
\multirow{4}{*}{\begin{sideways}Adult\end{sideways}} & Baseline & 0.185 $\pm$ 0.015 & 0.140 $\pm$ 0.068 & 0.112 $\pm$ 0.034 & 0.836 $\pm$ 0.004 & 0.784 $\pm$ 0.007 & 0.745 $\pm$ 0.007 & 0.760 $\pm$ 0.006 \\\cmidrule{2-9}
 & EG-SP & \textbf{0.018 $\pm$ 0.015} & 0.269 $\pm$ 0.069 & 0.150 $\pm$ 0.037 & 0.817 $\pm$ 0.006 & 0.761 $\pm$ 0.014 & 0.697 $\pm$ 0.009 & 0.718 $\pm$ 0.010 \\
 & EG-EO & \textbf{0.095 $\pm$ 0.010} & \textbf{0.042 $\pm$ 0.038} & \textbf{0.021 $\pm$ 0.017} & 0.822 $\pm$ 0.005 & 0.770 $\pm$ 0.012 & 0.707 $\pm$ 0.010 & 0.728 $\pm$ 0.010 \\
 & GEG-CP & \textbf{0.077 $\pm$ 0.037} & 0.119 $\pm$ 0.094 & \textbf{0.059 $\pm$ 0.053} & 0.829 $\pm$ 0.007 & 0.790 $\pm$ 0.017 & 0.705 $\pm$ 0.008 & 0.731 $\pm$ 0.009 \\
\midrule
\multirow{4}{*}{\begin{sideways}Compas\end{sideways}} & Baseline & 0.174 $\pm$ 0.040 & 0.102 $\pm$ 0.031 & 0.150 $\pm$ 0.042 & 0.675 $\pm$ 0.021 & 0.675 $\pm$ 0.020 & 0.666 $\pm$ 0.020 & 0.666 $\pm$ 0.021 \\\cmidrule{2-9}
 & EG-SP & \textbf{0.041 $\pm$ 0.029} & \textbf{0.061 $\pm$ 0.047} & \textbf{0.036 $\pm$ 0.031} & 0.665 $\pm$ 0.019 & 0.665 $\pm$ 0.021 & 0.654 $\pm$ 0.019 & 0.654 $\pm$ 0.019 \\
 & EG-EO & \textbf{0.047 $\pm$ 0.042} & \textbf{0.032 $\pm$ 0.023} & \textbf{0.046 $\pm$ 0.024} & 0.667 $\pm$ 0.017 & 0.666 $\pm$ 0.016 & 0.657 $\pm$ 0.018 & 0.657 $\pm$ 0.020 \\
 & GEG-CP & \textbf{0.063 $\pm$ 0.052} & \textbf{0.038 $\pm$ 0.025} & \textbf{0.055 $\pm$ 0.040} & 0.670 $\pm$ 0.020 & 0.670 $\pm$ 0.019 & 0.660 $\pm$ 0.018 & 0.660 $\pm$ 0.019 \\
\midrule
\multirow{4}{*}{\begin{sideways}German\end{sideways}} & Baseline & 0.212 $\pm$ 0.126 & 0.181 $\pm$ 0.161 & 0.173 $\pm$ 0.144 & 0.752 $\pm$ 0.051 & 0.705 $\pm$ 0.073 & 0.671 $\pm$ 0.063 & 0.679 $\pm$ 0.066 \\\cmidrule{2-9}
 & EG-SP & \textbf{0.109 $\pm$ 0.065} & \textit{0.106 $\pm$ 0.150} & 0.132 $\pm$ 0.056 & 0.742 $\pm$ 0.057 & 0.690 $\pm$ 0.083 & 0.655 $\pm$ 0.066 & 0.663 $\pm$ 0.072 \\
 & EG-EO & \textit{0.155 $\pm$ 0.079} & 0.139 $\pm$ 0.116 & 0.159 $\pm$ 0.071 & 0.748 $\pm$ 0.052 & 0.698 $\pm$ 0.076 & 0.663 $\pm$ 0.064 & 0.671 $\pm$ 0.069 \\
 & GEG-CP & 0.155 $\pm$ 0.093 & \textit{0.116 $\pm$ 0.147} & 0.171 $\pm$ 0.088 & 0.746 $\pm$ 0.060 & 0.698 $\pm$ 0.086 & 0.660 $\pm$ 0.066 & 0.670 $\pm$ 0.073 \\
\bottomrule
\end{tabular}}}
\end{table*}

\revised{table* \ref{tab:rq2_binary} reports the fairness and effectiveness achieved by \tool for binary classification. We recall that, in this context, our novel contribution is the extension of the original EG approach from Agarwal et al. \cite{agarwal_reductions_2018} with the CP constraint (GEG-CP in Table \ref{tab:rq2_binary}).}

\subsection{\revised{RQ$_3$ Detailed Results}}

\begin{table*}[ht]
    \centering
    \caption{\revised{RQ$_3$: Comparison with the DEMV pre-processing approach for multi-class classification.}}
    \label{tab:rq3_DEMV}
    \resizebox{\textwidth}{!}{
   \begin{tabular}{ll|ccc|cccc}
\toprule
& \textbf{Approach} & \textbf{SP} & \textbf{EO} & \textbf{AO} & \textbf{Accuracy} & \textbf{Precision} & \textbf{Recall} & \textbf{F1} \\
\midrule\midrule
\multirow{5}{*}{\begin{sideways}CMC\end{sideways}} & DEMV & 0.054 $\pm$ 0.043 & 0.181 $\pm$ 0.111 & 0.082 $\pm$ 0.046 & 0.603 $\pm$ 0.028 & 0.580 $\pm$ 0.036 & 0.553 $\pm$ 0.026 & 0.544 $\pm$ 0.028 \\\cmidrule{2-9}
 & GEG-CP & 0.053 $\pm$ 0.040 & \textit{0.126 $\pm$ 0.063} & 0.078 $\pm$ 0.062 & 0.610 $\pm$ 0.030 & 0.591 $\pm$ 0.037 & \textit{0.565 $\pm$ 0.023} & 0.558 $\pm$ 0.029 \\
 & GEG-SP & 0.075 $\pm$ 0.072 & 0.190 $\pm$ 0.062 & 0.125 $\pm$ 0.058 & 0.609 $\pm$ 0.033 & 0.584 $\pm$ 0.044 & 0.561 $\pm$ 0.033 & 0.554 $\pm$ 0.037 \\
 & GEG-EO & 0.048 $\pm$ 0.032 & 0.154 $\pm$ 0.070 & 0.082 $\pm$ 0.040 & 0.605 $\pm$ 0.031 & 0.584 $\pm$ 0.040 & 0.561 $\pm$ 0.025 & 0.555 $\pm$ 0.027 \\
\midrule
\multirow{5}{*}{\begin{sideways}Crime\end{sideways}} & DEMV & 0.343 $\pm$ 0.061 & 0.350 $\pm$ 0.200 & 0.246 $\pm$ 0.118 & 0.457 $\pm$ 0.034 & 0.423 $\pm$ 0.049 & 0.433 $\pm$ 0.036 & 0.406 $\pm$ 0.039 \\\cmidrule{2-9}
 & GEG-CP & \textbf{0.233 $\pm$ 0.118} & 0.318 $\pm$ 0.157 & 0.190 $\pm$ 0.095 & 0.426 $\pm$ 0.044 & 0.406 $\pm$ 0.038 & 0.402 $\pm$ 0.048 & 0.384 $\pm$ 0.050 \\
 & GEG-SP & \textbf{0.193 $\pm$ 0.132} & 0.422 $\pm$ 0.196 & 0.242 $\pm$ 0.106 & 0.343 $\pm$ 0.070 & 0.332 $\pm$ 0.093 & 0.319 $\pm$ 0.058 & 0.281 $\pm$ 0.078 \\
 & GEG-EO & \textbf{0.296 $\pm$ 0.063} & 0.393 $\pm$ 0.175 & 0.244 $\pm$ 0.096 & 0.447 $\pm$ 0.026 & 0.409 $\pm$ 0.043 & 0.422 $\pm$ 0.025 & 0.392 $\pm$ 0.027 \\
\midrule
\multirow{5}{*}{\begin{sideways}Drug\end{sideways}} & DEMV & 0.084 $\pm$ 0.081 & 0.243 $\pm$ 0.125 & 0.116 $\pm$ 0.064 & 0.686 $\pm$ 0.031 & 0.621 $\pm$ 0.037 & 0.609 $\pm$ 0.026 & 0.611 $\pm$ 0.030 \\\cmidrule{2-9}
 & GEG-CP & 0.128 $\pm$ 0.097 & 0.280 $\pm$ 0.083 & 0.157 $\pm$ 0.049 & 0.682 $\pm$ 0.025 & 0.612 $\pm$ 0.027 & 0.607 $\pm$ 0.016 & 0.602 $\pm$ 0.019 \\
 & GEG-SP & 0.082 $\pm$ 0.073 & 0.300 $\pm$ 0.140 & 0.143 $\pm$ 0.065 & 0.682 $\pm$ 0.030 & 0.613 $\pm$ 0.037 & 0.605 $\pm$ 0.027 & 0.606 $\pm$ 0.031 \\
 & GEG-EO & 0.194 $\pm$ 0.139 & 0.324 $\pm$ 0.132 & 0.191 $\pm$ 0.087 & 0.692 $\pm$ 0.023 & 0.626 $\pm$ 0.035 & 0.619 $\pm$ 0.022 & 0.616 $\pm$ 0.028 \\
\midrule
\multirow{5}{*}{\begin{sideways}Law\end{sideways}} & DEMV & 0.066 $\pm$ 0.023 & 0.100 $\pm$ 0.052 & 0.071 $\pm$ 0.027 & 0.661 $\pm$ 0.023 & 0.639 $\pm$ 0.021 & 0.648 $\pm$ 0.025 & 0.642 $\pm$ 0.022 \\\cmidrule{2-9}
 & GEG-CP & 0.046 $\pm$ 0.031 & 0.103 $\pm$ 0.031 & 0.059 $\pm$ 0.029 & \textbf{0.683 $\pm$ 0.010} & \textbf{0.657 $\pm$ 0.010} & \textbf{0.671 $\pm$ 0.010} & \textbf{0.660 $\pm$ 0.010} \\
 & GEG-SP & \textbf{0.018 $\pm$ 0.016} & 0.097 $\pm$ 0.039 & 0.063 $\pm$ 0.019 & \textit{0.673 $\pm$ 0.015} & 0.649 $\pm$ 0.014 & \textit{0.661 $\pm$ 0.014} & \textit{0.652 $\pm$ 0.014} \\
 & GEG-EO & 0.047 $\pm$ 0.025 & 0.078 $\pm$ 0.040 & \textit{0.051 $\pm$ 0.026} & \textbf{0.688 $\pm$ 0.012} & \textbf{0.661 $\pm$ 0.011} & \textbf{0.675 $\pm$ 0.011} & \textbf{0.661 $\pm$ 0.010} \\
\midrule
\multirow{5}{*}{\begin{sideways}Obesity\end{sideways}} & DEMV & 0.049 $\pm$ 0.045 & 0.513 $\pm$ 0.117 & 0.261 $\pm$ 0.050 & 0.665 $\pm$ 0.047 & 0.663 $\pm$ 0.040 & 0.663 $\pm$ 0.035 & 0.654 $\pm$ 0.036 \\\cmidrule{2-9}
 & GEG-CP & 0.061 $\pm$ 0.046 & \textbf{0.316 $\pm$ 0.138} & \textbf{0.195 $\pm$ 0.060} & 0.623 $\pm$ 0.105 & 0.613 $\pm$ 0.092 & 0.616 $\pm$ 0.111 & 0.601 $\pm$ 0.115 \\
 & GEG-SP & 0.050 $\pm$ 0.042 & 0.447 $\pm$ 0.095 & \textbf{0.217 $\pm$ 0.052} & \textit{0.687 $\pm$ 0.035} & \textit{0.685 $\pm$ 0.027} & \textbf{0.686 $\pm$ 0.024} & \textit{0.668 $\pm$ 0.028} \\
 & GEG-EO & 0.063 $\pm$ 0.048 & \textbf{0.368 $\pm$ 0.118} & \textbf{0.209 $\pm$ 0.051} & 0.648 $\pm$ 0.025 & 0.623 $\pm$ 0.031 & 0.642 $\pm$ 0.031 & 0.623 $\pm$ 0.025 \\
\midrule
\multirow{5}{*}{\begin{sideways}Park\end{sideways}} & DEMV & 0.169 $\pm$ 0.070 & 0.285 $\pm$ 0.090 & 0.185 $\pm$ 0.077 & 0.467 $\pm$ 0.019 & 0.387 $\pm$ 0.059 & 0.396 $\pm$ 0.015 & 0.350 $\pm$ 0.013 \\\cmidrule{2-9}
 & GEG-CP & \textbf{0.083 $\pm$ 0.064} & \textbf{0.139 $\pm$ 0.084} & \textbf{0.093 $\pm$ 0.061} & \textit{0.479 $\pm$ 0.018} & 0.329 $\pm$ 0.044 & 0.402 $\pm$ 0.015 & 0.348 $\pm$ 0.019 \\
 & GEG-SP & \textbf{0.053 $\pm$ 0.027} & \textbf{0.137 $\pm$ 0.060} & \textbf{0.068 $\pm$ 0.025} & \textbf{0.480 $\pm$ 0.017} & 0.317 $\pm$ 0.011 & 0.400 $\pm$ 0.012 & 0.343 $\pm$ 0.012 \\
 & GEG-EO & 0.131 $\pm$ 0.051 & \textbf{0.156 $\pm$ 0.062} & 0.160 $\pm$ 0.064 & \textbf{0.479 $\pm$ 0.014} & 0.384 $\pm$ 0.075 & \textit{0.407 $\pm$ 0.016} & 0.365 $\pm$ 0.033 \\
\midrule
\multirow{5}{*}{\begin{sideways}Wine\end{sideways}} & DEMV & 0.137 $\pm$ 0.055 & 0.190 $\pm$ 0.066 & 0.150 $\pm$ 0.057 & 0.457 $\pm$ 0.016 & 0.258 $\pm$ 0.068 & 0.261 $\pm$ 0.008 & 0.205 $\pm$ 0.015 \\\cmidrule{2-9}
 & GEG-CP & \textbf{0.053 $\pm$ 0.038} & \textbf{0.083 $\pm$ 0.049} & \textbf{0.059 $\pm$ 0.040} & 0.457 $\pm$ 0.016 & 0.269 $\pm$ 0.082 & 0.259 $\pm$ 0.003 & 0.195 $\pm$ 0.007 \\
 & GEG-SP & \textbf{0.057 $\pm$ 0.023} & \textbf{0.087 $\pm$ 0.029} & \textbf{0.065 $\pm$ 0.023} & 0.459 $\pm$ 0.014 & 0.260 $\pm$ 0.052 & 0.261 $\pm$ 0.006 & 0.199 $\pm$ 0.014 \\
 & GEG-EO & \textbf{0.024 $\pm$ 0.022} & \textbf{0.058 $\pm$ 0.024} & \textbf{0.034 $\pm$ 0.023} & \textit{0.464 $\pm$ 0.015} & 0.279 $\pm$ 0.059 & 0.266 $\pm$ 0.007 & 0.208 $\pm$ 0.017 \\
\bottomrule
\end{tabular}
    }
\end{table*}

\begin{table*}[tb]
    \centering
    \caption{\revised{RQ$_3$: Comparison with the Blackbox post-processing approach for multi-class classification.}}
    \label{tab:rq3_blackbox}
    \resizebox{\textwidth}{!}{
    \begin{tabular}{ll|ccc|cccc}
\toprule
& \textbf{Approach} & \textbf{SP} & \textbf{EO} & \textbf{AO} & \textbf{Accuracy} & \textbf{Precision} & \textbf{Recall} & \textbf{F1} \\
\midrule\midrule
\multirow{5}{*}{\begin{sideways}CMC\end{sideways}} 
 & Blackbox & 0.169 $\pm$ 0.131 & 0.687 $\pm$ 0.130 & 0.654 $\pm$ 0.082 & 0.384 $\pm$ 0.071 & 0.362 $\pm$ 0.071 & 0.364 $\pm$ 0.073 & 0.352 $\pm$ 0.072 \\\cmidrule{2-9}
 & GEG-CP & \textbf{0.053 $\pm$ 0.040} & \textbf{0.126 $\pm$ 0.063} & \textbf{0.078 $\pm$ 0.062} & \textbf{0.610 $\pm$ 0.030} & \textbf{0.591 $\pm$ 0.037} & \textbf{0.565 $\pm$ 0.023} & \textbf{0.558 $\pm$ 0.029} \\
 & GEG-SP & 0.075 $\pm$ 0.072 & \textbf{0.190 $\pm$ 0.062} & \textbf{0.125 $\pm$ 0.058} & \textbf{0.609 $\pm$ 0.033} & \textbf{0.584 $\pm$ 0.044} & \textbf{0.561 $\pm$ 0.033} & \textbf{0.554 $\pm$ 0.037} \\
 & GEG-EO & \textbf{0.048 $\pm$ 0.032} & \textbf{0.154 $\pm$ 0.070} & \textbf{0.082 $\pm$ 0.040} & \textbf{0.605 $\pm$ 0.031} & \textbf{0.584 $\pm$ 0.040} & \textbf{0.561 $\pm$ 0.025} & \textbf{0.555 $\pm$ 0.027} \\
\midrule
\multirow{5}{*}{\begin{sideways}Crime\end{sideways}} 
 & Blackbox & 0.403 $\pm$ 0.084 & 0.648 $\pm$ 0.153 & 0.561 $\pm$ 0.088 & 0.207 $\pm$ 0.040 & 0.197 $\pm$ 0.065 & 0.202 $\pm$ 0.042 & 0.174 $\pm$ 0.040 \\\cmidrule{2-9}
 & GEG-CP & \textbf{0.233 $\pm$ 0.118} & \textbf{0.318 $\pm$ 0.157} & \textbf{0.190 $\pm$ 0.095} & \textbf{0.426 $\pm$ 0.044} & \textbf{0.406 $\pm$ 0.038} & \textbf{0.402 $\pm$ 0.048} & \textbf{0.384 $\pm$ 0.050} \\
 & GEG-SP & \textbf{0.193 $\pm$ 0.132} & \textbf{0.422 $\pm$ 0.196} & \textbf{0.242 $\pm$ 0.106} & \textbf{0.343 $\pm$ 0.070} & \textbf{0.332 $\pm$ 0.093} & \textbf{0.319 $\pm$ 0.058} & \textbf{0.281 $\pm$ 0.078} \\
 & GEG-EO & \textbf{0.296 $\pm$ 0.063} & \textbf{0.393 $\pm$ 0.175} & \textbf{0.244 $\pm$ 0.096} & \textbf{0.447 $\pm$ 0.026} & \textbf{0.409 $\pm$ 0.043} & \textbf{0.422 $\pm$ 0.025} & \textbf{0.392 $\pm$ 0.027} \\
\midrule
\multirow{5}{*}{\begin{sideways}Drug\end{sideways}} 
 & Blackbox & 0.342 $\pm$ 0.161 & 0.574 $\pm$ 0.228 & 0.437 $\pm$ 0.173 & 0.458 $\pm$ 0.061 & 0.396 $\pm$ 0.080 & 0.410 $\pm$ 0.040 & 0.384 $\pm$ 0.060 \\\cmidrule{2-9}
 & GEG-CP & \textbf{0.128 $\pm$ 0.097} & \textbf{0.280 $\pm$ 0.083} & \textbf{0.157 $\pm$ 0.049} & \textbf{0.682 $\pm$ 0.025} & \textbf{0.612 $\pm$ 0.027} & \textbf{0.607 $\pm$ 0.016} & \textbf{0.602 $\pm$ 0.019} \\
 & GEG-SP & \textbf{0.082 $\pm$ 0.073} & \textbf{0.300 $\pm$ 0.140} & \textbf{0.143 $\pm$ 0.065} & \textbf{0.682 $\pm$ 0.030} & \textbf{0.613 $\pm$ 0.037} & \textbf{0.605 $\pm$ 0.027} & \textbf{0.606 $\pm$ 0.031} \\
 & GEG-EO & \textbf{0.194 $\pm$ 0.139} & \textbf{0.324 $\pm$ 0.132} & \textbf{0.191 $\pm$ 0.087} & \textbf{0.692 $\pm$ 0.023} & \textbf{0.626 $\pm$ 0.035} & \textbf{0.619 $\pm$ 0.022} & \textbf{0.616 $\pm$ 0.028} \\
\midrule
\multirow{5}{*}{\begin{sideways}Law\end{sideways}}
 & Blackbox & 0.205 $\pm$ 0.025 & 0.234 $\pm$ 0.037 & 0.215 $\pm$ 0.027 & 0.431 $\pm$ 0.045 & 0.476 $\pm$ 0.031 & 0.452 $\pm$ 0.039 & 0.432 $\pm$ 0.050 \\\cmidrule{2-9}
 & GEG-CP & \textbf{0.046 $\pm$ 0.031} & \textbf{0.103 $\pm$ 0.031} & \textbf{0.059 $\pm$ 0.029} & \textbf{0.683 $\pm$ 0.010} & \textbf{0.657 $\pm$ 0.010} & \textbf{0.671 $\pm$ 0.010} & \textbf{0.660 $\pm$ 0.010} \\
 & GEG-SP & \textbf{0.018 $\pm$ 0.016} & \textbf{0.097 $\pm$ 0.039} & \textbf{0.063 $\pm$ 0.019} & \textbf{0.673 $\pm$ 0.015} & \textbf{0.649 $\pm$ 0.014} & \textbf{0.661 $\pm$ 0.014} & \textbf{0.652 $\pm$ 0.014} \\
 & GEG-EO & \textbf{0.047 $\pm$ 0.025} & \textbf{0.078 $\pm$ 0.040} & \textbf{0.051 $\pm$ 0.026} & \textbf{0.688 $\pm$ 0.012} & \textbf{0.661 $\pm$ 0.011} & \textbf{0.675 $\pm$ 0.011} & \textbf{0.661 $\pm$ 0.010} \\
\midrule
\multirow{5}{*}{\begin{sideways}Obesity\end{sideways}}
 & Blackbox & 0.251 $\pm$ 0.065 & 0.647 $\pm$ 0.097 & 0.431 $\pm$ 0.054 & 0.366 $\pm$ 0.090 & 0.356 $\pm$ 0.097 & 0.356 $\pm$ 0.076 & 0.336 $\pm$ 0.073 \\\cmidrule{2-9}
 & GEG-CP & \textbf{0.061 $\pm$ 0.046} & \textbf{0.316 $\pm$ 0.138} & \textbf{0.195 $\pm$ 0.060} & \textbf{0.623 $\pm$ 0.105} & \textbf{0.613 $\pm$ 0.092} & \textbf{0.616 $\pm$ 0.111} & \textbf{0.601 $\pm$ 0.115} \\
 & GEG-SP & \textbf{0.050 $\pm$ 0.042} & \textbf{0.447 $\pm$ 0.095} & \textbf{0.217 $\pm$ 0.052} & \textbf{0.687 $\pm$ 0.035} & \textbf{0.685 $\pm$ 0.027} & \textbf{0.686 $\pm$ 0.024} & \textbf{0.668 $\pm$ 0.028} \\
 & GEG-EO & \textbf{0.063 $\pm$ 0.048} & \textbf{0.368 $\pm$ 0.118} & \textbf{0.209 $\pm$ 0.051} & \textbf{0.648 $\pm$ 0.025} & \textbf{0.623 $\pm$ 0.031} & \textbf{0.642 $\pm$ 0.031} & \textbf{0.623 $\pm$ 0.025} \\
\midrule
\multirow{5}{*}{\begin{sideways}Park\end{sideways}}
 & Blackbox & 0.220 $\pm$ 0.092 & 0.239 $\pm$ 0.078 & 0.223 $\pm$ 0.080 & 0.334 $\pm$ 0.041 & 0.348 $\pm$ 0.054 & 0.343 $\pm$ 0.036 & 0.288 $\pm$ 0.031 \\\cmidrule{2-9}
 & GEG-CP & \textbf{0.083 $\pm$ 0.064} & \textbf{0.139 $\pm$ 0.084} & \textbf{0.093 $\pm$ 0.061} & \textbf{0.479 $\pm$ 0.018} & 0.329 $\pm$ 0.044 & \textbf{0.402 $\pm$ 0.015} & \textbf{0.348 $\pm$ 0.019} \\
 & GEG-SP & \textbf{0.053 $\pm$ 0.027} & \textbf{0.137 $\pm$ 0.060} & \textbf{0.068 $\pm$ 0.025} & \textbf{0.480 $\pm$ 0.017} & 0.317 $\pm$ 0.011 & \textbf{0.400 $\pm$ 0.012} & \textbf{0.343 $\pm$ 0.012} \\
 & GEG-EO & \textbf{0.131 $\pm$ 0.051} & \textbf{0.156 $\pm$ 0.062} & 0.160 $\pm$ 0.064 & \textbf{0.479 $\pm$ 0.014} & 0.384 $\pm$ 0.075 & \textbf{0.407 $\pm$ 0.016} & \textbf{0.365 $\pm$ 0.033} \\
\midrule
\multirow{5}{*}{\begin{sideways}Wine\end{sideways}}
 & Blackbox & 0.039 $\pm$ 0.044 & 0.297 $\pm$ 0.192 & 0.271 $\pm$ 0.116 & 0.244 $\pm$ 0.088 & 0.240 $\pm$ 0.031 & 0.257 $\pm$ 0.027 & 0.172 $\pm$ 0.041 \\\cmidrule{2-9}
 & GEG-CP & 0.053 $\pm$ 0.038 & \textbf{0.083 $\pm$ 0.049} & \textbf{0.059 $\pm$ 0.040} & \textbf{0.457 $\pm$ 0.016} & 0.269 $\pm$ 0.082 & 0.259 $\pm$ 0.003 & 0.195 $\pm$ 0.007 \\
 & GEG-SP & 0.057 $\pm$ 0.023 & \textbf{0.087 $\pm$ 0.029} & \textbf{0.065 $\pm$ 0.023} & \textbf{0.459 $\pm$ 0.014} & 0.260 $\pm$ 0.052 & 0.261 $\pm$ 0.006 & \textbf{0.199 $\pm$ 0.014} \\
 & GEG-EO & 0.024 $\pm$ 0.022 & \textbf{0.058 $\pm$ 0.024} & \textbf{0.034 $\pm$ 0.023} & \textbf{0.464 $\pm$ 0.015} & 0.279 $\pm$ 0.059 & 0.266 $\pm$ 0.007 & \textbf{0.208 $\pm$ 0.017} \\
\bottomrule
\end{tabular}
    }
\end{table*}

\revised{Table \ref{tab:rq3_DEMV} shows the results of the comparison between the three versions of \tool and the DEMV baseline for multi-class classification, while Table \ref{tab:rq3_blackbox} reports the comparison between the three versions of \tool and the Blackbox post-processing approach.}

\subsection{\revised{RQ$_5$ Detailed Results}}

\revised{Table \ref{tab:rq4_rf} presents the results of applying \tool with an RF classifier, while Table \ref{tab:rq4_gb} reports the results of applying \tool with a GB classification method.} Figures \ref{fig:rq5_fairness} and \ref{fig:rq5_effectiveness} show the comparison of single fairness and effectiveness scores, respectively.

\begin{table*}[ht]
    \centering
    \caption{RQ$_4$: Results obtained with RF base classifier.}
    \label{tab:rq4_rf}
      \resizebox{\textwidth}{!}{
        \begin{tabular}{ll|cccc|cccccc}
        \toprule
        & \textbf{Approach} & \textbf{SP} & \textbf{EO} & \textbf{AO} & \textbf{Accuracy} & \textbf{Precision} & \textbf{Recall} & \textbf{F1} \\
        \midrule\midrule
        \multirow{6}{*}{\begin{sideways}CMC\end{sideways}} & Baseline & 0.132 $\pm$ 0.077 & 0.131 $\pm$ 0.070 & 0.064 $\pm$ 0.035 & 0.976 $\pm$ 0.014 & 0.974 $\pm$ 0.012 & 0.972 $\pm$ 0.015 & 0.973 $\pm$ 0.014 \\
 & Blackbox & 0.102 $\pm$ 0.074 & 0.332 $\pm$ 0.181 & 0.204 $\pm$ 0.109 & 0.882 $\pm$ 0.047 & 0.883 $\pm$ 0.058 & 0.869 $\pm$ 0.064 & 0.868 $\pm$ 0.060 \\
 & DEMV & 0.123 $\pm$ 0.088 & \textit{0.088 $\pm$ 0.041} & \textit{0.044 $\pm$ 0.021} & 0.978 $\pm$ 0.014 & 0.978 $\pm$ 0.012 & 0.975 $\pm$ 0.015 & 0.976 $\pm$ 0.013 \\\cmidrule{2-9}
 & GEG-CP & 0.138 $\pm$ 0.055 & 0.173 $\pm$ 0.091 & 0.093 $\pm$ 0.048 & 0.800 $\pm$ 0.059 & 0.836 $\pm$ 0.030 & 0.836 $\pm$ 0.044 & 0.795 $\pm$ 0.061 \\
 & GEG-SP & \textbf{0.069 $\pm$ 0.036} & 0.193 $\pm$ 0.043 & 0.107 $\pm$ 0.036 & 0.872 $\pm$ 0.024 & 0.869 $\pm$ 0.028 & 0.883 $\pm$ 0.022 & 0.862 $\pm$ 0.030 \\
 & GEG-EO & 0.153 $\pm$ 0.080 & 0.104 $\pm$ 0.049 & 0.052 $\pm$ 0.024 & \textbf{0.982 $\pm$ 0.010} & \textbf{0.983 $\pm$ 0.008} & \textbf{0.980 $\pm$ 0.012} & \textbf{0.981 $\pm$ 0.010} \\
\midrule
\multirow{6}{*}{\begin{sideways}Crime\end{sideways}} & Baseline & 0.422 $\pm$ 0.055 & 0.491 $\pm$ 0.226 & 0.347 $\pm$ 0.118 & 0.518 $\pm$ 0.036 & 0.490 $\pm$ 0.035 & 0.494 $\pm$ 0.026 & 0.482 $\pm$ 0.029 \\
 & Blackbox & 0.411 $\pm$ 0.070 & 0.652 $\pm$ 0.173 & 0.566 $\pm$ 0.074 & 0.192 $\pm$ 0.056 & 0.189 $\pm$ 0.045 & 0.186 $\pm$ 0.049 & 0.166 $\pm$ 0.040 \\
 & DEMV & \textbf{0.346 $\pm$ 0.039} & 0.456 $\pm$ 0.192 & \textit{0.307 $\pm$ 0.097} & 0.487 $\pm$ 0.026 & 0.473 $\pm$ 0.026 & 0.467 $\pm$ 0.026 & 0.460 $\pm$ 0.026 \\\cmidrule{2-9}
 & GEG-CP & \textbf{0.180 $\pm$ 0.048} & 0.412 $\pm$ 0.095 & \textbf{0.222 $\pm$ 0.060} & 0.414 $\pm$ 0.026 & 0.436 $\pm$ 0.036 & 0.383 $\pm$ 0.023 & 0.363 $\pm$ 0.027 \\
 & GEG-SP & \textbf{0.148 $\pm$ 0.064} & 0.412 $\pm$ 0.132 & \textbf{0.188 $\pm$ 0.060} & 0.432 $\pm$ 0.033 & 0.430 $\pm$ 0.042 & 0.406 $\pm$ 0.023 & 0.388 $\pm$ 0.032 \\
 & GEG-EO & 0.423 $\pm$ 0.050 & 0.514 $\pm$ 0.191 & 0.359 $\pm$ 0.110 & 0.496 $\pm$ 0.031 & 0.465 $\pm$ 0.030 & 0.475 $\pm$ 0.021 & 0.462 $\pm$ 0.026 \\
\midrule
\multirow{6}{*}{\begin{sideways}Drug\end{sideways}} & Baseline & 0.160 $\pm$ 0.102 & 0.380 $\pm$ 0.105 & 0.223 $\pm$ 0.068 & 0.681 $\pm$ 0.021 & 0.608 $\pm$ 0.024 & 0.599 $\pm$ 0.020 & 0.600 $\pm$ 0.019 \\
 & Blackbox & 0.357 $\pm$ 0.135 & 0.647 $\pm$ 0.110 & 0.504 $\pm$ 0.148 & 0.435 $\pm$ 0.126 & 0.402 $\pm$ 0.108 & 0.415 $\pm$ 0.073 & 0.384 $\pm$ 0.090 \\
 & DEMV & \textbf{0.088 $\pm$ 0.055} & 0.322 $\pm$ 0.111 & \textbf{0.157 $\pm$ 0.085} & 0.667 $\pm$ 0.029 & 0.594 $\pm$ 0.025 & 0.582 $\pm$ 0.023 & 0.584 $\pm$ 0.024 \\\cmidrule{2-9}
 & GEG-CP & \textit{0.113 $\pm$ 0.091} & \textbf{0.209 $\pm$ 0.072} & \textbf{0.127 $\pm$ 0.051} & 0.642 $\pm$ 0.026 & 0.553 $\pm$ 0.043 & 0.571 $\pm$ 0.034 & 0.532 $\pm$ 0.032 \\
 & GEG-SP & \textit{0.089 $\pm$ 0.064} & 0.336 $\pm$ 0.049 & 0.205 $\pm$ 0.047 & 0.602 $\pm$ 0.028 & 0.553 $\pm$ 0.031 & 0.557 $\pm$ 0.035 & 0.540 $\pm$ 0.027 \\
 & GEG-EO & 0.161 $\pm$ 0.109 & 0.344 $\pm$ 0.080 & 0.202 $\pm$ 0.050 & 0.684 $\pm$ 0.033 & 0.617 $\pm$ 0.033 & 0.607 $\pm$ 0.032 & 0.607 $\pm$ 0.030 \\
\midrule
\multirow{6}{*}{\begin{sideways}Law\end{sideways}} & Baseline & 0.162 $\pm$ 0.024 & 0.151 $\pm$ 0.035 & 0.078 $\pm$ 0.019 & 0.971 $\pm$ 0.005 & 0.973 $\pm$ 0.004 & 0.972 $\pm$ 0.005 & 0.973 $\pm$ 0.005 \\
 & Blackbox & 0.162 $\pm$ 0.021 & 0.270 $\pm$ 0.045 & 0.141 $\pm$ 0.024 & 0.888 $\pm$ 0.022 & 0.901 $\pm$ 0.021 & 0.887 $\pm$ 0.021 & 0.888 $\pm$ 0.022 \\
 & DEMV & \textbf{0.143 $\pm$ 0.026} & \textbf{0.078 $\pm$ 0.043} & \textbf{0.039 $\pm$ 0.021} & 0.973 $\pm$ 0.004 & 0.974 $\pm$ 0.004 & 0.975 $\pm$ 0.004 & 0.974 $\pm$ 0.004 \\\cmidrule{2-9}
 & GEG-CP & \textbf{0.081 $\pm$ 0.027} & \textbf{0.124 $\pm$ 0.037} & \textbf{0.032 $\pm$ 0.016} & 0.915 $\pm$ 0.009 & 0.920 $\pm$ 0.008 & 0.927 $\pm$ 0.008 & 0.920 $\pm$ 0.008 \\
 & GEG-SP & \textbf{0.038 $\pm$ 0.026} & 0.177 $\pm$ 0.032 & 0.072 $\pm$ 0.021 & 0.952 $\pm$ 0.009 & 0.953 $\pm$ 0.008 & 0.955 $\pm$ 0.008 & 0.953 $\pm$ 0.008 \\
 & GEG-EO & 0.161 $\pm$ 0.025 & 0.145 $\pm$ 0.043 & 0.076 $\pm$ 0.022 & 0.971 $\pm$ 0.007 & 0.972 $\pm$ 0.006 & 0.972 $\pm$ 0.007 & 0.972 $\pm$ 0.006 \\
\midrule
\multirow{6}{*}{\begin{sideways}Obesity\end{sideways}} & Baseline & 0.056 $\pm$ 0.046 & 0.296 $\pm$ 0.191 & 0.153 $\pm$ 0.098 & 0.928 $\pm$ 0.017 & 0.932 $\pm$ 0.015 & 0.925 $\pm$ 0.017 & 0.926 $\pm$ 0.016 \\
 & Blackbox & 0.137 $\pm$ 0.082 & 0.451 $\pm$ 0.094 & 0.263 $\pm$ 0.061 & 0.680 $\pm$ 0.062 & 0.726 $\pm$ 0.091 & 0.674 $\pm$ 0.064 & 0.658 $\pm$ 0.075 \\
 & DEMV & 0.057 $\pm$ 0.049 & 0.325 $\pm$ 0.182 & 0.167 $\pm$ 0.092 & 0.915 $\pm$ 0.021 & 0.922 $\pm$ 0.017 & 0.913 $\pm$ 0.022 & 0.914 $\pm$ 0.019 \\\cmidrule{2-9}
 & GEG-CP & 0.058 $\pm$ 0.038 & 0.301 $\pm$ 0.190 & 0.154 $\pm$ 0.093 & 0.929 $\pm$ 0.012 & 0.935 $\pm$ 0.009 & 0.927 $\pm$ 0.013 & 0.927 $\pm$ 0.011 \\
 & GEG-SP & 0.050 $\pm$ 0.026 & 0.336 $\pm$ 0.164 & 0.173 $\pm$ 0.079 & 0.904 $\pm$ 0.018 & 0.918 $\pm$ 0.016 & 0.902 $\pm$ 0.019 & 0.903 $\pm$ 0.019 \\
 & GEG-EO & 0.061 $\pm$ 0.043 & 0.309 $\pm$ 0.193 & 0.157 $\pm$ 0.097 & 0.928 $\pm$ 0.014 & 0.933 $\pm$ 0.010 & 0.926 $\pm$ 0.014 & 0.926 $\pm$ 0.013 \\
\midrule
\multirow{6}{*}{\begin{sideways}Park\end{sideways}} & Baseline & 0.034 $\pm$ 0.022 & 0.200 $\pm$ 0.076 & 0.088 $\pm$ 0.041 & 0.854 $\pm$ 0.010 & 0.865 $\pm$ 0.007 & 0.851 $\pm$ 0.012 & 0.857 $\pm$ 0.010 \\
 & Blackbox & 0.187 $\pm$ 0.038 & 0.338 $\pm$ 0.100 & 0.231 $\pm$ 0.047 & 0.641 $\pm$ 0.032 & 0.658 $\pm$ 0.032 & 0.629 $\pm$ 0.034 & 0.634 $\pm$ 0.032 \\
 & DEMV & 0.032 $\pm$ 0.020 & 0.211 $\pm$ 0.067 & 0.092 $\pm$ 0.038 & 0.858 $\pm$ 0.017 & 0.870 $\pm$ 0.015 & 0.855 $\pm$ 0.017 & 0.861 $\pm$ 0.016 \\\cmidrule{2-9}
 & GEG-CP & 0.031 $\pm$ 0.020 & 0.206 $\pm$ 0.058 & 0.088 $\pm$ 0.035 & 0.853 $\pm$ 0.013 & 0.864 $\pm$ 0.011 & 0.850 $\pm$ 0.015 & 0.855 $\pm$ 0.013 \\
 & GEG-SP & 0.050 $\pm$ 0.041 & 0.220 $\pm$ 0.064 & 0.132 $\pm$ 0.041 & 0.824 $\pm$ 0.008 & 0.832 $\pm$ 0.007 & 0.823 $\pm$ 0.011 & 0.826 $\pm$ 0.008 \\
 & GEG-EO & 0.036 $\pm$ 0.021 & 0.206 $\pm$ 0.070 & 0.089 $\pm$ 0.037 & 0.853 $\pm$ 0.006 & 0.863 $\pm$ 0.003 & 0.851 $\pm$ 0.009 & 0.856 $\pm$ 0.006 \\
\midrule
\multirow{6}{*}{\begin{sideways}Wine\end{sideways}} & Baseline & 0.119 $\pm$ 0.047 & 0.187 $\pm$ 0.042 & 0.132 $\pm$ 0.027 & 0.710 $\pm$ 0.014 & 0.753 $\pm$ 0.063 & 0.550 $\pm$ 0.021 & 0.581 $\pm$ 0.029 \\
 & Blackbox & 0.108 $\pm$ 0.084 & 0.604 $\pm$ 0.075 & 0.479 $\pm$ 0.069 & 0.423 $\pm$ 0.065 & 0.342 $\pm$ 0.019 & 0.331 $\pm$ 0.019 & 0.315 $\pm$ 0.032 \\
 & DEMV & \textbf{0.085 $\pm$ 0.045} & 0.177 $\pm$ 0.060 & \textbf{0.103 $\pm$ 0.025} & 0.707 $\pm$ 0.012 & 0.746 $\pm$ 0.048 & 0.552 $\pm$ 0.017 & 0.584 $\pm$ 0.025 \\\cmidrule{2-9}
 & GEG-CP & 0.120 $\pm$ 0.041 & 0.180 $\pm$ 0.030 & 0.120 $\pm$ 0.026 & 0.705 $\pm$ 0.018 & 0.706 $\pm$ 0.072 & 0.547 $\pm$ 0.025 & 0.575 $\pm$ 0.034 \\
 & GEG-SP & \textbf{0.064 $\pm$ 0.043} & 0.199 $\pm$ 0.053 & \textbf{0.105 $\pm$ 0.026} & 0.706 $\pm$ 0.016 & 0.745 $\pm$ 0.087 & 0.546 $\pm$ 0.027 & 0.577 $\pm$ 0.039 \\
 & GEG-EO & 0.109 $\pm$ 0.040 & 0.184 $\pm$ 0.049 & \textit{0.118 $\pm$ 0.028} & 0.708 $\pm$ 0.012 & 0.747 $\pm$ 0.039 & 0.551 $\pm$ 0.022 & 0.582 $\pm$ 0.027 \\
        \bottomrule
        \end{tabular}
       
        }
\end{table*}

\begin{table*}[tb]
    \centering
    \caption{RQ$_4$: Results obtained with GB base classifier.}
    \label{tab:rq4_gb}
     \resizebox{\textwidth}{!}{
    \begin{tabular}{ll|ccc|cccc}
\toprule
& \textbf{Approach} & \textbf{SP} & \textbf{EO} & \textbf{AO} & \textbf{Accuracy} & \textbf{Precision} & \textbf{Recall} & \textbf{F1} \\
\midrule\midrule
\multirow{6}{*}{\begin{sideways}CMC\end{sideways}} & Baseline & 0.148 $\pm$ 0.082 & 0.014 $\pm$ 0.017 & 0.007 $\pm$ 0.008 & 0.996 $\pm$ 0.005 & 0.995 $\pm$ 0.006 & 0.996 $\pm$ 0.005 & 0.995 $\pm$ 0.005 \\
 & Blackbox & 0.138 $\pm$ 0.079 & 0.062 $\pm$ 0.090 & 0.039 $\pm$ 0.056 & 0.981 $\pm$ 0.023 & 0.980 $\pm$ 0.024 & 0.979 $\pm$ 0.026 & 0.979 $\pm$ 0.025 \\
 & DEMV & 0.148 $\pm$ 0.082 & 0.014 $\pm$ 0.017 & 0.007 $\pm$ 0.008 & 0.996 $\pm$ 0.005 & 0.995 $\pm$ 0.006 & 0.996 $\pm$ 0.005 & 0.995 $\pm$ 0.005 \\\cmidrule{2-9}
 & GEG-CP & 0.128 $\pm$ 0.087 & 0.143 $\pm$ 0.080 & 0.085 $\pm$ 0.035 & 0.790 $\pm$ 0.058 & 0.839 $\pm$ 0.031 & 0.832 $\pm$ 0.039 & 0.787 $\pm$ 0.057 \\
 & GEG-SP & 0.096 $\pm$ 0.066 & 0.197 $\pm$ 0.061 & 0.098 $\pm$ 0.031 & 0.886 $\pm$ 0.041 & 0.888 $\pm$ 0.035 & 0.901 $\pm$ 0.035 & 0.878 $\pm$ 0.044 \\
 & GEG-EO & 0.148 $\pm$ 0.082 & 0.014 $\pm$ 0.017 & 0.007 $\pm$ 0.008 & 0.996 $\pm$ 0.005 & 0.995 $\pm$ 0.006 & 0.996 $\pm$ 0.005 & 0.995 $\pm$ 0.005 \\
\midrule
\multirow{6}{*}{\begin{sideways}Crime\end{sideways}} & Baseline & 0.410 $\pm$ 0.028 & 0.574 $\pm$ 0.152 & 0.389 $\pm$ 0.070 & 0.491 $\pm$ 0.050 & 0.462 $\pm$ 0.054 & 0.471 $\pm$ 0.045 & 0.462 $\pm$ 0.049 \\
 & Blackbox & 0.392 $\pm$ 0.085 & 0.678 $\pm$ 0.128 & 0.576 $\pm$ 0.061 & 0.195 $\pm$ 0.043 & 0.188 $\pm$ 0.072 & 0.191 $\pm$ 0.042 & 0.165 $\pm$ 0.033 \\
 & DEMV & \textbf{0.335 $\pm$ 0.051} & \textit{0.494 $\pm$ 0.138} & \textbf{0.313 $\pm$ 0.086} & 0.474 $\pm$ 0.035 & 0.462 $\pm$ 0.032 & 0.457 $\pm$ 0.035 & 0.454 $\pm$ 0.032 \\\cmidrule{2-9}
 & GEG-CP & \textbf{0.181 $\pm$ 0.063} & \textbf{0.371 $\pm$ 0.211} & \textbf{0.215 $\pm$ 0.075} & 0.374 $\pm$ 0.043 & 0.416 $\pm$ 0.044 & 0.349 $\pm$ 0.029 & 0.328 $\pm$ 0.037 \\
 & GEG-SP & \textbf{0.055 $\pm$ 0.027} & \textbf{0.385 $\pm$ 0.079} & \textbf{0.223 $\pm$ 0.038} & 0.428 $\pm$ 0.028 & 0.434 $\pm$ 0.030 & 0.410 $\pm$ 0.021 & 0.406 $\pm$ 0.025 \\
 & GEG-EO & \textbf{0.331 $\pm$ 0.041} & \textit{0.484 $\pm$ 0.188} & \textbf{0.314 $\pm$ 0.094} & 0.483 $\pm$ 0.043 & 0.458 $\pm$ 0.042 & 0.463 $\pm$ 0.037 & 0.454 $\pm$ 0.039 \\
\midrule
\multirow{6}{*}{\begin{sideways}Drug\end{sideways}} & Baseline & 0.168 $\pm$ 0.128 & 0.357 $\pm$ 0.112 & 0.220 $\pm$ 0.076 & 0.676 $\pm$ 0.032 & 0.606 $\pm$ 0.033 & 0.601 $\pm$ 0.028 & 0.599 $\pm$ 0.030 \\
 & Blackbox & 0.335 $\pm$ 0.145 & 0.566 $\pm$ 0.243 & 0.436 $\pm$ 0.170 & 0.422 $\pm$ 0.102 & 0.389 $\pm$ 0.089 & 0.401 $\pm$ 0.061 & 0.360 $\pm$ 0.081 \\
 & DEMV & 0.112 $\pm$ 0.076 & 0.393 $\pm$ 0.149 & 0.208 $\pm$ 0.073 & 0.675 $\pm$ 0.031 & 0.607 $\pm$ 0.046 & 0.596 $\pm$ 0.035 & 0.597 $\pm$ 0.039 \\\cmidrule{2-9}
 & GEG-CP & \textit{0.085 $\pm$ 0.078} & 0.336 $\pm$ 0.158 & \textbf{0.169 $\pm$ 0.071} & 0.657 $\pm$ 0.042 & 0.579 $\pm$ 0.046 & 0.577 $\pm$ 0.039 & 0.573 $\pm$ 0.042 \\
 & GEG-SP & \textbf{0.058 $\pm$ 0.040} & \textbf{0.274 $\pm$ 0.104} & \textbf{0.134 $\pm$ 0.060} & 0.677 $\pm$ 0.027 & 0.602 $\pm$ 0.032 & 0.596 $\pm$ 0.024 & 0.595 $\pm$ 0.027 \\
 & GEG-EO & 0.111 $\pm$ 0.070 & 0.323 $\pm$ 0.188 & 0.171 $\pm$ 0.090 & 0.672 $\pm$ 0.041 & 0.601 $\pm$ 0.047 & 0.596 $\pm$ 0.039 & 0.593 $\pm$ 0.042 \\
\midrule
\multirow{6}{*}{\begin{sideways}Law\end{sideways}} & Baseline & 0.134 $\pm$ 0.024 & 0.002 $\pm$ 0.002 & 0.001 $\pm$ 0.001 & 1.000 $\pm$ 0.000 & 1.000 $\pm$ 0.000 & 1.000 $\pm$ 0.000 & 1.000 $\pm$ 0.000 \\
 & Blackbox & 0.138 $\pm$ 0.024 & 0.032 $\pm$ 0.031 & 0.016 $\pm$ 0.016 & 0.998 $\pm$ 0.001 & 0.998 $\pm$ 0.001 & 0.998 $\pm$ 0.001 & 0.998 $\pm$ 0.001 \\
 & DEMV & 0.134 $\pm$ 0.024 & 0.002 $\pm$ 0.002 & 0.001 $\pm$ 0.001 & 1.000 $\pm$ 0.000 & 1.000 $\pm$ 0.000 & 1.000 $\pm$ 0.000 & 1.000 $\pm$ 0.000 \\\cmidrule{2-9}
 & GEG-CP & \textbf{0.057 $\pm$ 0.030} & 0.083 $\pm$ 0.024 & 0.045 $\pm$ 0.012 & 0.942 $\pm$ 0.006 & 0.948 $\pm$ 0.005 & 0.953 $\pm$ 0.005 & 0.946 $\pm$ 0.006 \\
 & GEG-SP & \textbf{0.034 $\pm$ 0.021} & 0.203 $\pm$ 0.037 & 0.102 $\pm$ 0.019 & 0.978 $\pm$ 0.005 & 0.978 $\pm$ 0.005 & 0.980 $\pm$ 0.005 & 0.978 $\pm$ 0.005 \\
 & GEG-EO & 0.134 $\pm$ 0.024 & 0.002 $\pm$ 0.002 & 0.001 $\pm$ 0.001 & 1.000 $\pm$ 0.000 & 1.000 $\pm$ 0.000 & 1.000 $\pm$ 0.000 & 1.000 $\pm$ 0.000 \\
\midrule
\multirow{6}{*}{\begin{sideways}Obesity\end{sideways}} & Baseline & 0.063 $\pm$ 0.042 & 0.195 $\pm$ 0.118 & 0.108 $\pm$ 0.065 & 0.950 $\pm$ 0.014 & 0.948 $\pm$ 0.014 & 0.948 $\pm$ 0.014 & 0.947 $\pm$ 0.014 \\
 & Blackbox & 0.093 $\pm$ 0.072 & 0.468 $\pm$ 0.118 & 0.273 $\pm$ 0.074 & 0.739 $\pm$ 0.062 & 0.751 $\pm$ 0.066 & 0.727 $\pm$ 0.065 & 0.721 $\pm$ 0.069 \\
 & DEMV & 0.066 $\pm$ 0.040 & 0.197 $\pm$ 0.111 & 0.105 $\pm$ 0.063 & 0.945 $\pm$ 0.014 & 0.944 $\pm$ 0.015 & 0.944 $\pm$ 0.015 & 0.943 $\pm$ 0.015 \\\cmidrule{2-9}
 & GEG-CP & 0.063 $\pm$ 0.042 & 0.195 $\pm$ 0.118 & 0.108 $\pm$ 0.065 & 0.950 $\pm$ 0.014 & 0.948 $\pm$ 0.014 & 0.948 $\pm$ 0.014 & 0.947 $\pm$ 0.014 \\
 & GEG-SP & 0.051 $\pm$ 0.026 & 0.205 $\pm$ 0.113 & 0.111 $\pm$ 0.062 & 0.927 $\pm$ 0.021 & 0.929 $\pm$ 0.019 & 0.927 $\pm$ 0.020 & 0.925 $\pm$ 0.021 \\
 & GEG-EO & 0.063 $\pm$ 0.042 & 0.195 $\pm$ 0.118 & 0.108 $\pm$ 0.065 & 0.950 $\pm$ 0.014 & 0.948 $\pm$ 0.014 & 0.948 $\pm$ 0.014 & 0.947 $\pm$ 0.014 \\
\midrule
\multirow{6}{*}{\begin{sideways}Park\end{sideways}} & Baseline & 0.045 $\pm$ 0.022 & 0.173 $\pm$ 0.055 & 0.068 $\pm$ 0.027 & 0.867 $\pm$ 0.012 & 0.880 $\pm$ 0.010 & 0.864 $\pm$ 0.013 & 0.869 $\pm$ 0.012 \\
 & Blackbox & 0.188 $\pm$ 0.039 & 0.302 $\pm$ 0.120 & 0.207 $\pm$ 0.055 & 0.658 $\pm$ 0.042 & 0.674 $\pm$ 0.042 & 0.645 $\pm$ 0.043 & 0.651 $\pm$ 0.043 \\
 & DEMV & \textit{0.032 $\pm$ 0.017} & 0.186 $\pm$ 0.060 & 0.081 $\pm$ 0.031 & \textit{0.870 $\pm$ 0.013} & \textit{0.885 $\pm$ 0.010} & 0.866 $\pm$ 0.015 & 0.872 $\pm$ 0.013 \\\cmidrule{2-9}
 & GEG-CP & 0.075 $\pm$ 0.033 & 0.166 $\pm$ 0.041 & 0.057 $\pm$ 0.022 & 0.847 $\pm$ 0.019 & 0.858 $\pm$ 0.021 & 0.849 $\pm$ 0.015 & 0.851 $\pm$ 0.017 \\
 & GEG-SP & 0.038 $\pm$ 0.017 & 0.189 $\pm$ 0.072 & 0.083 $\pm$ 0.040 & 0.871 $\pm$ 0.013 & 0.885 $\pm$ 0.010 & 0.867 $\pm$ 0.014 & 0.873 $\pm$ 0.012 \\
 & GEG-EO & 0.090 $\pm$ 0.035 & 0.250 $\pm$ 0.068 & 0.133 $\pm$ 0.027 & 0.668 $\pm$ 0.018 & 0.730 $\pm$ 0.014 & 0.652 $\pm$ 0.013 & 0.647 $\pm$ 0.018 \\
\midrule
\multirow{6}{*}{\begin{sideways}Wine\end{sideways}} & Baseline & 0.116 $\pm$ 0.037 & 0.207 $\pm$ 0.045 & 0.142 $\pm$ 0.021 & 0.605 $\pm$ 0.013 & 0.566 $\pm$ 0.043 & 0.453 $\pm$ 0.020 & 0.475 $\pm$ 0.026 \\
 & Blackbox & \textbf{0.081 $\pm$ 0.049} & 0.568 $\pm$ 0.107 & 0.462 $\pm$ 0.062 & 0.309 $\pm$ 0.072 & 0.294 $\pm$ 0.022 & 0.294 $\pm$ 0.032 & 0.247 $\pm$ 0.036 \\
 & DEMV & \textbf{0.071 $\pm$ 0.045} & \textbf{0.148 $\pm$ 0.045} & \textbf{0.090 $\pm$ 0.022} & 0.593 $\pm$ 0.018 & 0.516 $\pm$ 0.046 & 0.432 $\pm$ 0.023 & 0.448 $\pm$ 0.029 \\\cmidrule{2-9}
 & GEG-CP & \textbf{0.064 $\pm$ 0.043} & 0.194 $\pm$ 0.065 & \textbf{0.111 $\pm$ 0.032} & 0.602 $\pm$ 0.018 & 0.563 $\pm$ 0.040 & 0.451 $\pm$ 0.014 & 0.474 $\pm$ 0.019 \\
 & GEG-SP & \textbf{0.040 $\pm$ 0.017} & 0.290 $\pm$ 0.098 & 0.167 $\pm$ 0.053 & 0.601 $\pm$ 0.015 & 0.571 $\pm$ 0.050 & 0.455 $\pm$ 0.016 & 0.480 $\pm$ 0.023 \\
 & GEG-EO & \textbf{0.049 $\pm$ 0.027} & \textbf{0.168 $\pm$ 0.074} & \textbf{0.090 $\pm$ 0.039} & 0.597 $\pm$ 0.018 & 0.541 $\pm$ 0.043 & 0.438 $\pm$ 0.018 & 0.458 $\pm$ 0.023 \\
\bottomrule
\end{tabular}}
\end{table*}

\begin{figure*}[ht]
    \centering
    \includegraphics[width=\linewidth]{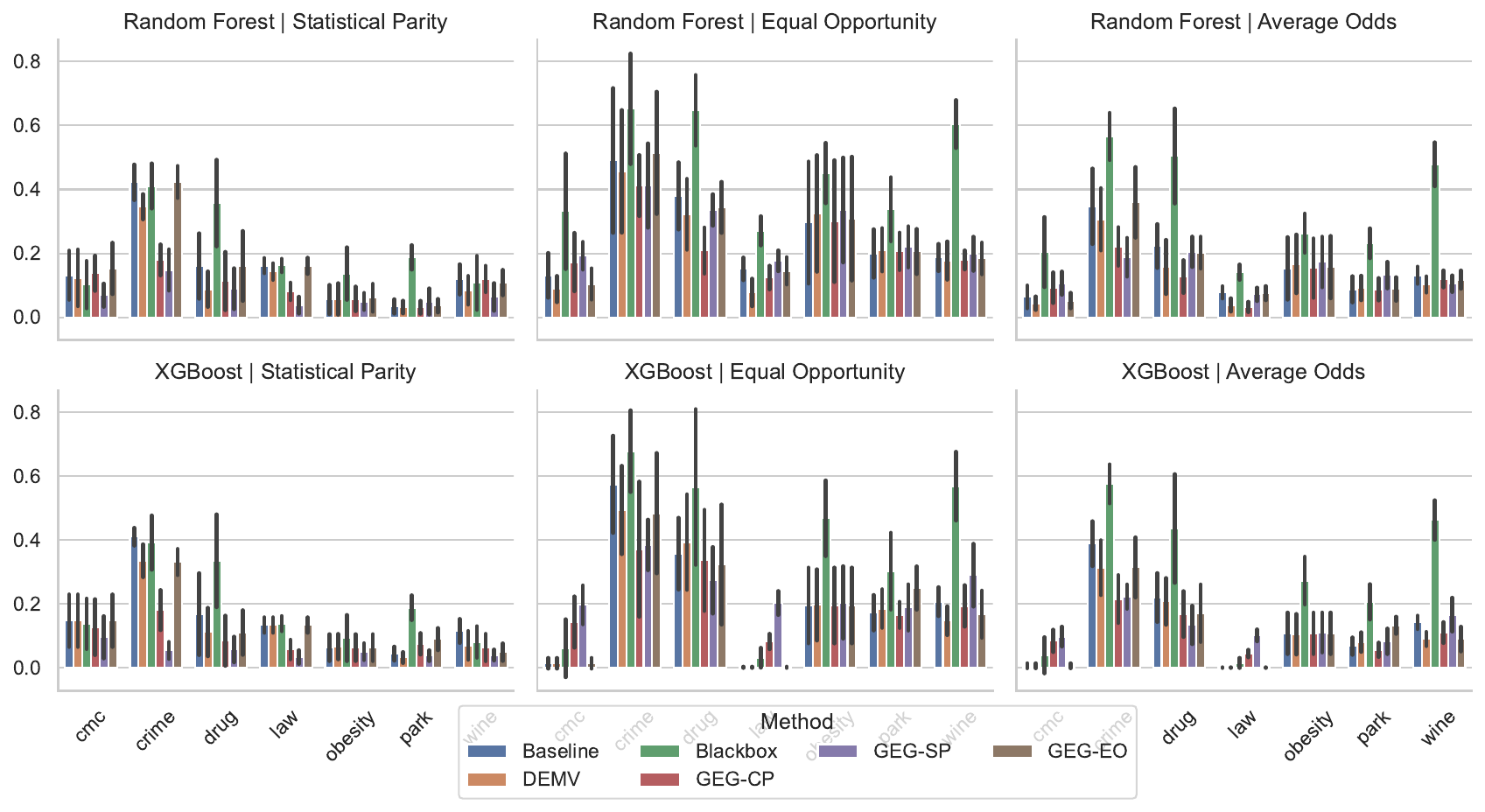}
    \caption{RQ5: Comparison of fairness scores with more complex classifiers}
    \label{fig:rq5_fairness}
\end{figure*}

\begin{figure*}[ht]
    \centering
    \includegraphics[width=\linewidth]{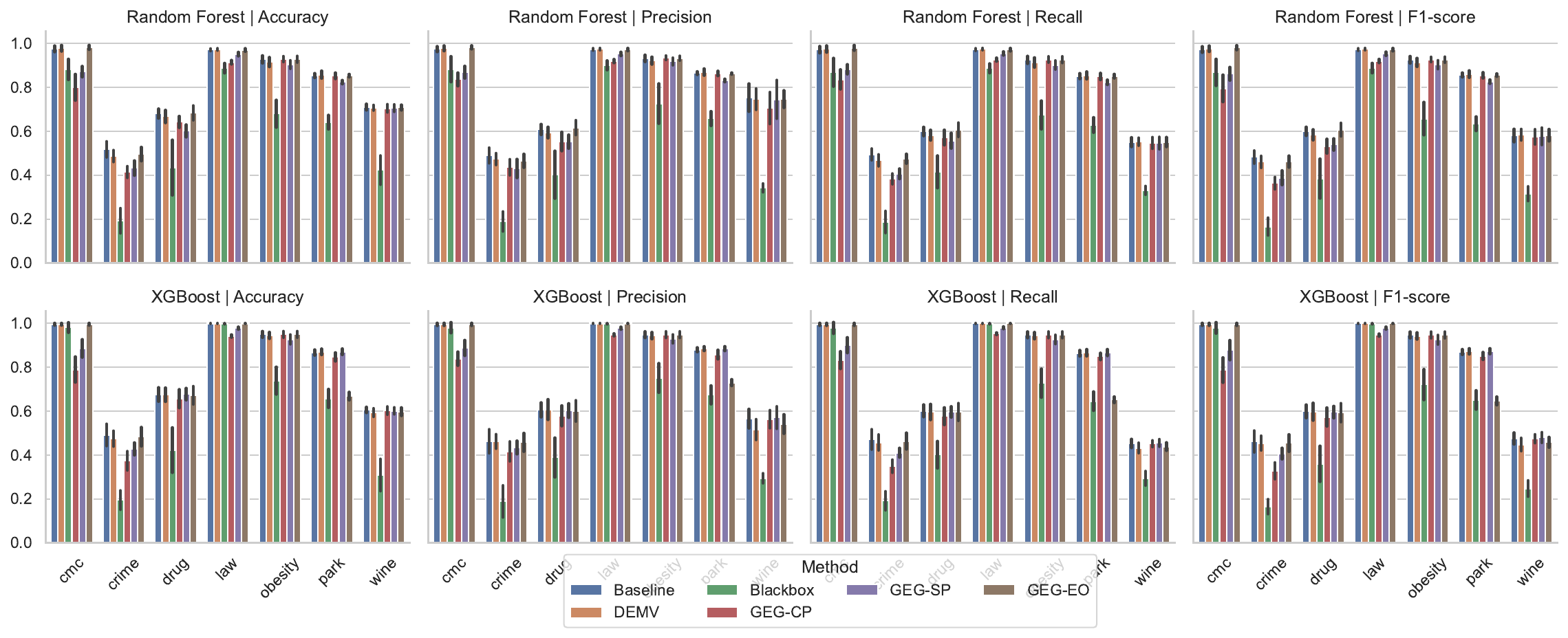}
    \caption{RQ5: Comparison of effectiveness scores with more complex classifiers}
    \label{fig:rq5_effectiveness}
\end{figure*}

\begin{figure*}[ht]
    \centering
    \includegraphics[width=\linewidth]{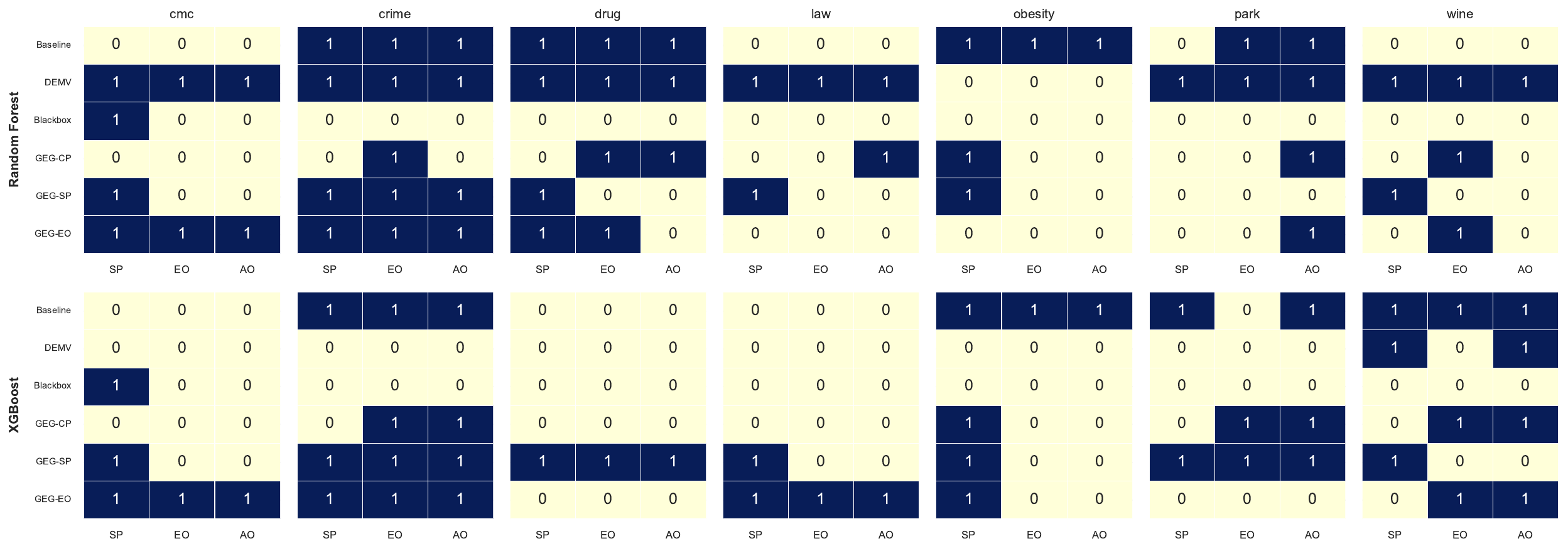}
    \caption{RQ5: Per-dataset Pareto optimality results}
    \label{fig:rq5_pareto_optimality_dataset}
\end{figure*}


\end{document}

%% file: commands.tex
\usepackage{tikz,pifont}
\usepackage{amsfonts,amsmath}
\usepackage{hyperref}
\usepackage{xspace}
\usepackage{multirow}
\usepackage{makecell}
\usepackage{enumitem}
\usepackage{booktabs}
\usepackage[many]{tcolorbox}
\usepackage[table]{xcolor}

\usepackage[font=footnotesize]{subcaption}

\newboolean{showrevision}
\setboolean{showrevision}{false}
\ifthenelse{\boolean{showrevision}}
{
	\newcommand\revised[1]{\textcolor{blue}{#1}}
}
{
	\newcommand\revised[1]{{#1}}
}

\newtcolorbox{rqanswer}{
    colback=gray!25, 
    colframe=black, 
    enhanced,
    boxrule=0mm,
    frame hidden,
    left=1mm, 
    top=1mm, bottom=1mm, right=1mm, 
    borderline west={0.8mm}{0mm}{gray!80!black}, 
    sharp corners, 
    fontupper=\small 
}



%% file: acks.tex
\newcommand{\TAAck}{Territori Aperti (a project funded by Fondo Territori, Lavoro e Conoscenza CGIL CISL UIL)\xspace}

\newcommand{\SoBigDataITAck}{European Union - NextGenerationEU - National Recovery and Resilience Plan (Piano Nazionale di Ripresa e Resilienza, PNRR) - Project: “SoBigData.it - Strengthening the Italian RI for Social Mining and Big Data Analytics” - Prot. IR0000013 - Avviso n. 3264 del 28/12/2021\xspace}

\newcommand{\FringeAck}{European Union – NextGenerationEU through the Italian Ministry of University and Research, Projects PRIN 2022 PNRR “FRINGE: context-aware FaiRness engineerING in complex software systEms" grant n.P2022553SL\xspace}

%% file: sections/introduction.tex
With the increasing adoption of AI- and ML-based software systems in sensitive domains such as healthcare \cite{canali_challenges_2022}, finance \cite{kozodoi_fairness_2022}, and education \cite{austin2016will}, it is critical to ensure that they act in an \textit{unbiased} and \textit{ethical} way. In other words, they must be \emph{fair}. The relevance of \emph{software fairness} has been highlighted in recent years not only in research literature \cite{mehrabi_survey_2021,caton_fairness_2023,chen_fairness_2022}, but also in regulations such as the European Union's recently introduced AI Act~\cite{noauthor_eu_2023}.

\emph{Fairness} is defined as: \textit{``The absence of prejudice and favouritism of a software system toward individuals or groups"}~\cite{mehrabi_survey_2021}. When a system behaves unfairly, it is said to be \emph{biased}. When not properly addressed, \emph{software bias} may cause severe consequences and discrimination, like for the recruitment instrument employed by Amazon, which penalised women candidates for IT job positions,\footnote{\url{https://www.bbc.co.uk/news/technology-45809919}} or the criminal recidivism predictions made by the commercial risk assessment software COMPAS, which misjudged black individuals based on biased profiling \cite{angwin2016machine}. 

Bias can originate from three main sources~\cite{mehrabi_survey_2021}: the \textbf{data} used to train the AI and ML components, a \textbf{biased implementation} of the AI and ML components, and the \textbf{people} that interact with those components. In this paper, we focus on mitigating bias in an ML model trained on biased data (i.e., \textbf{data} bias). More specifically, this work focuses on the classification task using structured, tabular data. 

The research community has proposed several methods for \textit{bias mitigation} at different processing levels \cite{mehrabi_survey_2021,caton_fairness_2023}. However, the majority of them can be applied only to binary classification tasks. Instead, several examples of multi-class classification approaches have been applied in sensitive domains such as education \cite{baskota2018graduate,yanes2020machine}, food \cite{suchithra_improving_2018,meenachi_multi_2022}, and health \cite{zhang_application_2013}. Ensuring that these systems behave in a \emph{fair} and unbiased way is paramount, also to achieve some of the United Nations Sustainable Development Goals, (SDG)\footnote{\url{https://sdgs.un.org/goals}} e.g., SDG 2 (zero hunger) \cite{suchithra_improving_2018,meenachi_multi_2022}, SDG 3 (good health and well-being) {\cite{zhang_application_2013}}, and SDG 4 (quality education) {\cite{baskota2018graduate,yanes2020machine}}.

\subsection{Research Objectives}

To address the lack of bias mitigation for multi-class classification, in this paper, we first formulate the problem of fairness in multi-class classification as a multi-objective problem between effectiveness (i.e., prediction correctness) and multiple fairness definitions. Next, we propose a Generalised Exponentiated Gradient (GEG) algorithm to solve this optimisation problem. \tool is an extension of the original Exponentiated Gradient (EG) bias mitigation algorithm first proposed by Agarwal et al. for binary classification~\cite{agarwal_reductions_2018}. However, \tool differs from the original EG approach in two aspects: first, it can mitigate bias in both binary and multi-class classification tasks; second, it mitigates bias across multiple fairness constraints simultaneously. These capabilities make \tool more flexible and practical for multiple use cases \revised{(see Section \ref{sec:discussion})}.

We perform an extensive evaluation of \tool, benchmarking it against six approaches across seven multi-class and three binary datasets, using four widely adopted effectiveness metrics and three fairness definitions. Results show that \tool is an efficient approach for bias mitigation in multi-class classification, achieving fairness improvement up to 92\%, overcoming existing baselines. Additionally, from our empirical evaluation, we draw practical tips on employing \tool in real-case scenarios.

\subsection{Main Contributions}

The main contributions of our work are the following:

\begin{itemize}
    \item We formulate the problem of fairness in multi-class classification as a multi-objective problem between multiple fairness definitions and prediction effectiveness.

    \item We propose a Generalised Exponentiated Gradient (GEG) approach to mitigate bias in binary and multi-class classification tasks under multiple fairness definitions simultaneously.

    \item We perform an extensive empirical evaluation of \tool against multiple baselines, datasets, and metrics.

    \item We draw a set of \revised{theoretical and} practical insights on adopting \tool in real-case scenarios. 

    \item We release a replication package including a Python implementation of \tool and the results of our empirical evaluation to foster future research. \replpackage
\end{itemize}

\subsection{Roadmap}

The rest of this paper is structured as follows: Section \ref{sec:background} provides background knowledge on fairness and discusses related work. Section \ref{sec:methodology} presents the fairness learning in multi-class classification as a multi-objective optimisation problem and describes the \tool approach. Section \ref{sec:evaluation} describes the empirical evaluation performed, while Section \ref{sec:results} discusses the obtained results and provides practical insights. \revised{Section \ref{sec:discussion} discusses theoretical and practical implications drawn from the empirical evaluation of \tool.} Finally, Section \ref{sec:conclusion} discusses future work and concludes the study.

%% file: sections/rw.tex
\revised{In the following, we provide background knowledge in the context of bias assessment and mitigation (Section \ref{sec:bias}, and place our contribution within related work on bias mitigation methods (Section \ref{sec:rw}) and the original Exponentiated Gradient Approach (Section \ref{sec:eg})}.

\subsection{Background Knowledge on Bias Assessment and Mitigation}\label{sec:bias}

In general, the fairness of an ML model can be assessed following two main criteria: \textbf{individual} and \textbf{group} fairness~\cite{speicher2018unified}. 
\textbf{Individual} fairness requires that two individuals who are similar to one another receive the same treatment (i.e., an ML model should make identical predictions). 
Most of the time, two individuals are treated as similar if they only differ in sensitive attributes\footnote{In the rest of this paper, we will use the terms "sensitive attributes", "protected attributes" or "sensitive features" as synonyms.} (e.g., \textit{ethnicity}, \textit{gender}, \textit{age}).
\textbf{Group} fairness, on the other hand, addresses fairness by treating population groups, defined by protected attributes (like \textit{ethnicity}, \textit{gender}, or \textit{age}), equally. In this work, we focus on group fairness criteria, as they are more common and have been more extensively addressed in previous work~\cite{hort2024bias,d2025fair}. Specifically, many group fairness definitions and corresponding metrics have been proposed in the literature \cite{mehrabi_survey_2021,caton_fairness_2023}. The general idea behind all group fairness definitions is that, given two groups named \textbf{privileged} and \textbf{unprivileged} (e.g., \textit{men} and \textit{women}), they must have the same probability of having a given \textbf{positive} outcome from the ML model, possibly conditioned on the ground truth label~\cite{mehrabi_survey_2021}. In Sections \ref{sec:methodology} and \ref{sec:evaluation}, we provide the formal specification of the fairness definitions we address in this paper.  

In addition to measuring bias, research has proposed several methods for mitigating bias at different processing levels~\cite{hort2024bias,mehrabi_survey_2021}. Generally, improvement in fairness implies a reduction in the effectiveness of a model's predictions \cite{chen_fairness_2022,hort2024bias,hort2021fairea}, and all bias mitigation methods try to identify the optimal trade-off between fairness and effectiveness.  
In particular, there are three main categories of bias mitigation methods based on when they are applied in an ML workflow: \textbf{pre-processing}, \textbf{in-processing}, and \textbf{post-processing}~\cite{friedler2019comparative,hort2024bias,mehrabi_survey_2021}. 
\textbf{Pre-processing} bias mitigation methods aim to reduce bias by applying changes to the training data.
For instance, one can assign more weight to data instances for a population group that is prone to being misclassified~\cite{kamiran_data_2012,daloisio_debiaser_2023}.
\textbf{In-processing} bias mitigation methods make changes to the design and training process of ML models to achieve fairness. One example is the inclusion of fairness metrics as part of the training loss~\cite{beutel2019putting,agarwal_reductions_2018}. 
Alternatives include the tuning of hyperparameters~\cite{tizpaz2022fairness} or the use of ensembles, where each model can consider different population groups~\cite{dwork_fairness_2012} or metrics~\cite{chen2022maat}. 
\textbf{Post-processing} bias mitigation methods are applied once an ML model has been successfully trained. This can involve changes to the model's predictions~\cite {hardt2016equality} or modifications to the model itself~\cite{kamiran2010discrimination}. 

\subsection{Related Work}\label{sec:rw}

The majority of bias mitigation methods proposed in the literature focus on binary classification tasks \cite{hort2024bias,mehrabi_survey_2021}. Although recent contributions address specific limitations, such as automatically inferring sensitive variables \cite{CAI2026104469,daloisio_towards_2025},
fairness in graph representations \cite{ZHANG2024103570,WANG2024103682}, or fairness in recommendations \cite{MA2024103750,BOUANANE2024103865}, very few studies address fairness in multi-class classification. One of the first approaches proposed for multi-class bias mitigation is the \textit{Blackbox} post-processing approach by Putzel et al.  \cite{putzel_blackbox_2022}, which extends the \textit{Equalized Odds} algorithm \cite{hardt2016equality} to the multi-class setting. This algorithm builds a linear optimisation program that optimises the predictions of an already trained classifier to satisfy the \textit{Equalized Odds} fairness definition for multi-class settings. A similar approach is the \textit{Demographic Parity} post-processing approach proposed by Denis et al.~\cite{denis2024fairness}, where the predictions are instead optimised under the \textit{Demographic Parity} fairness definition.
Quadros et al.~\cite{quadros2025multi} adapted a series of pre-processing and post-processing bias mitigation methods for binary classification to the multi-class classification task and tested them on a wage discrimination dataset.
Finally, one of the most recent approaches specifically designed for multi-class bias mitigation is the pre-processing \textit{Debiaser for Multiple Variables (DEMV)} algorithm proposed by d'Aloisio et al.~\cite{daloisio_debiaser_2023}. This algorithm extends the \textit{Sampling} method of Kamiran et al.~\cite{kamiran_data_2012} to the multi-class setting and has been shown to overcome existing bias mitigation methods for multi-class classification.

Our proposed approach differs from the previous ones in that it is an in-processing bias mitigation method. \revised{Indeed, having a diverse set of bias mitigation methods (i.e., pre-processing, in-processing, and post-processing methods) is a needed requirement due to the diverse scenarios in which bias must be mitigated \cite{friedler2019comparative}. Additionally, in-processing approaches could subsequently be chained with post-processing approaches to further improve a model's fairness.}

In particular, our work extends the \textit{Exponentiated Gradient} in-processing algorithm proposed by Agarwal et al.~\cite{agarwal_reductions_2018}. In their work, the authors formulate a multi-objective optimisation problem to train a binary classifier under specific fairness constraints. Next, they present an Exponentiated Gradient (EG) method to solve this optimisation task. Our proposed approach extends the original EG algorithm to the multi-class classification setting and to the simultaneous optimisation of multiple fairness constraints, making it more general and practical for real-world use cases. \revised{We describe the original EG approach in Section \ref{sec:eg}, while we detail our novel contribution in Section \ref{sec:methodology}.}

\subsection{\revised{Exponentiated Gradient Approach}}\label{sec:eg}

\revised{Agarwal et al.~\cite{agarwal_reductions_2018} proposed the Exponentiated Gradient (EG) approach as an in-processing framework for learning fair classifiers under explicit fairness constraints. The main objective is to learn a classifier that minimises the prediction error while satisfying group fairness conditions with respect to a protected attribute. Unlike pre-processing or post-processing strategies, fairness is directly enforced during the training process.}

\revised{The learning problem is formulated in a binary classification setting, where training data consist of triplets $(X, A, Y)$, with $X \in \mathcal{X}$ denoting the input features, $A \in \mathcal{A}$ representing a protected attribute (such as gender or race), and $Y \in \{0,1\}$ being the ground-truth label. The goal is to learn a classifier $h : \mathcal{X} \rightarrow \{0,1\}$ from a hypothesis class $\mathcal{H}$ that achieves high predictive accuracy while respecting fairness constraints.}

\revised{Fairness is defined at the group level, meaning that predictions should satisfy parity conditions across sub-populations defined by the protected attribute. The EG approach focuses specifically on two fairness definitions, namely Demographic Parity \cite{kamiran_data_2012} and Equalized Odds \cite{hardt2016equality}, which impose equality constraints on conditional probabilities of predictions across groups.}

\revised{
Specifically, Demographic Parity requires that the probability of predicting the positive class is independent of the protected attribute \cite{kamiran_data_2012}, while Equalized Odds requires that the prediction is conditionally independent of the protected attribute given the true label \cite{hardt2016equality}. Both definitions can be rewritten as linear constraints on conditional expectations, also referred to as statistical moments of the classifier \cite{agarwal_reductions_2018}.}

\revised{
Under this formulation, the fair learning problem is written as a constrained optimisation problem where the classification error $\mathcal{R}(h) = \mathbb{P}(h(X) \neq Y)$ is minimised subject to linear constraints on conditional moments of the classifier \cite{agarwal_reductions_2018}.
}

\revised{To address the constrained learning problem, Agarwal et al.~\cite{agarwal_reductions_2018} introduce a convex relaxation based on randomized classifiers. Instead of searching for a single deterministic predictor $h \in \mathcal{H}$, the learning task is reformulated in terms of a probability distribution $Q \in \Delta(\mathcal{H})$ over the hypothesis class. Under this relaxation, both the classification risk and the fairness constraints become linear with respect to $Q$, transforming the original non-convex problem into a convex optimisation problem over the simplex $\Delta(\mathcal{H})$. In practice, all expectations are replaced by empirical averages computed over the training set.}

\revised{To solve the constrained problem, the authors rely on a Lagrangian relaxation, reformulating the task as a saddle-point problem over distributions $Q \in \Delta(\mathcal{H})$ and dual variables $\boldsymbol{\lambda}$. This formulation can be interpreted as a two-player zero-sum game between a learner and an auditor. At each iteration, the learner computes a best response by solving a cost-sensitive classification problem, where each training sample is assigned a weight reflecting both its contribution to the prediction error and its impact on the fairness constraints. The auditor then updates the dual variables using an exponentiated-gradient rule, increasing the weight of the most violated constraints. Over successive iterations, the algorithm builds a sparse mixture of classifiers. Convergence is assessed through a duality-gap criterion; when this gap falls below a predefined threshold, the algorithm terminates.}

\revised{The original Exponentiated Gradient framework is designed for binary classification and is typically instantiated with a single fairness definition. In Section~\ref{sec:methodology}, we build upon this reduction-based principle to support decision-level fairness in multi-class prediction spaces and the simultaneous enforcement of linear fairness constraints, while preserving the convex optimisation structure required by the EG scheme.}

%% file: sections/methodology.tex

\revised{
In this section, we present a general in-processing 
framework for learning fair classifiers that supports 
both binary and multi-class classification and enables 
the simultaneous enforcement of multiple fairness 
constraints. Our approach builds on the reduction-based 
framework of Agarwal et al.~\cite{agarwal_reductions_2018} 
(recalled in Section~\ref{sec:eg}) introducing 
three main extensions. First, we define 
general label variants of Demographic Parity and 
Equalized Odds that enforce fairness across all classes 
$y_k \in \mathcal{Y}$ simultaneously, without requiring 
the specification of a positive label, addressing 
settings where no single favourable class can be 
naturally identified. Second, we derive the corresponding 
linear moment constraints. Third, we introduce a combined 
constraint system that enables the simultaneous enforcement 
of multiple fairness requirements within a unified 
optimisation framework. These extensions enhance practical 
applicability in real-world scenarios where multi-class 
outcomes and multiple fairness criteria naturally arise.}



Consider training data as triplets $(X,A,Y)$, where 
$X\in\mathcal{X}$ denotes the input features, 
$A\in\mathcal{A}$ is a protected attribute (e.g., 
\textit{gender} or \textit{race}), and $Y\in\mathcal{Y}$ 
is the output label. We aim to learn a classifier 
$h:\mathcal{X}\to\mathcal{Y}$ from a hypothesis class 
$\mathcal{H}$ that achieves high predictive accuracy 
while satisfying fairness constraints. We assume that 
the true labels $Y$ and predictions $h(X)$ lie in the 
same space $\mathcal{Y} = \{y_0, y_1, \dots, y_K\}$, so 
that $|\mathcal{Y}| = K+1$, and denote by $y_k \in \mathcal{Y}$ 
any class label.

A common strategy for ensuring fairness during training is to incorporate fairness as explicit constraints in the learning objective~\cite{agarwal_reductions_2018,radovanovic2020enforcing}. This leads to the constrained optimisation problem
\begin{equation}\label{eq1}
    \min_{h \in \mathcal{H}} \quad \mathcal{R}(h)
    \quad \text{subject to} \quad
    \gamma_i(h) \leq \epsilon_i, \quad i = 1, \dots, n,
\end{equation}
where $\mathcal{R}(h)=\mathbb{P}(h(X)\neq Y)$ is the classification error and each $\gamma_i(h)$ quantifies the violation of the $i$-th fairness constraint, with tolerance level $\epsilon_i$.
We detail the specific fairness constraints considered in the next subsection.

\subsection{Fairness Constraints}
In the context of fairness constraints, the literature 
has proposed several group fairness notions~\cite{mehrabi_survey_2021,hort2024bias,d2025fair}, 
each capturing different requirements across groups 
defined by the protected attribute $A$. 
Here, we introduce the \textit{general label} 
formulation of Demographic Parity and Equalized Odds, 
which enforces fairness across all classes 
$y_k \in \mathcal{Y}$ simultaneously, without requiring 
the specification of a single favourable class.

\begin{definition}[General Label Demographic Parity]
\label{def:gldp}
A classifier $h$ satisfies \textit{General Label 
Demographic Parity} if all groups defined by the 
protected attribute $A$ have the same probability 
of being assigned any possible prediction $y_k$. 
In mathematical terms, it can be expressed as:
\begin{equation}
\begin{aligned}
\mathbb{P}(h(X) = y_k \mid A = a) = 
\mathbb{P}(h(X) = y_k), \\ 
\;\; \forall a \in \mathcal{A},\ \forall y_k \in \mathcal{Y}.
\end{aligned}
\end{equation}
\end{definition}

\begin{definition}[General Label Equalized Odds]
\label{def:gleo}
A classifier $h$ meets the \textit{General Label 
Equalized Odds} fairness definition when the probability 
of predicting any class label is the same across all 
groups determined by the protected attribute $A$, 
given a value $y$ of the true label $Y$. In mathematical 
terms, this condition means:
\begin{equation}
\begin{aligned}
\mathbb{P}(h(X)=y_k \mid Y = y, A = a) 
= 
\mathbb{P}(h(X)=y_k \mid Y = y), \\
\quad \forall a \in \mathcal{A},\ \forall y, y_k \in \mathcal{Y}.
\end{aligned}
\end{equation}
\end{definition}

Note that these definitions recover the standard binary 
fairness definitions when $\mathcal{Y} = \{0,1\}$, since 
the two constraints for $y_0$ and $y_1$ are equivalent 
given that $\mathbb{P}(h(X)=0) + \mathbb{P}(h(X)=1) = 1$.

In the remainder of this section, we present the 
general label formulation (Definitions~\ref{def:gldp} 
and~\ref{def:gleo}) as training constraints. 

Fairness constraints are often reworded with expectations 
to make them easier to integrate into optimisation problems. 
Following the approach of 
\cite{agarwal_reductions_2018} for binary classification, 
we write the general label variants as multi-class linear 
constraints with expected values, where 
$h(X) \in \mathcal{Y}$. 
We apply indicator functions to each class 
$y_k \in \mathcal{Y}$ simultaneously. The chance 
of predicting $y_k$ becomes
\begin{equation}
    \mathbb{P}(h(X) = y_k) = 
    \mathbb{E}[\mathbf{1}_{\{h(X) = y_k\}}],
    \quad \forall y_k \in \mathcal{Y},
\end{equation}
with the indicator function defined as
\[
\mathbf{1}_{\{h(X) = y_k\}} =
\begin{cases}
1 & \text{if } h(X) = y_k, \\
0 & \text{otherwise}.
\end{cases}
\]
The General Label Demographic Parity constraint becomes
\begin{equation}
    \mathbb{E}[\mathbf{1}_{\{h(X) = y_k\}} \mid A = a] = 
    \mathbb{E}[\mathbf{1}_{\{h(X) = y_k\}}] 
    \quad \forall a \in \mathcal{A},~\forall y_k \in \mathcal{Y},
\end{equation}  
while General Label Equalized Odds become
\begin{equation}
\begin{aligned}
    \mathbb{E}[\mathbf{1}_{\{h(X) = y_k\}} \mid A = a, Y = y] = 
    \mathbb{E}[\mathbf{1}_{\{h(X) = y_k\}} \mid Y = y], \\
    \quad \forall a \in \mathcal{A},~\forall y, y_k \in \mathcal{Y}.
    \end{aligned}
\end{equation}
These fairness notions can be reformulated in a 
structured manner that is compatible with linear 
optimisation.

\subsection{Fairness Constraints as Linear Moment Conditions}

Following the approach introduced in 
Section~\ref{sec:eg}, fairness constraints can be 
expressed as linear inequalities over classifier moments. 
This section details how we construct such linear 
constraints for the general label 
(Definitions~\ref{def:gldp} and~\ref{def:gleo}) 
fairness definitions in both binary and multi-class 
settings, extending the original EG formulation to 
support multiple simultaneous constraints.

The constraints take the general form:
\begin{equation}
\gamma_{i}(h)= \sum_{j=1}^{m} M_{ij}\, \mu_{j}(h) \leq \epsilon_{i}, \quad i=1,\dots,n.
\end{equation}
Here, $M_{ij}$ are the entries of a matrix 
$M \in \mathbb{R}^{n \times m}$ that defines how each 
moment contributes to each constraint, where $m$ is 
the number of moments and $n$ is the number of 
constraints. The term $\epsilon_i$ denotes the upper 
bound for the $i$-th constraint, and \( \mu_j(h) \) 
is the $j$-th moment of the classifier \( h \), 
given by:
\begin{equation}
    \mu_j(h) = \mathbb{E}[g_j(X, A, Y, h(X)) \mid E_j], \quad j=1,\dots,m.
\end{equation}
In this formula, the function \( g_j \colon \mathcal{X} 
\times \mathcal{A} \times \mathcal{Y} \times \mathcal{Y} 
\rightarrow [0,1] \) is a measurable function influenced 
by the predicted label \( h(X) \). The event \( E_j \) 
is a set condition on the variables \( (X, A, Y) \), 
like \( A = a \) or \( A = a \And Y = y \), and it 
does not depend on the model.

To illustrate the moment construction, consider the 
case of Demographic Parity. For each group 
$a \in \mathcal{A}$ and each class $y_k \in \mathcal{Y}$, 
we define the following moment
\begin{equation}
\mu_a^{(k)}(h) = \mathbb{E}[\mathbf{1}_{\{h(X) = y_k\}} \mid A = a],
\end{equation}
where the indicator function $\mathbf{1}_{\{h(X) = y_k\}}$ 
captures whether the classifier predicts class $y_k$ 
for individuals within group $a$. Moreover, we define 
the overall moment:
\begin{equation}
\mu_*^{(k)}(h) = \mathbb{E}[\mathbf{1}_{\{h(X) = y_k\}}],
\end{equation}
corresponding to the unconditional selection rate of 
class $y_k$ across the entire population.

A classifier satisfies Definition~\ref{def:gldp} if 
it provides the same expected value of 
$\mathbf{1}_{\{h(X) = y_k\}}$ regardless of the value 
of $A$, for all $y_k \in \mathcal{Y}$. Therefore, 
each equality constraint can be expressed as
\begin{equation}
\mu_a^{(k)}(h) = \mu_*^{(k)}(h), \quad \forall a \in \mathcal{A},\ \forall y_k \in \mathcal{Y},
\end{equation}
which can be equivalently rewritten as the pair of 
inequalities:
\[
\mu_a^{(k)}(h) - \mu_*^{(k)}(h) \leq 0,
\]
\[
\mu_*^{(k)}(h) - \mu_a^{(k)}(h) \leq 0.
\]

In the binary sensitive attribute case $A \in \{0,1\}$, 
the group $A = 0$ is our \textit{unprivileged} group 
and $A = 1$ our \textit{privileged} group. Hence, for 
each class $y_k \in \mathcal{Y}$, these inequalities 
become:
\[
\mu_0^{(k)}(h) - \mu_*^{(k)}(h) \leq 0,
\]
\[
\mu_*^{(k)}(h) - \mu_0^{(k)}(h) \leq 0,
\]
\[
\mu_1^{(k)}(h) - \mu_*^{(k)}(h) \leq 0,
\]
\[
\mu_*^{(k)}(h) - \mu_1^{(k)}(h) \leq 0.
\]

In the case of a biased classifier, we expect 
$\mu_1^{(k)}(h) > \mu_*^{(k)}(h) > \mu_0^{(k)}(h)$ 
for the privileged class $y_k$. To achieve General 
Label Demographic Parity, we build a block-diagonal 
constraint system $M\,\mu(h) \leq \epsilon$, with 
$\epsilon = \mathbf{0} \in \mathbb{R}^{4|\mathcal{Y}|}$,
\begin{equation}
\begin{aligned}
\mu(h) = 
\begin{bmatrix}
\mu_0^{(0)}(h) & \mu_1^{(0)}(h) & \mu_*^{(0)}(h) &
\\\cdots &
\mu_0^{(K)}(h) & \mu_1^{(K)}(h) & \mu_*^{(K)}(h)
\end{bmatrix}^\top
\in \mathbb{R}^{3|\mathcal{Y}|},
\end{aligned}
\end{equation}
\begin{equation}
\begin{aligned}
M &= \mathrm{diag}(\underbrace{M_0, \dots, M_0}_{|\mathcal{Y}|})
\in \mathbb{R}^{4|\mathcal{Y}| \times 3|\mathcal{Y}|},
\\
M_0 &= 
\begin{bmatrix}
\phantom{-}1 & \phantom{-}0 & -1 \\
-1 & \phantom{-}0 & \phantom{-}1 \\
\phantom{-}0 & \phantom{-}1 & -1 \\
\phantom{-}0 & -1 & \phantom{-}1
\end{bmatrix}.
\end{aligned}
\end{equation}
Here, the coefficients $M_{ij}$ are simply $+1$, $-1$, 
or $0$ depending on whether the moment $\mu_j^{(k)}(h)$ 
appears with a positive sign, a negative sign, or not 
at all. For instance, the inequality 
$\mu_0^{(k)}(h) - \mu_*^{(k)}(h) \leq 0$ corresponds 
to the row $[1, 0, -1]$ in $M_0$, while 
$\mu_*^{(k)}(h) - \mu_0^{(k)}(h) \leq 0$ corresponds 
to $[-1, 0, 1]$. Note that $M_0$ is identical for 
every class $y_k \in \mathcal{Y}$, since the constraint 
structure depends only on group membership, not on $y_k$.
We report the derivation of the linear moment constraint 
for the Equalized Odds constraint in \ref{app:eq_odds}.

It is possible to enforce multiple fairness definitions 
simultaneously by combining their respective constraint 
formulations. In particular, in this paper, we propose 
the Combined Parity (CP) constraint, which combines 
both Demographic Parity and Equalized Odds as joint 
conditions on the classifier. Details about the 
derivation of the CP linear constraint are reported 
in \ref{app:combined_parity}.

After defining the fairness constraints as linear 
inequalities over conditional moments, we can now move 
to the optimisation procedure. We introduce a general 
form of the Exponentiated Gradient algorithm of 
Agarwal et al.~\cite{agarwal_reductions_2018}, which 
was first made for binary classification tasks and for 
only one fairness constraint. Our updated version 
supports both binary and multi-class classification 
tasks in the presence of multiple fairness constraints 
under the general label formulation.\subsection{Generalized Exponentiated Gradient (\tool)}
In this section, we present our proposed Generalized 
Exponentiated Gradient (\tool) method, an in-processing 
bias mitigation algorithm that extends the original EG 
approach (Section~\ref{sec:eg}) to support both binary 
and multi-class classification tasks under multiple 
simultaneous fairness constraints.

Our contribution with respect to the original EG 
framework \cite{agarwal_reductions_2018} lies in three 
main extensions: (i) support for multi-class label 
spaces $\mathcal{Y} = \{0,1,\dots,K\}$ through the 
general label fairness constraints 
(Definitions~\ref{def:gldp} and~\ref{def:gleo}); 
(ii) simultaneous enforcement of multiple fairness 
definitions (i.e., Combined Parity); and 
(iii) a generalised cost-sensitive learning procedure 
that handles multi-class targets and multiple dual 
variables. These extensions preserve the reduction-based 
structure and convex optimisation properties of the 
original approach while significantly broadening its 
applicability.

Following the EG framework recalled in 
Section~\ref{sec:eg}, we adopt the reduction-based 
strategy and formulate the problem in terms of 
randomized classifiers, represented as distributions 
\( Q \in \Delta(\mathcal{H}) \) over the hypothesis 
class. This relaxation transforms the constrained 
optimisation problem~\eqref{eq1} into a convex 
saddle-point formulation. In a real-world scenario, 
we replace all expectations by empirical averages 
over the training set and allow for adaptive tolerance 
levels \(\widehat{\epsilon}_i\) on constraint 
violations.

\noindent The empirical optimisation problem becomes:
\begin{equation}\label{eq15}
  \min_{Q \in \Delta(\mathcal{H})} \widehat{\mathcal{R}}(Q) \quad \text{subject to} \quad \widehat{\gamma}_{i}(Q) \leq \widehat{\epsilon}_{i}, \quad i=1, \dots, n,  
\end{equation}
where \(\widehat{\mathcal{R}}(Q)\) and 
\(\widehat{\gamma}_i(Q)\) denote empirical estimates 
over the training samples.

\noindent Following the Lagrangian approach described 
in Section~\ref{sec:eg}, we reformulate this as a 
saddle-point problem with dual variables 
$\boldsymbol{\lambda}=(\lambda_{1},\lambda_{2}, 
\dots,\lambda_{n})$ associated with the $n$ fairness 
constraints. The Lagrangian function is defined as
\begin{equation}
\mathcal{L}(Q, \boldsymbol{\lambda}) = \widehat{\mathcal{R}}(Q) + \sum_{i=1}^{n} \lambda_i \left( \widehat{\gamma}_i(Q) - \widehat{\epsilon}_i \right),
\end{equation}
and the saddle-point formulation, with an 
\(\ell_1\)-norm constraint on \( \boldsymbol{\lambda} \) 
for numerical stability, is:
\begin{equation} \label{eq18}
\min_{Q \in \Delta(\mathcal{H})} \max_{\boldsymbol{\lambda} \geq 0,\ \|\boldsymbol{\lambda}\|_1 \leq B} \mathcal{L}(Q, \boldsymbol{\lambda}).
\end{equation}

The key difference from the original EG formulation 
is that we now handle $n$ simultaneous constraints 
rather than a single fairness definition, requiring 
a multi-dimensional dual update scheme that balances 
multiple constraint violations.

\noindent
As in the original EG approach, we interpret this as 
a zero-sum game between a \textit{learner} and an 
\textit{auditor}. The learner selects a randomized 
classifier \( Q \in \Delta(\mathcal{H}) \) to minimize 
the classification loss while satisfying fairness 
constraints, and the auditor updates the dual variables 
\( \boldsymbol{\lambda} \) to maximize the Lagrangian 
by penalizing constraint violations.

\noindent
At each iteration, the learner constructs a new 
classifier \( h_t \in \mathcal{H} \) by solving a 
\textit{cost-sensitive classification problem}. 
The key extension here compared to the binary EG is 
the handling of multi-class targets: rather than 
assigning a single signed weight with respect to a 
fixed positive label, each training sample receives 
a weight vector 
$\mathbf{w}_j^{(t)} \in \mathbb{R}^{|\mathcal{Y}|}$, 
with one entry per class $y_k \in \mathcal{Y}$, 
whose $k$-th entry is defined as:
\[
\left[\mathbf{w}_j^{(t)}\right]_k = 
\gamma_j^{\text{error},(k)} + 
\sum_{i=1}^{n} \lambda_i^{(t)} \cdot 
\gamma_{i,j}^{\text{fair},(k)},
\quad \forall y_k \in \mathcal{Y},
\]
where $\gamma_j^{\text{error},(k)} \in \{+1,-1\}$ 
encodes the misclassification cost with respect to 
class $y_k$:
\[
\gamma_j^{\text{error},(k)} =
\begin{cases}
-1 & \text{if } y_j \neq y_k, \\
+1 & \text{if } y_j = y_k,
\end{cases}
\]
and $\gamma_{i,j}^{\text{fair},(k)} \equiv 
\gamma_i^{(k)}(h(x_j))$ denotes the per-sample 
contribution to the $i$-th fairness constraint 
for class $y_k$. The adjusted label is selected 
as the class with the highest positive weight:
\[
\tilde{y}_j^{(t)} =
\begin{cases}
\displaystyle\arg\max_{y_k \in \mathcal{Y}} 
\left[\mathbf{w}_j^{(t)}\right]_k 
& \text{if } \max_k 
\left[\mathbf{w}_j^{(t)}\right]_k > 0, \\
y_j & \text{otherwise,}
\end{cases}
\]
and the learner solves:
\[
h_t = \arg\min_{h \in \mathcal{H}} \sum_j 
\max_k \left|\left[\mathbf{w}_j^{(t)}\right]_k\right| 
\cdot \mathbf{1}_{\{h(x_j) \neq \tilde{y}_j^{(t)}\}}.
\]
In the binary setting ($|\mathcal{Y}|=2$), 
$\mathbf{w}_j^{(t)}$ reduces to a scalar, recovering 
the standard binary EG formulation.

\noindent
The auditor updates the dual variables 
\( \boldsymbol{\lambda}^{(t)} \) by computing:
\[
\lambda_i^{(t)} = \frac{B \cdot \exp(\theta_i^{(t)})}{1 + \sum_{k=1}^{n} \exp(\theta_{k}^{(t)})},\quad \text{for all}\; i=1,\dots,n.
\]
In practice, the update of \( \boldsymbol{\theta}^{(t)} \) 
may also involve a learning rate or smoothing strategy 
to stabilise optimisation, as implemented in our 
algorithm. This formula ensures that 
\( \boldsymbol{\lambda}^{(t)} \) lies in the scaled 
probability simplex of radius \( B \), emphasising 
the most violated constraints.

\noindent
The process converges to an approximate saddle point 
\((Q^\star, \boldsymbol{\lambda}^\star)\), which 
represents a randomised classifier that achieves an 
optimal balance between predictive performance and 
fairness. The resulting distribution \( Q^\star \) is 
sparse, supported on a small number of base classifiers 
\( h_t \), and is normalised to form a valid probability 
distribution over the hypothesis class. The final 
weights \( \boldsymbol{\lambda}^\star \) provide 
insight into the most influential fairness constraints.

The optimisation process stops when the \textit{duality 
gap} falls below a small threshold \( \nu \), indicating 
that the current solution is close to a saddle point. 
The duality gap is computed as the difference between 
the Lagrangian value of the best single classifier and 
that of the current mixture \( Q \):
\[
\text{Gap}(Q, \boldsymbol{\lambda}) = \max_{h \in \mathcal{H}} \mathcal{L}(h, \boldsymbol{\lambda}) - \mathcal{L}(Q, \boldsymbol{\lambda}).
\]
When this gap becomes sufficiently small, no further 
improvement is expected, and the optimisation 
terminates.

\noindent
This entire procedure is formally presented in 
Algorithm~\ref{alg:geg_gen}.

\begin{algorithm}[H]
\caption{Generalized Exponentiated Gradient (GEG)}
\label{alg:geg_gen}
\begin{algorithmic}[1]
\Require 
Training data $(X, A, Y)$ with $Y \in \{0, 1, \dots, K\}$ \\
Hypothesis class $\mathcal{H}$, fairness constraints $\{\widehat{\gamma}_i\}_{i=1}^n$ with thresholds $\{\epsilon_i\}_{i=1}^n$ \\
Parameters: learning rate $\eta > 0$, tolerance $\delta > 0$, max iterations $T$, duality gap threshold $\nu > 0$, minimum iterations $t_{\min} \in \mathbb{N}$.
\Ensure 
Randomized classifier $Q \in \Delta(\mathcal{H})$
\vspace{0.5em}
\State Initialize dual variables: $\boldsymbol{\theta} \gets \mathbf{0}$, count vector $Q \gets \emptyset$, budget $B \gets 1/\delta$
\For{$t = 1$ to $T$}
    \State Compute dual weights: $\lambda_{i}^{(t)} \gets \dfrac{B \cdot \exp(\theta_{i}^{(t)})}{1 + \sum_{k=1}^{n} \exp(\theta_{k}^{(t)})}$
    \For{each training sample $j=1$ to $N$}
        \State Compute weight vector: $$\left[\mathbf{w}_{j}^{(t)}\right]_k \gets \gamma^{\text{error},(k)}_j + \sum_{i} \lambda_{i}^{(t)} \cdot \gamma^{\text{fair},(k)}_{i,j},$$ \quad where $\gamma^{\text{error},(k)}_j = \begin{cases} +1 & \text{if } y_j = y_k \\ -1 & \text{otherwise} \end{cases}\ \forall\, y_k \in \mathcal{Y}$
        \State Adjust label: $$\tilde{y}_j \gets
            \begin{cases}
            \arg\max_{y_k \in \mathcal{Y}} \left[\mathbf{w}_{j}^{(t)}\right]_k & \text{if } \max_k \left[\mathbf{w}_{j}^{(t)}\right]_k > 0 \\
            y_j & \text{otherwise}
            \end{cases}$$
    \EndFor
    \State Normalize weights: $$w_{j}^{(t)} \gets \dfrac{N \cdot \max_k \left|\left[\mathbf{w}_{j}^{(t)}\right]_k\right|}{\sum_{l=1}^{N} \max_k \left|\left[\mathbf{w}_{l}^{(t)}\right]_k\right|},\quad \text{for all}\; j$$
    \State Train classifier $h_t$ on $\{(x_j, \tilde{y}_j, w_{j}^{(t)})\}_{j=1}^N$
    \State Update count: $Q[h_t] \gets Q[h_t] + 1$
    \State Compute constraint violations: $\widehat{\gamma}_i(h_t)$
    \State Update dual: $\theta_{i}^{(t+1)} \gets \theta_{i}^{(t)} + \eta \cdot \left(\widehat{\gamma}_i(h_t) - \epsilon_i\right)$
    \State Compute current mixture: $$Q_t(h) \gets \dfrac{Q(h)}{\sum_{h'} Q(h')}\; \text{for all}\; h$$
    \State Compute duality gap: $$\text{Gap}_t \gets \max_{h \in \mathcal{H}} \mathcal{L}(h, \boldsymbol{\lambda}) - \mathcal{L}(Q_t, \boldsymbol{\lambda})$$
    \If{$\text{Gap}_t < \nu$ and $t \geq t_{\min}$}
        \State \textbf{break}
    \EndIf
\EndFor
\State Normalize final distribution: $$Q(h) \gets \dfrac{Q(h)}{\sum_{h'} Q(h')},\; \text{for all} \; h \in \mathcal{H}.$$
\State \Return $Q$
\end{algorithmic}
\end{algorithm}
\subsection{\revised{Convergence Analysis}
\label{sec:convergence}}

\revised{In the following, we analyse the convergence properties of the proposed Generalized Exponentiated Gradient (\tool) algorithm.
Our analysis follows the reduction-based framework of Agarwal et al.~\cite{agarwal_reductions_2018} and shows that the proposed extensions to multi-class classification and to multiple simultaneous fairness constraints preserve the theoretical guarantees of the original method.
}
\revised{
We recall that the constrained empirical optimisation problem~\eqref{eq15} admits the following convex--concave saddle-point formulation:
\begin{equation}
\min_{Q \in \Delta(\mathcal{H})} 
\max_{\boldsymbol{\lambda} \ge 0,\ \|\boldsymbol{\lambda}\|_1 \le B}
\mathcal{L}(Q, \boldsymbol{\lambda}),
\end{equation}
with Lagrangian
\begin{equation}
\mathcal{L}(Q, \boldsymbol{\lambda}) 
= 
\widehat{\mathcal{R}}(Q) 
+ 
\sum_{i=1}^{n} \lambda_i \left( \widehat{\gamma}_i(Q) - \widehat{\epsilon}_i \right).
\end{equation}}

\revised{Both the empirical classification risk and the empirical fairness constraint functions are affine with respect to the randomized classifier $Q \in \Delta(\mathcal{H})$.
Since $Q$ represents a probability distribution over base classifiers, these quantities can be written as expectations over $Q$, and hence as weighted averages of their values on individual classifiers.
Consequently, the Lagrangian $\mathcal{L}(Q,\boldsymbol{\lambda})$ is affine in $Q$ and in the dual variables $\boldsymbol{\lambda}$, which implies convexity with respect to $Q$ and concavity with respect to $\boldsymbol{\lambda}$.}

\revised{Throughout the analysis, we assume that the empirical loss and the fairness moments are bounded, namely $\widehat{\mathcal{R}}(h) \in [0,1]$ and $\widehat{\gamma}_i(h) \in [-1,1]$ for all $h \in \mathcal{H}$ and all constraints $i$.
Moreover, the dual variables are restricted to the compact convex set
\[
\Lambda = \{ \boldsymbol{\lambda} \ge 0 : \|\boldsymbol{\lambda}\|_1 \le B \}.
\]
Under these conditions, the Lagrangian defines a convex--concave function over compact convex sets.
By Sion’s minimax theorem~\cite{sion1958general}, the above saddle-point problem admits at least one solution.}

\revised{We now turn to the algorithmic aspect.
The \tool algorithm can be interpreted as a repeated zero-sum game between a learner and an auditor.
At each iteration $t$, the learner selects a classifier by solving a cost-sensitive classification problem induced by the current dual variables, while the auditor updates the dual vector using an Exponentiated Gradient step over $\Lambda$.
Since Exponentiated Gradient is a no-regret online optimisation method on the simplex, it satisfies standard regret guarantees~\cite{kivinen1997exponentiated}.
}
\revised{
Let $Q_T$ denote the empirical mixture of classifiers obtained after $T$ iterations.
We assume that the cost-sensitive learner returns at each iteration a $\tau$-approximate best response, that is,
\begin{equation}
\mathcal{L}(h_t,\boldsymbol{\lambda}^{(t)})
\le
\min_{h\in\mathcal{H}} \mathcal{L}(h,\boldsymbol{\lambda}^{(t)}) + \tau .
\end{equation}
}
\begin{theorem}[Convergence of \tool]
\label{thm:convergence}
Let $\eta = \Theta(1/\sqrt{T})$ be the learning rate used in the Exponentiated Gradient updates.
Under the boundedness assumptions and a $\tau$-approximate cost-sensitive oracle, the iterate
$(Q_T,\boldsymbol{\lambda}^{(T)})$ produced by \tool satisfies 
\begin{equation}
\max_{\boldsymbol{\lambda} \in \Lambda} 
\mathcal{L}(Q_T, \boldsymbol{\lambda})
-
\min_{Q \in \Delta(\mathcal{H})}
\mathcal{L}(Q, \boldsymbol{\lambda}^{(T)})
\;\le\;
O\!\left( \frac{B \sqrt{\log(n+1)}}{\sqrt{T}} \right) + \tau .
\end{equation}
In particular, $Q_T$ is an $O(1/\sqrt{T})$-approximate solution to~\eqref{eq15}, up to $\tau$.
\end{theorem}

\revised{
\begin{proof}
The claim follows from the no-regret guarantee of Exponentiated Gradient, i.e., multiplicative-weights
updates over the simplex $\Lambda$ for bounded linear losses, combined with the $\tau$-approximate
best-response property of the cost-sensitive oracle. The standard regret-to-duality-gap conversion
for repeated zero-sum games then yields an $O(1/\sqrt{T})$ rate for the averaged primal iterate.
See \ref{app:proof_convergence} for details.
\end{proof}
}

\subsection{Implementation Details}

   We implemented \tool in Python 3.9 by extending the EG implementation provided by the \texttt{Fairlearn} Python library \cite{fairlearn_doc}. 
We provide the implementation of \tool and the evaluation scripts online for public use and research.\replpackage

%% file: sections/evaluation.tex
In this section, we describe the empirical evaluation conducted to assess the effectiveness of \tool. Specifically, our evaluation is driven by the following research questions (RQ):

\begin{enumerate}
    \item[\textbf{RQ$_1$}] \textbf{Multi-class classification:} \textit{To what extent is \tool able to mitigate bias while keeping a high prediction effectiveness in a multi-class classification context?} 
    

    \item[\textbf{RQ$_2$}] \textbf{Binary classification:} \textit{To what extent is \tool able to mitigate bias while keeping a high prediction effectiveness in a binary classification context?}


    \item[\textbf{RQ$_3$}] \textbf{Baseline comparison:} \textit{How does \tool compare against existing bias mitigation methods in the multi-class classification tasks?}


    \item[\textbf{RQ$_4$}] \textbf{Sensitivity analysis:} \textit{How stable is \tool under different $\delta$ and $\eta$ values?}

    \item[\textbf{RQ$_5$}] \textbf{Different base classifiers:} \textit{How does \tool perform under different base-classifiers?}

    
\end{enumerate}

\begin{figure}
    \centering
    \includegraphics[width=\linewidth]{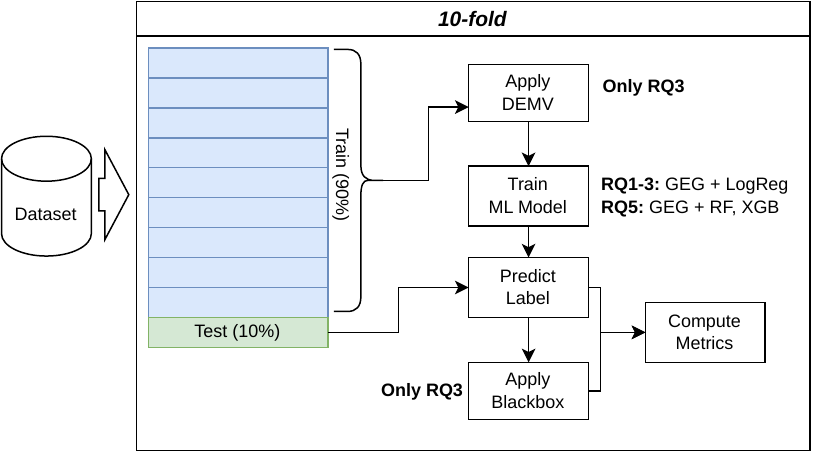}
    \caption{\revised{Experimental process}}
    \label{fig:exp}
\end{figure}

Figure \ref{fig:exp} details the experimental process. To mitigate the risk of data selection bias, for each \textbf{RQ}, we perform a 10-fold cross-validation with shuffling. For each fold, we train the models on the training set and compute the fairness and effectiveness metrics on the testing set. To ensure a fair evaluation, we use the same splits for all approaches in each RQ by fixing the random seed.
Additionally, when we evaluate DEMV to address \textbf{RQ3}, following the original paper \cite{daloisio_debiaser_2023}, we apply the pre-processing approach only on the training set. \revised{Similarly, the post-processing Blackbox approach is applied on the prediction from the ML model \cite{putzel_blackbox_2022}.}
Finally, to address \textbf{RQ5}, we perform the same pipeline, but with different $\delta$ and $\eta$ values, respectively.

In the following, we describe the datasets employed (Section \ref{sec:data}), the benchmarks used in the evaluation (Section \ref{sec:benchmarks}), and the metrics and statistical methods adopted (Section \ref{sec:metrics}).

\subsection{Datasets}\label{sec:data}

\begin{table}[tb!]
    \centering
    \caption{Employed Datasets. The seven top-most rows report multi-class datasets, while the lower-most three rows describe binary datasets.}
    \label{tab:dataset}
    \resizebox{\linewidth}{!}{\begin{tabular}{l|l|r|r|r|r}
    \toprule
    \textbf{Name} & \textbf{Sens. Attribute} & \textbf{Instances} & \textbf{Features} & \textbf{Classes} & \textbf{Class Distr.} \\
       \midrule
       \midrule
       \textbf{CMC} \cite{lim2000comparison} & \textit{religion} & 1473 & 10 & 3 & \makecell[r]{\textit{1}: 42.7\%\\\underline{\textit{2}: 22.6\%}\\\textit{3}: 34.7\%}\\
       \midrule
       \textbf{Crime} \cite{redmond2002data} & \textit{race} & 1,994 & 100 & 6 & \makecell[r]{\underline{\textit{100}: 23.8\%}\\\textit{200}: 15.8\%\\\textit{300}: 21.2\%\\\textit{400}: 19.9\%\\\textit{500}: 19.3\%}\\
       \midrule
       \textbf{Drug} \cite{fehrman_five_2017} & \textit{race} & 1,885 & 15 & 3 & \makecell[r]{\underline{\textit{0}: 21.9\%}\\\textit{1}: 25.1\%\\\textit{2}: 52.9\%}\\
       \midrule
       \textbf{Law} \cite{austin2016will} & \textit{gender} & 20,427 & 14 & 3 & \makecell[r]{\textit{0}: 41.6\%\\\textit{1}: 27.9\%\\\underline{\textit{2}: 31.1\%}}\\\midrule
       \textbf{Obesity} \cite{palechor_dataset_2019} & \textit{age} & 1,490 & 17 & 5 & \makecell[r]{\underline{\textit{0}: 19.3\%}\\\textit{1}: 19.5\%\\\textit{2}: 19.4\%\\\textit{3}: 23.5\%\\\textit{4}: 18.2\%}\\
       \midrule
       \textbf{Park} \cite{tsanas_accurate_2009} & \textit{age} & 5,875 & 19 & 3 & \makecell[r]{\underline{\textit{0}: 30.0\%}\\\textit{1}: 44.6\%\\\textit{2}: 24.9\%}\\
        \midrule
        \textbf{Wine} \cite{cortez2009modeling} & \textit{type} & 6,438 & 13 & 4 & \makecell[r]{\textit{4}: 3.4\%\\\textit{5}: 34.1\%\\\underline{\textit{6}: 45.3\%}\\\textit{7}: 17.2\%}\\
        \midrule
        \midrule
       \textbf{Adult} \cite{kohavi_scaling_1996} & \textit{sex} & 30,940 & 102 & 2 & \makecell[r]{\textit{0}: 75.7\%\\\underline{\textit{1}: 24.2\%}}\\
       \midrule
       \textbf{COMPAS} \cite{angwin2016machine} & \textit{race} & 6,167 & 399 & 2 & \makecell[r]{\underline{\textit{0}: 54.4\%}\\\textit{1}: 45.5\%}\\
       \midrule
       \textbf{German} \cite{ratanamahatana2002scaling} & \textit{sex} & 1,000 & 59 & 2 & \makecell[r]{\textit{0}: 30\%\\\underline{\textit{1}: 70\%}}\\
       \bottomrule
    \end{tabular}}
\end{table}

Table \ref{tab:dataset} reports the list of datasets employed in our study. For each dataset, we report its name, the sensitive attribute as reported in the corresponding source paper, the number of instances and features, the number of possible classes to be predicted, and their distribution. The seven top-most rows report multi-class datasets, while the lower-most three rows describe binary datasets. The datasets have been selected based on their relevance, diversity, and adoption in previous fairness studies \cite{daloisio_debiaser_2023,daloisio_towards_2025,fabris_algorithmic_2022}. \revised{Additionally, we refer to the same positive labels specified in previous works \cite{daloisio_debiaser_2023,daloisio_towards_2025}.} Specifically, to answer \textbf{RQ$_1$}, \textbf{RQ$_3$}, \textbf{RQ$_4$}, and \textbf{RQ$_5$} we employ the following seven multi-class datasets:

\begin{enumerate}
    \item \textbf{Contraceptive Method Choice (CMC)} \cite{lim2000comparison}. This dataset contains 1,473 instances and 10 features about the adoption of contraceptive methods by women in Indonesia. The sensitive feature is \textit{religion} and the positive outcome is \textit{2 (long-term use)}.

    \item \textbf{Communities and Crime (Crime)} \cite{redmond2002data}. This dataset includes 1,994 instances and 100 features about the per-capita violent crimes in U.S. communities. The sensitive feature is \textit{race} and the positive outcome is \textit{100 (low rate of crimes)}.

    \item \textbf{Drug Usage (Drug)} \cite{fehrman_five_2017}. This dataset includes 1,885 instances and 15 features about the frequency of drug consumption. The sensitive attribute is \textit{race} and the positive class is \textit{0 (never use)}.

    \item \textbf{Law School Admission (Law)} \cite{austin2016will}. This dataset contains 20,427 samples and 14 features about admissions scores to a law school. The sensitive attribute is \textit{gender} and the positive outcome is \textit{2 (high admission score)}.

    \item \textbf{Obesity Estimation (Obesity)} \cite{palechor_dataset_2019}. This dataset contains 1,490 instances and 17 features about patients' obesity estimation. The sensitive feature is \textit{age} and the positive class is \textit{0 (no obesity)}.

    \item \textbf{Parkinson's Telemonitoring (Park)} \cite{tsanas_accurate_2009}. This dataset includes 5,875 instances and 19 features about patients affected by Parkinson's disease, measured with the Unified Parkinson’s Disease Rating Scale (UPDRS) classification. The sensitive attribute is \textit{age} and the positive class is \textit{0 (mild class)}.

    \item \textbf{Wine Quality (Wine)} \cite{cortez2009modeling}. This dataset includes 6,438 instances and 13 features about wine quality classification. The sensitive feature is wine \textit{type} and the positive outcome is \textit{6 (high quality class)}.
\end{enumerate}

We employ instead the following binary datasets to answer the \textbf{RQ$_2$} of our study:

\begin{enumerate}
    \item \textbf{Adult Income (Adult)} \cite{kohavi_scaling_1996}. This dataset comprises 30,940 instances and 102 features related to the income of people in the U.S. The sensitive attribute is \textit{sex} and the positive outcome is \textit{1 (high income class)}.

    \item \textbf{ProPublic Recidivism (COMPAS)} \cite{angwin2016machine}. This dataset contains 6,167 samples and 399 features (one-hot encoded) about the recidivism prediction of condemned people. The sensitive feature is \textit{race} and the positive class is \textit{0 (no recidivism)}.

    \item \textbf{German Credit (German)} \cite{ratanamahatana2002scaling}. This dataset includes 1,000 instances and 59 features about the classification of people as \textit{good} or \textit{bad} credit risk. The sensitive attribute is \textit{sex} and the positive outcome is \textit{1 (good credit risk)}.
\end{enumerate}

\subsection{Benchmarks}\label{sec:benchmarks}


To answer the \textbf{RQ$_1$}, \textbf{RQ$_2$}, and \textbf{RQ$_4$} of our study, we compare the fairness and effectiveness of \tool with those of a Logistic Regression (LR) classifier. We have chosen this model because it has been successfully applied in previous fairness studies and in multi-class classification tasks \cite{hort2024bias,daloisio_debiaser_2023}. To ensure a fair comparison, the same LR model is used as a base-classifier for \tool. Specifically, concerning \textbf{RQ$_1$}, we employ three versions of \tool, each one adopting a different fairness constraint during the optimisation process: one version uses a Statistical Parity constraint (GEG-SP), another version employs the Equalised Odds constraint (GEG-EO), and the last version uses a Combined Parity constraint, optimising for both SP and EO at the same time (GEG-CP). Instead, concerning \textbf{RQ$_2$}, since the implementation of GEG-SP and GEG-EO for binary classification is equal to the already existing EG approach from Agarwal et al. \cite{agarwal_reductions_2018}, we consider these methods as additional baselines. Therefore, we compare these results with those of GEG-CP, which is our novel contribution for binary classification.

Concerning \textbf{RQ$_3$}, we compare the three versions of \tool (i.e., GEG-SP, GEG-EO, and GEG-CP) with the Debiaser for Multiple Variables (DEMV) \revised{pre-processing} approach \cite{daloisio_debiaser_2023} \revised{and the Blackbox post-processing approach \cite{putzel_blackbox_2022}. The rationale for choosing these benchmarks is twofold: first, they are among the main approaches proposed for multi-class bias mitigation; secondly, by benchmarking different categories of multi-class bias mitigation methods (i.e., DEMV is a pre-processing, \tool is an in-processing, and Blackbox is a post-processing bias mitigation approach), we have a comprehensive view of the strengths and weaknesses of methods applied at different bias mitigation stages.} 
As for the first two \textbf{RQs}, we employ an LR model as a base classifier.

Finally, for \textbf{RQ$_5$}, we benchmark the three versions of \tool against a Random Forest (RF) and an Extreme Gradient Boosting (XGB) classifier. The choice for these models is still driven by their adoption in previous fairness studies and multi-class classification tasks \cite{hort2024bias,daloisio_debiaser_2023}. As with the other \textbf{RQs}, to ensure a fair comparison, we use the same base classifiers for \tool. 

In the answer to \textbf{RQ$_1$}, \textbf{RQ$_2$}, \textbf{RQ$_3$}, and \textbf{RQ$_5$}, in accordance with the original EG approach \cite{agarwal_reductions_2018}, we set $\eta$ to 2.0 and $\delta$ to $0.05$. Instead, to perform the sensitivity analysis in the answer to \textbf{RQ$_4$}, we explore a space of $[0.0 : 0.1]_{0.005}$ for $\delta$ and $[1.0 : 10.0]_{0.5}$ for $\eta$.

For all ML models, we use their implementation available in the \textit{scikit-learn} Python library, with their default hyperparameters \cite{scikit-learn}. For DEMV, we employ the implementation available in the original paper with its default hyperparameters \cite{daloisio_debiaser_2023}.
Finally, we would like to emphasise that, while Large Language Models (LLMs) are currently considered state-of-the-art for many tasks, we did not include them in our evaluation. This decision was made because LLMs tend to be less effective in tabular classification tasks compared to ensemble methods such as RF or XGB \cite{grinsztajn_why_2022,fang_large_2024}.

\subsection{Evaluation Metrics and Methods}\label{sec:metrics}

In the following sections, we describe the fairness and effectiveness metrics, as well as the statistical methods \revised{and Pareto optimal strategy} employed in our evaluation.

\subsubsection{Metrics}

We employ a heterogeneous set of metrics to evaluate the fairness and effectiveness of the approaches analysed in each RQ. 

Concerning effectiveness metrics, following previous studies \cite{daloisio_towards_2025,chen_fairness_2024}, we employ the following metrics:

\begin{itemize}
    \item \textbf{Accuracy}. This metric is defined as the percentage of correct predictions over the total predictions of a model:
    \begin{equation}
    \text{Accuracy =} \frac{1}{N}\sum_{i=1}^N \mathbf{1}(\hat{y}_i = y_i) 
    \end{equation}
    Where $N$ is the total number of samples, $\hat{y}_i$ and $y_i$ are the $i$-th true and predicted samples, and $\mathbf{1}(\hat{y}_i = y_i)$ is a function which is equal to 1 if the prediction is equal to the true label and 0 otherwise. It ranges from 0 to 1, where 1 is the highest score \cite{rosenfield_coefficient_1986}.

    \item \textbf{Macro Precision}. This metric is an adaptation of the \textit{Precision} score for the multi-class classification context \cite{buckland_relationship_1994}. It is defined as the unweighted average of the class-wise \textit{precision} score:
    \begin{equation}
    \text{Macro Precision =} \frac{1}{K} \sum_{k=1}^K {Precision}_k
    \end{equation}
    Where $K$ is the number of classes and ${Precision}_k$ is the ratio of correctly predicted $k$ class over all $k$ predictions \cite{buckland_relationship_1994}. It ranges from 0 to 1, where 1 is the highest score.

    \item \textbf{Macro Recall}. Like Macro Precision, this metric is an adaptation of the \textit{Recall} score for multi-class classification \cite{buckland_relationship_1994}. It is defined as the unweighted average of the class-wise \textit{recall} score:
    \begin{equation}
    \text{Macro Recall =} \frac{1}{K} \sum_{k=1}^K {Recall}_k
    \end{equation}
    Where $K$ is the number of classes and ${Recall}_k$ is the ratio of instances of the $k$ class identified by the model \cite{buckland_relationship_1994}. It ranges from 0 to 1, where 1 is the highest score.

    \item \textbf{Macro F1 Score}. This metric is defined as the harmonic mean between \textit{Macro Precision} and \textit{Macro Recall} \cite{buckland_relationship_1994}:
    \begin{equation}
    \text{Macro F1 Score =} 2 \times \frac{\text{Macro Precision} \times \text{Macro Recall}}{\text{Macro Precision} + \text{Macro Recall}} 
    \end{equation}
    It ranges from 0 to 1, where 1 is the best score.
\end{itemize}

Concerning fairness, we consider \revised{six} widely adopted fairness definitions \cite{daloisio_debiaser_2023,daloisio_towards_2025,hort2024bias,putzel_blackbox_2022}:

\begin{itemize}

    \item \revised{\textbf{Multi-class Statistical Parity Difference (SPD)}. This metric implements the \textit{Multi-class Demographic Parity} fairness definition defined in Definition \ref{def:gldp}. It measures fairness as the maximum difference in the probability of being in the privileged group, given the predicted outcome $y$ \cite{dwork_fairness_2012}. It is defined as:
    \begin{equation}
        \text{SPD} = \max_{y \in \mathcal{Y}}\Big(
            \mathbb{P}(h(X)=y \mid A=1)
            -
            \mathbb{P}(h(X)=y \mid A=0)
            \Big)
    \end{equation}
     where $A=0$ and $A=1$ are the unprivileged and privileged groups, respectively. This metric ranges from -1 to +1, and the closer to 0, the fairer the model.}

    \item \revised{\textbf{Multi-class Equal Opportunity Difference (EOD)}. This metric assesses fairness as the maximum difference in the probability of having any possible outcome predicted conditioned on the value of the true label, being in the privileged group or not \cite{hardt2016equality}. It is defined as:
    \begin{equation}
    \begin{aligned}
    \text{EOD} =  \max_{y \in \mathcal{Y}}\Big(
         \mathbb{P}(h(X)=y | Y=y, A=1) \\ -  \mathbb{P}(h(x)=y | Y=y, A=0) \Big)
    \end{aligned}
    \end{equation}
    This metric ranges from -1 to +1, and the closer to 0, the fairer the model.}

    \item \revised{\textbf{Multi-class Average Odds Difference (AOD)}. This metric implements the \textit{Multi-class Equalized Odds} fairness definition shown in Definition \ref{def:gleo}. It measures fairness as the maximum difference between true positive (TPR) and false positive (FPR) rates, concerning any possible outcome, for items being in the privileged and unprivileged groups \cite{chen_fairness_2024}. Formally, it is defined as:
    \begin{equation}
    \begin{aligned}
    \text{AOD =}  \max_{y \in \mathcal{Y}}\Big(
        \frac{1}{2}((FPR_{y,A=0} - FPR_{y,A=1}) \\+ (TPR_{y,A=0} - TPR_{y,A=1}))
        \Big)
    \end{aligned}
    \end{equation}
    It ranges from -1 to +1, where the closer to 0, the fairer the model.
    }

\end{itemize}

Following previous work \cite{chen_fairness_2024,daloisio_towards_2025,daloisio_debiaser_2023}, we consider the absolute values of SPD, EOD, and AOD to gain a clearer understanding of a model's fairness improvement.

\subsubsection{\revised{Evaluation and} Statistical Methods}

When comparing \tool against baseline approaches in Section \ref{sec:results}, we report the trade-off between fairness and effectiveness using Pareto-optimality \cite{Harman2011} and trade-off scores. Pareto-optimality is computed by identifying the Pareto Front from all the solutions returned by all algorithms and counting the number of times a solution from a given algorithm appears in the front \cite{Harman2011}. We recall how the Pareto Front is defined as the set of solutions $ x \in X$ such that $x$ is better than all other solutions in $X$ in at least one objective, and no other solution dominates $x$ in all objectives \cite{Harman2011}. In this context, we use the F1-score as the effectiveness metric because it provides a comprehensive picture of the model's predictive accuracy.

The trade-off score is instead defined as:
\begin{equation}
    T(\alpha) = \alpha \cdot \text{F1-score} + (1-\alpha) \cdot (1- \text{Fair})
\end{equation}
where Fair is the average among the fairness scores $(\tfrac{SPD+EOD+AOD}{3})$. $\alpha \in [0,1]$ represents the weight that a user wants to give to effectiveness over fairness. For each approach analysed, we plot the variation of $T$ as $\alpha$ increases, and mark the point at which the trade-off achieved by the unbiased baseline model becomes better than that of the approach analysed.


\revised{Additionally, we report the detailed metrics achieved by each approach in \ref{app:detailed}. In describing the results, we employ the non-parametric one-sided Wilcoxon signed-rank test to assess the statistical significance of the difference between the metrics obtained by baselines and \tool.} The Wilcoxon test is a non-parametric test that verifies the null hypothesis that the median between two dependent samples is different \cite{woolson_wilcoxon_2005}. Being non-parametric, it raises the bar for significance by making no assumptions about the underlying samples. Specifically, the null hypothesis we check is \textit{"$H_0:$ The objective $O$ obtained by \tool is not improved with respect to the baseline approach $x$"}. The alternative hypothesis is: \textit{"$H_1:$ The objective $O$ obtained by \tool is improved with respect to the baseline approach $x$"}. For effectiveness metrics ``\textit{improved}" means that the score obtained by \tool is higher than the baseline. On the contrary, for fairness metrics ``\textit{improved}" means that the score obtained by \tool is lower than the baseline. Following standards \cite{sarro_multi-objective_2016,arcuri_practical_2011}, we set the confidence value to 0.05. \revised{In addition, we apply the Holm-Bonferroni correction to mitigate the risk of Type I statistical error (i.e., wrongly rejecting the null hypothesis) \cite{abdi2010holm}.} Therefore, we reject the null hypothesis if the \revised{adjusted} test's $p\text{-value is} < 0.05$.

%% file: sections/results.tex

In the following, we report the results of our empirical evaluation. 
For each RQ, we first discuss fairness and effectiveness scores in isolation. Then, we discuss the fairness/effectiveness trade-off achieved in terms of $T(\alpha)$ variation, as well as Pareto optimality. Detailed metrics are available in \ref{app:detailed}. 


\subsection{RQ$_1$: Multi-class classification.}

\subsubsection{Single Scores}

\begin{figure*}[ht!]
    \centering
    \includegraphics[width=.9\linewidth]{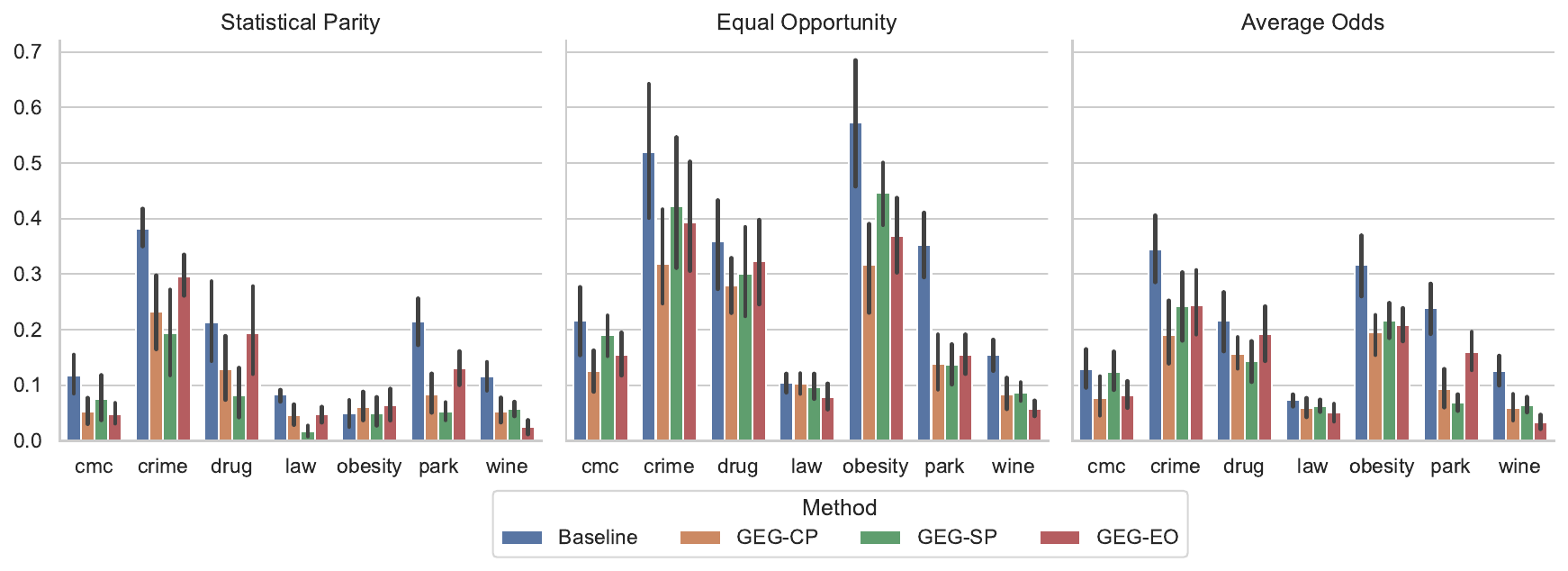}
    \caption{RQ1: Fairness scores comparison}
    \label{fig:rq1_fairness}
\end{figure*}

\begin{figure*}[ht!]
    \centering
    \includegraphics[width=.9\linewidth]{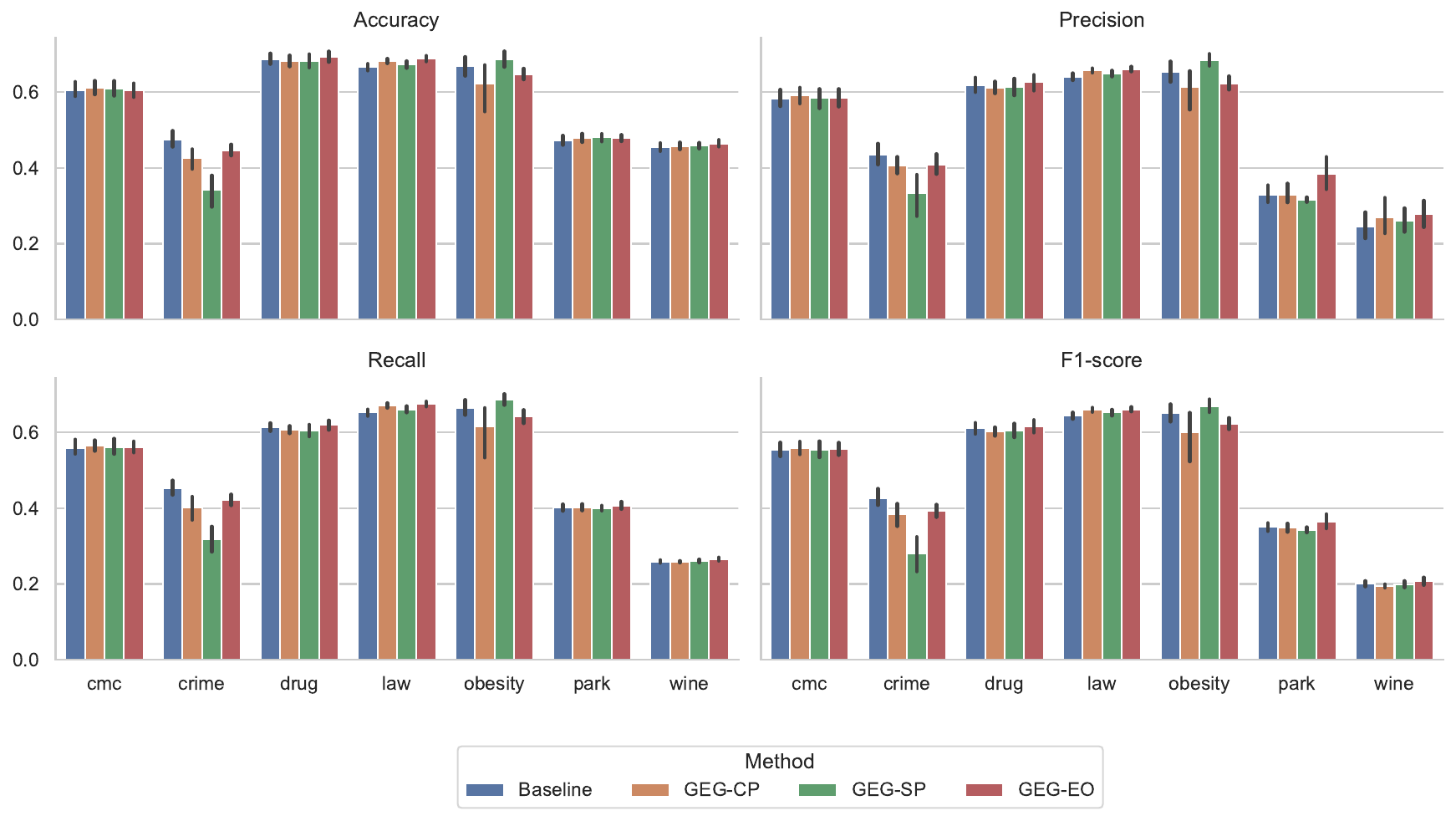}
    \caption{RQ1: Effectiveness scores comparison}
    \label{fig:rq1_effectiveness}
\end{figure*}

Figure \ref{fig:rq1_fairness} reports the fairness scores obtained for each dataset by the unbiased LR baseline and the different versions of \tool. From the figure, we observe that all versions of \tool are effective at reducing the underlying model's bias across all datasets and metrics considered, with a \textit{large} effect size in 66.7\% of cases analysed (as shown in Table \ref{tab:rq1_result}. Moreover, we observe a similar trend across the three versions of \tool, although GEG-SP exhibits slightly worse behaviour than the other two methods.


Figure \ref{fig:rq1_effectiveness} highlights the impact of \tool on effectiveness. We observe how \tool provides a prediction effectiveness consistent with the base LR model. Only two exceptions are observed for the \textit{Crime} and, partially, \textit{Obesity} dataset, where the effectiveness of \tool is slightly worse. This behaviour could be explained by the fact that \textit{Crime} and \textit{Obesity} are the datasets with the highest number of labels (see Table \ref{tab:dataset}); therefore, classification effectiveness may be more affected by bias mitigation. Finally, from Table \ref{tab:rq1_result}, we observe how GEG-CP and GEG-EO significantly improve the effectiveness of the base classifier with \textit{large} effect size for the \textit{Law} and, in part, \textit{Wine} datasets. Notably, this result could be explained by an already low bias of the base classifier.

\subsubsection{Trade-offs}

\begin{figure}[ht!]
    \centering
    \includegraphics[width=.8\linewidth]{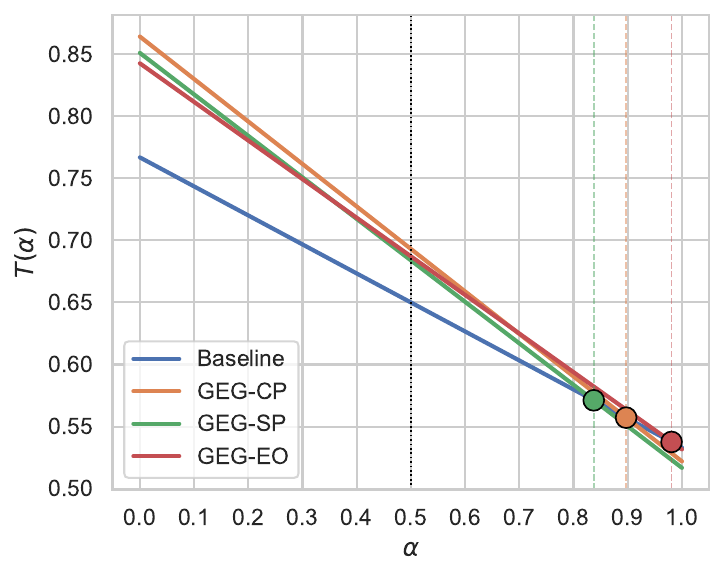}
    \caption{RQ1: Fairness-effectiveness trade-off score at different effectiveness weights $\alpha$}
    \label{fig:rq1_tradeoff}
\end{figure}

Figure \ref{fig:rq1_tradeoff} reports the variation of $T(\alpha)$ at increasing $\alpha$ values. 
We observe that when $\alpha$ is 0 (i.e., no relevance to effectiveness), all approaches achieve a higher score than the baseline, meaning that all versions of \tool effectively reduce the bias of the base LR model. Notably, when $\alpha$ equals 0.5 (i.e., same relevance to fairness and effectiveness), all versions of \tool still achieve a higher fairness/effectiveness trade-off than the baseline. Indeed, the first intersection point with the LR baseline is observed in GEG-SP, when $\alpha$ is around 0.83, while GEG-CP intersects the baseline when $\alpha$ is 0.9 and GEG-EO when $\alpha$ is close to 1.0. These results highlight that the fairness improvements achieved by all versions of \tool do not affect the model's effectiveness, and that the fairness/effectiveness trade-off is overcome by the baseline only when the relevance of effectiveness over fairness is very high.  

\begin{figure*}[ht!]
    \centering
    \begin{subfigure}{.9\textwidth}
        \includegraphics[width=\textwidth]{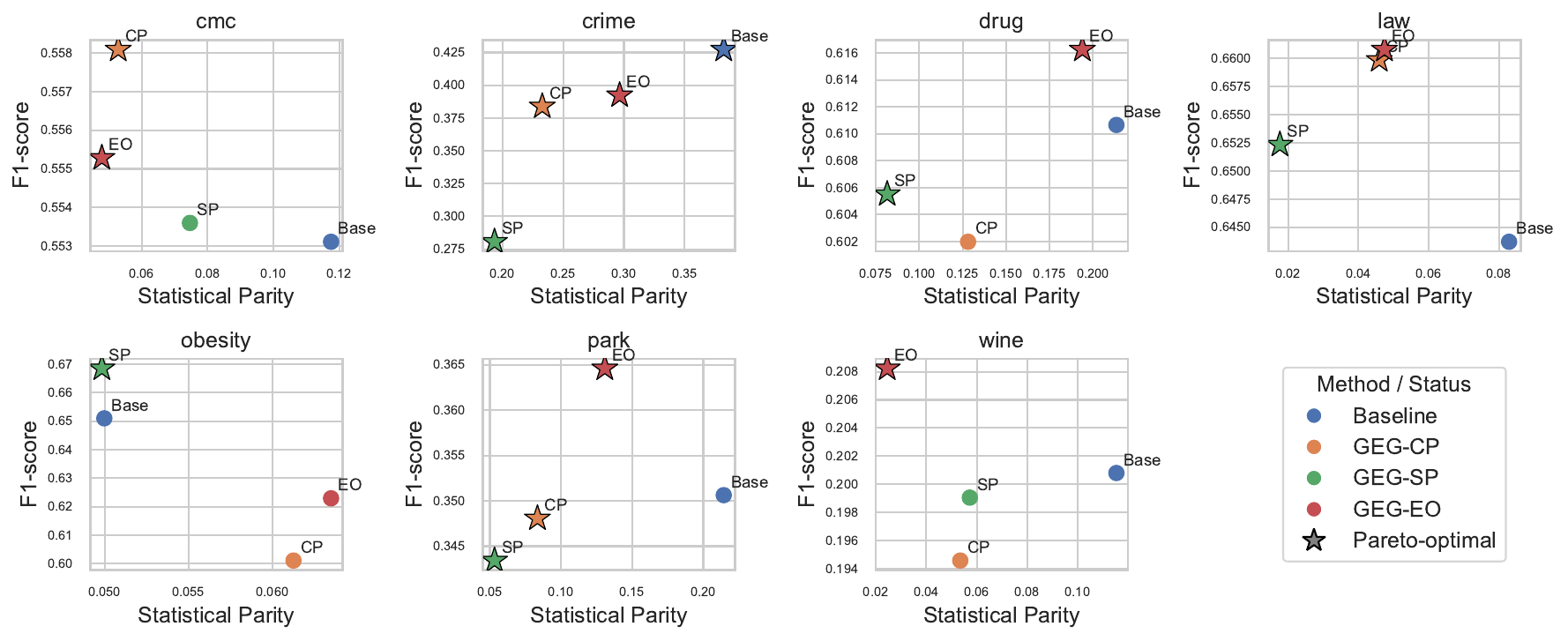}
        \caption{Statistical Parity}
        \label{fig:pareto_sp}
    \end{subfigure}
    \begin{subfigure}{.9\textwidth}
        \includegraphics[width=\textwidth]{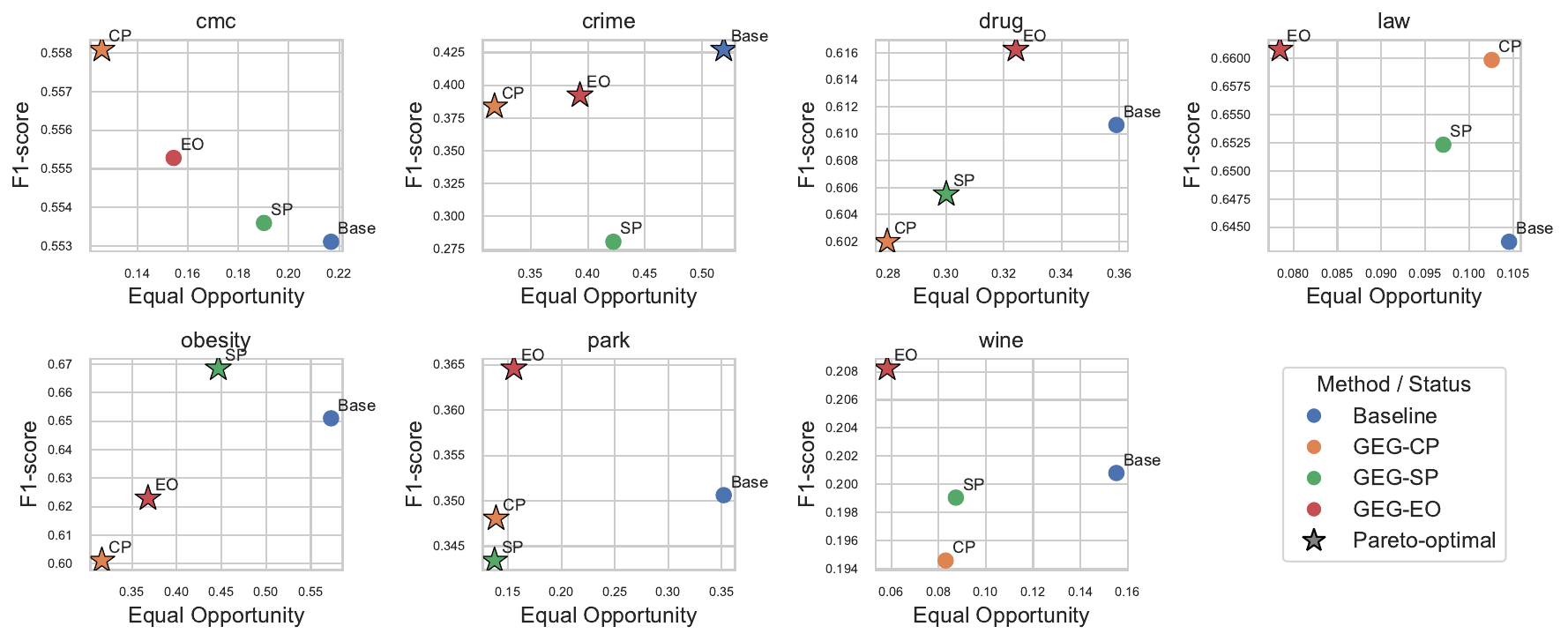}
        \caption{Equal Opportunity}
        \label{fig:pareto_eo}
    \end{subfigure}
    \begin{subfigure}{.9\textwidth}
        \includegraphics[width=\textwidth]{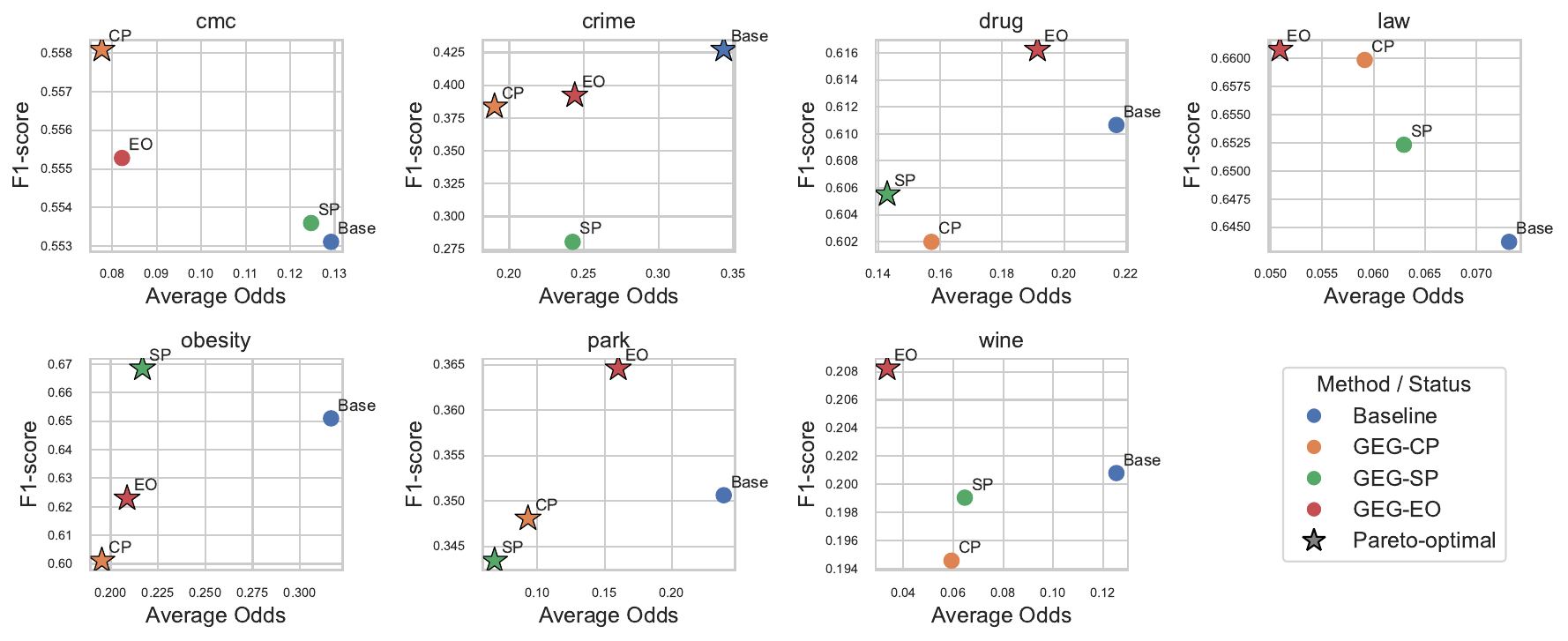}
        \caption{Average Odds}
        \label{fig:pareto_ao}
    \end{subfigure}
    \caption{RQ1: Performance of \tool and LR method for each dataset in the fairness/effectiveness space. Pareto optimal solutions are marked with a star.}
    \label{fig:rq1_pareto}
\end{figure*}

Figure \ref{fig:rq1_pareto} reports the performance of \tool and the base LR method for each dataset in terms of average F1 score and fairness metrics. In each plot, Pareto optimal solutions are marked with a star. The distribution of solutions confirms what has been highlighted in Figure \ref{fig:rq1_tradeoff}. Indeed, the base LR model appears in the Pareto front only for the \textit{Crime} dataset, while it is always dominated by \tool in the other datasets, regardless of the fairness definition considered. GEG-EO is the version of \tool that appears the most in the Pareto front, appearing 18 times in total, followed by GEG-CP (12 times), and GEG-SP (11 times). This result is in line with what has been observed above, where GEG-EO emerged as the solution performing better in the fairness/effectiveness trade-off.

\begin{rqanswer}
    \textbf{Answer to RQ$_1$:} \tool significantly improves the fairness of an LR classifier under multiple multi-class datasets and fairness definitions. The improvement in fairness achieved by \tool does not come with a high cost in effectiveness. Indeed, the trade-off between fairness and effectiveness is first overcome by the baseline when the effectiveness relevance $\alpha$ is around 0.8, and at least one \tool version Pareto dominates the baseline in 90.5\% of the cases analysed.
\end{rqanswer}

\subsection{RQ$_2$: Binary Classification}

\subsubsection{Single Scores}

\begin{figure*}[ht!]
    \centering
    \includegraphics[width=.8\linewidth]{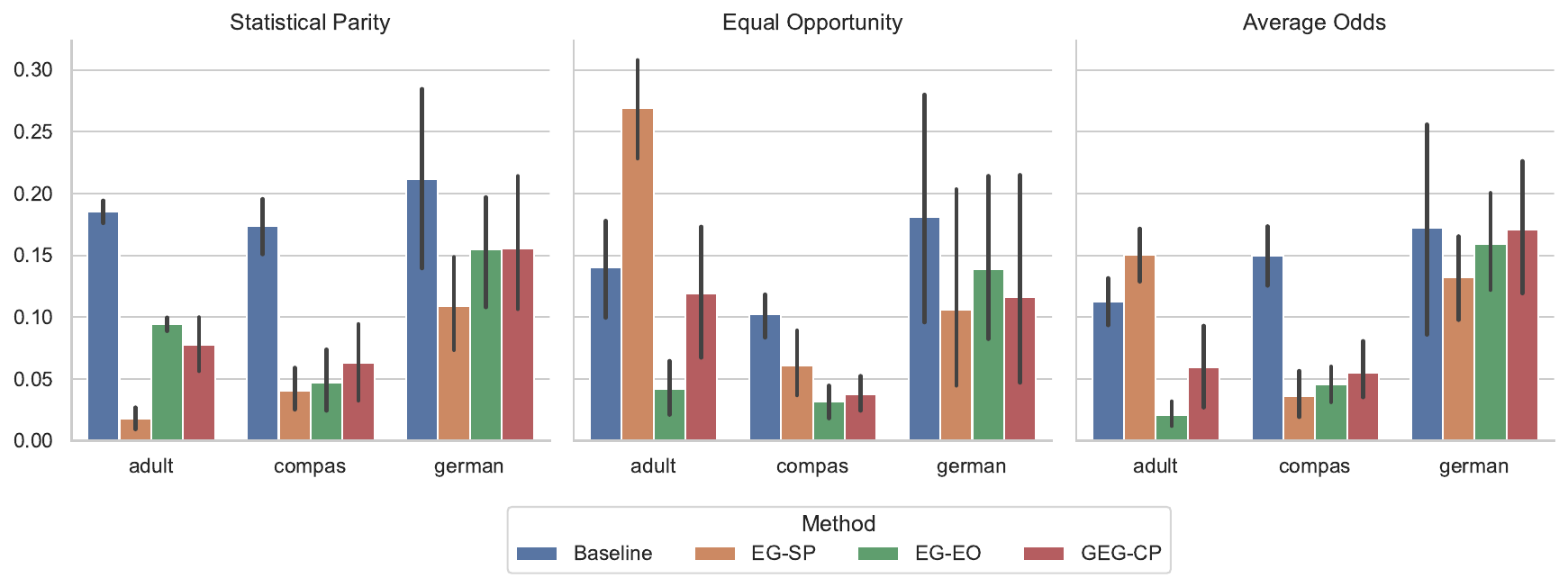}
    \caption{RQ2: Fairness scores comparison for binary classification datasets}
    \label{fig:rq2_fairness}
\end{figure*}

Figure \ref{fig:rq2_fairness} reports the fairness scores achieved by GEG-CP against the baseline approaches. We recall how, in the context of this RQ, EG-SP and EG-EO refer to the original implementation from Agarwal et al. \cite{agarwal_reductions_2018}. By looking at the single fairness scores, we observe how tailored versions of the EG algorithm perform better for the specific fairness definition they optimise (i.e., EG-SP performs better under the SPD definition, while EG-EO performs better under the EOD definition). This result is expected, since GEG-CP optimises the LR model for both SPD and EOD simultaneously, thereby providing a trade-off between fairness scores. We also observe that EG-SP and EG-EO outperform GEG-CP in AOD reduction across 2 of the 3 analysed datasets. However, GEG-CP consistently reduces the baseline across all fairness definitions and datasets considered. On the contrary, we observe how EG-SP worsens the bias of the classifier under EO and AO definitions, considering the Adult dataset. 

\begin{figure*}[ht!]
    \centering
    \includegraphics[width=.8\linewidth]{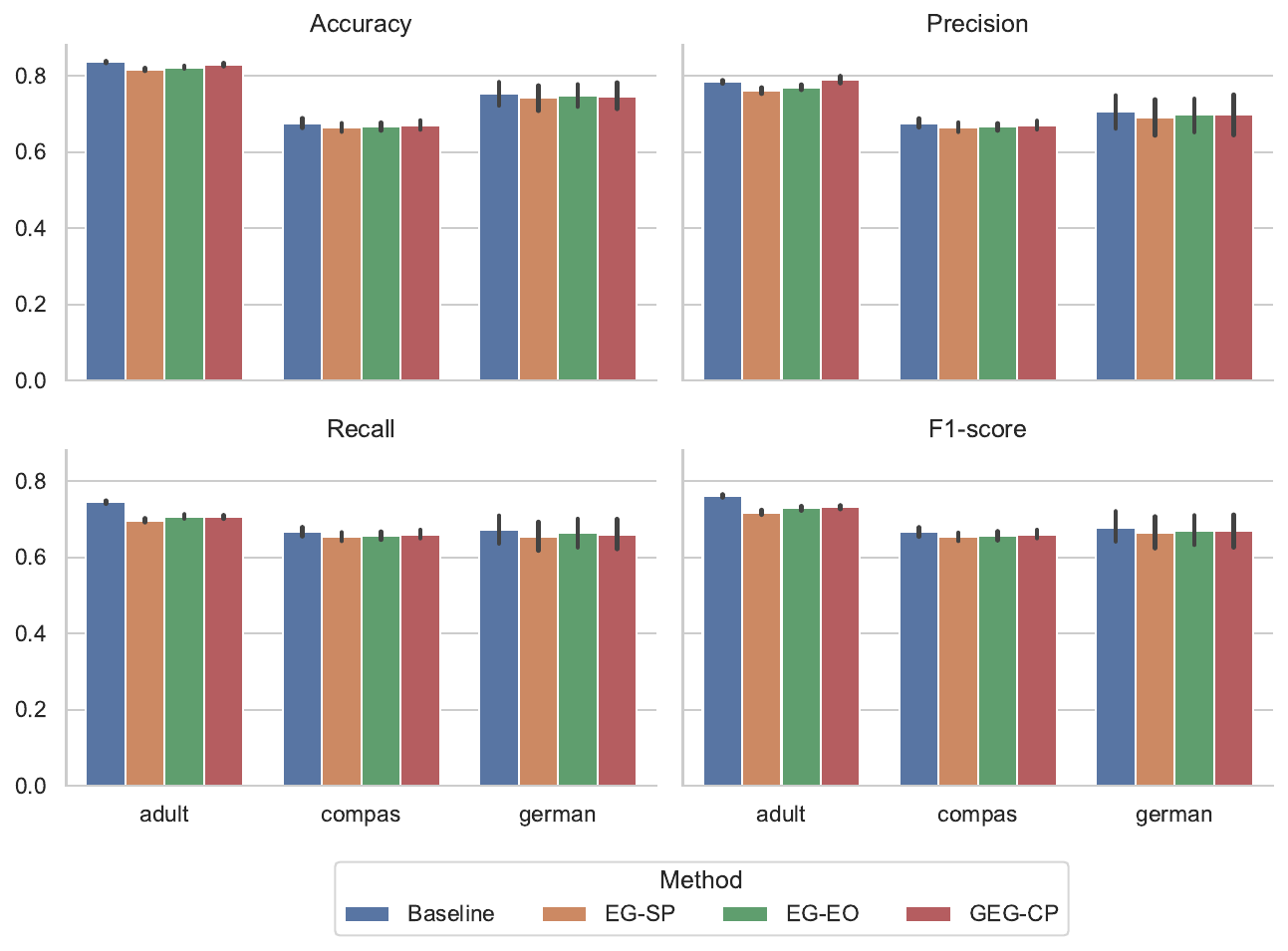}
    \caption{RQ2: Effectiveness scores comparison for binary classification datasets}
    \label{fig:rq2_effectiveness}
\end{figure*}

In terms of effectiveness, Figure \ref{fig:rq2_effectiveness} shows that GEG-CP achieves a consistent level of predictive performance compared to the baseline model and the other EG versions. 

\subsubsection{Trade-offs}

\begin{figure}[ht!]
    \centering
    \includegraphics[width=.8\linewidth]{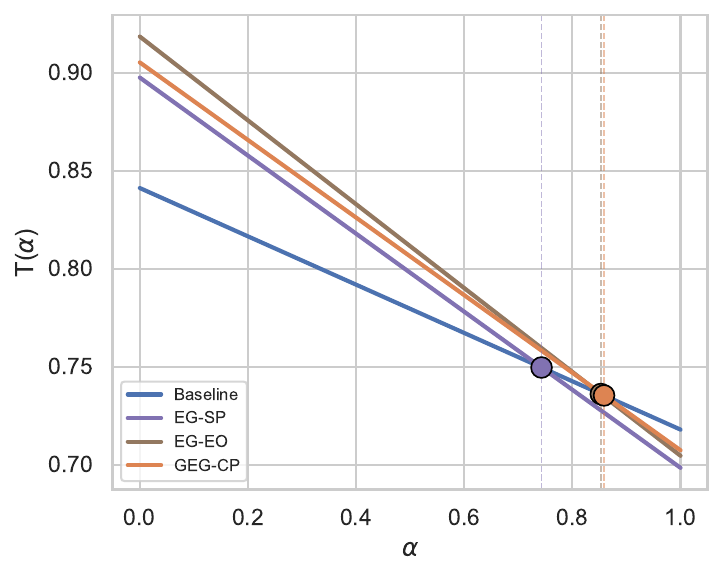}
    \caption{RQ2: Fairness-effectiveness trade-off score at different effectiveness weight $\alpha$}
    \label{fig:rq2_tradeoff}
\end{figure}

The trade-off score distribution in Figure \ref{fig:rq2_tradeoff} highlights how GEG-CP outperforms the baseline EG-SP, while it provides a comparable trade-off with EG-EO, where the $\alpha$ value that intersects the baseline curve is around 0.8. Notably, we observe that when $\alpha=0$ (no account for effectiveness), EG-EO provides a higher bias reduction. However, the slope of the EG-EO curve is steeper, indicating a more pronounced fairness-effectiveness trade-off for this algorithm. Finally, we observe how, when $\alpha=1.0$, i.e. no contribution of fairness, the score obtained by the baseline is higher than that obtained by each of the EG variants. This result indicates that the EG algorithm has a greater impact on effectiveness alone than the results for multi-class classification observed in Figure \ref{fig:rq1_tradeoff}.

\begin{figure*}[ht!]
    \centering
    \begin{subfigure}{\textwidth}
        \includegraphics[width=\textwidth]{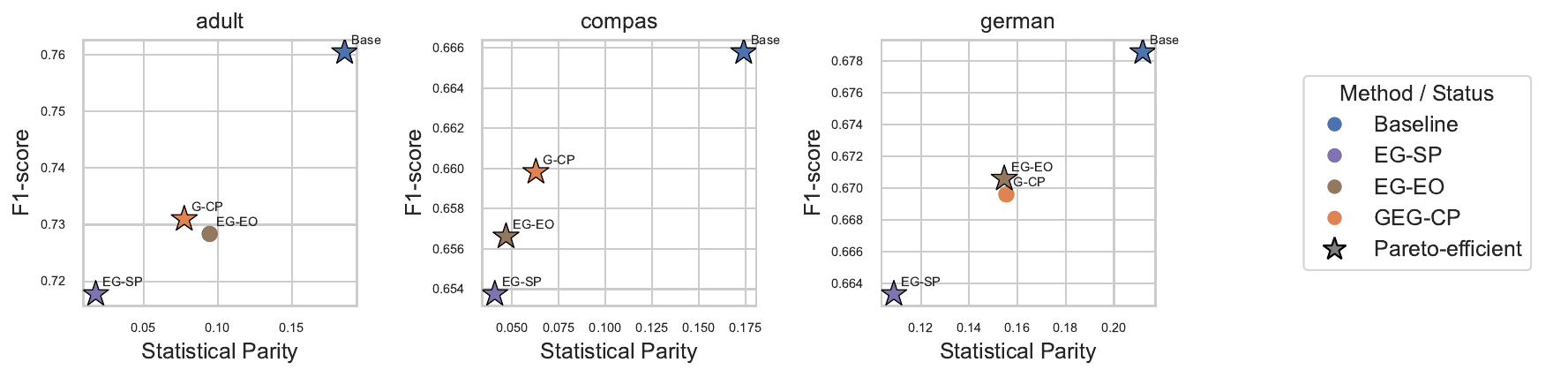}
        \caption{Statistical Parity}
    \end{subfigure}
    \begin{subfigure}{\textwidth}
        \includegraphics[width=\textwidth]{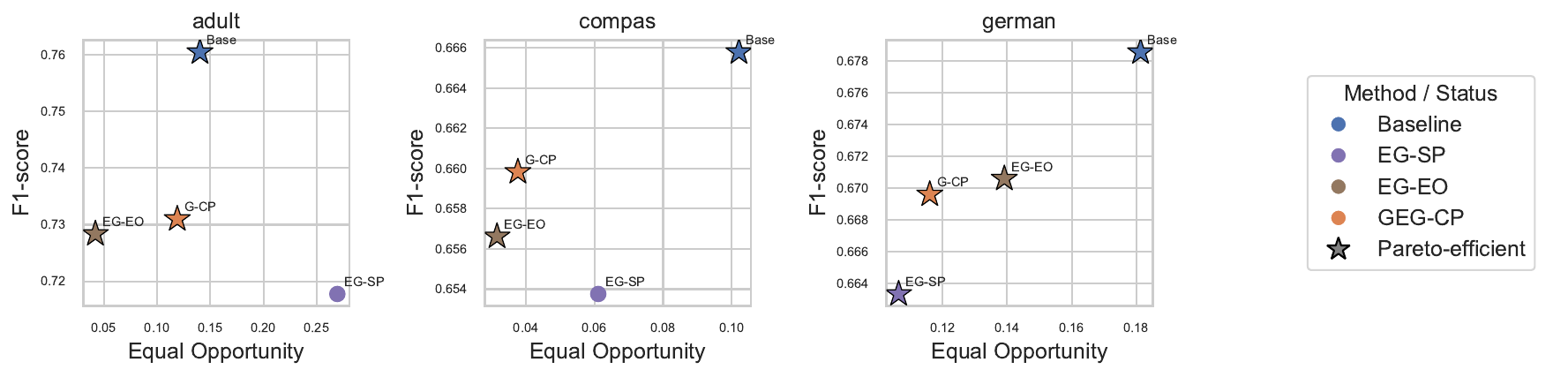}
        \caption{Equal Opportunity}
    \end{subfigure}
    \begin{subfigure}{\textwidth}
        \includegraphics[width=\textwidth]{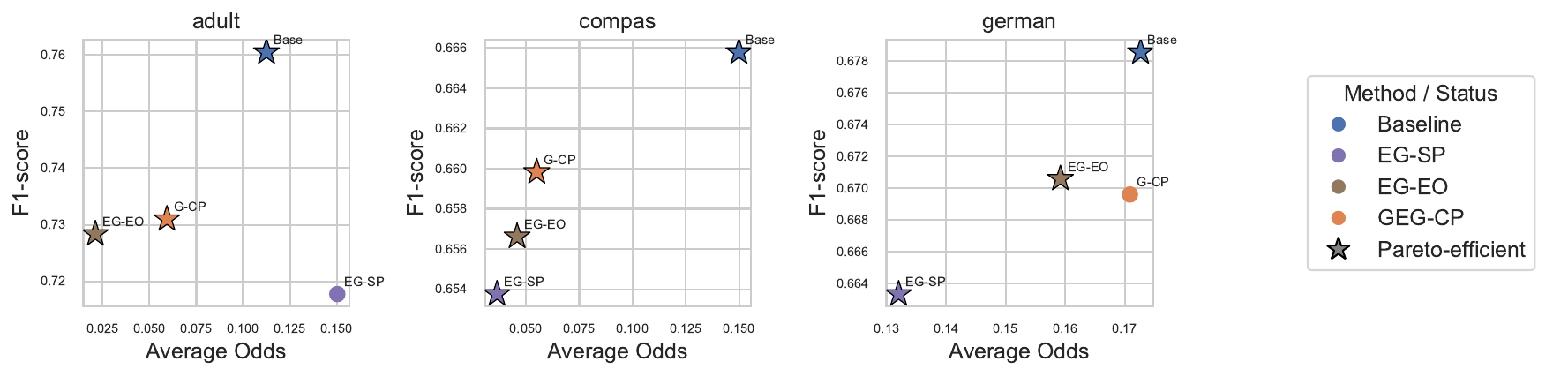}
        \caption{Average Odds}
    \end{subfigure}
    \caption{RQ2: Performance of \tool and LR method for each binary dataset in the fairness/effectiveness space. Pareto optimal solutions are marked with a star.}
    \label{fig:rq2_pareto}
\end{figure*}

The Pareto optimality analysis reported in Figure \ref{fig:rq2_pareto} confirms the trends observed in Figure \ref{fig:rq2_tradeoff}. Specifically, we observe that the baseline LR model is always present on the Pareto fronts, given its higher effectiveness. GEG-CP appears in 7 out of 9 Pareto fronts, while EG-EO appears in 8 out of 9 fronts. EG-SP is instead the approach that is mostly dominated by other approaches, appearing in 6 out of 9 fronts.

\begin{rqanswer}
    \textbf{Answer to RQ$_2$:} \revised{In the binary classification context, GEG-CP achieves a consistent fairness improvement under all fairness definitions analysed. The impact on effectiveness is comparable to that of EG-EO, which, however, may not perform consistently across all fairness definitions.} 
\end{rqanswer}

\subsection{RQ$_3$: Bias Mitigation Methods Comparison}

\subsubsection{Single Scores}

\begin{figure*}[ht!]
    \centering
    \includegraphics[width=.8\linewidth]{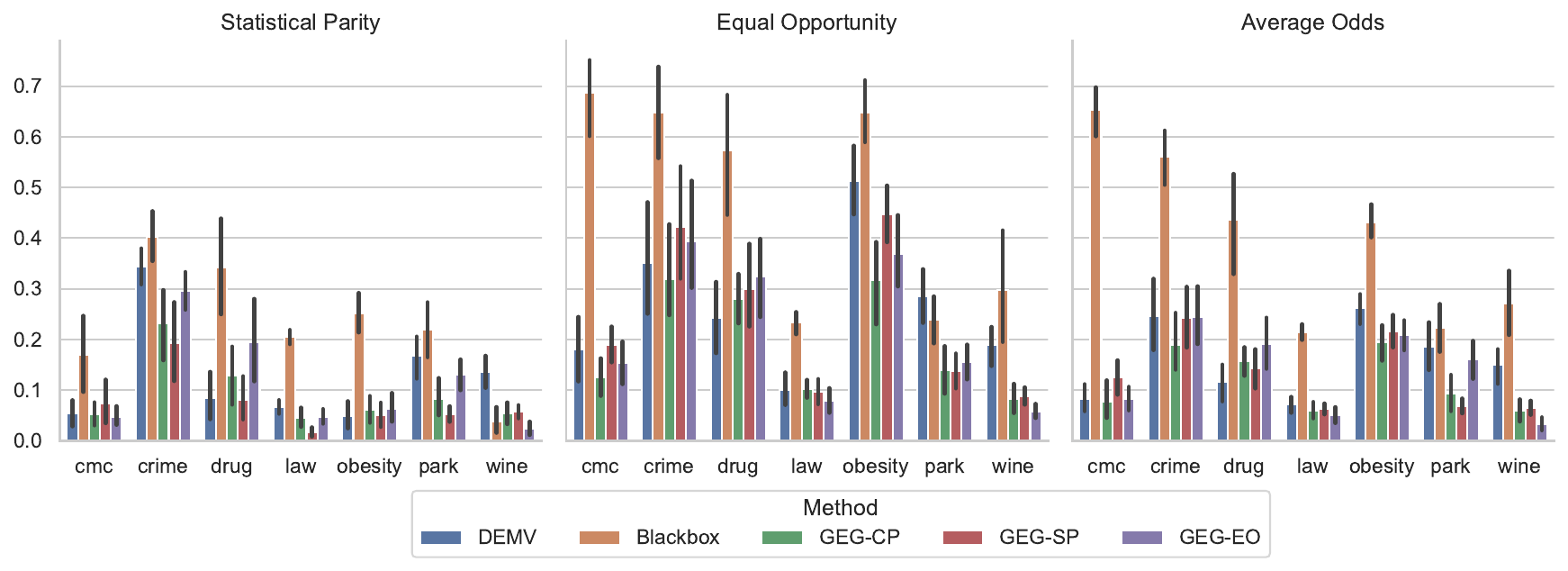}
    \caption{RQ3:Fairness scores comparison between DEMV, Blackbox and \tool variants}
    \label{fig:rq3_fairness}
\end{figure*}

Figure \ref{fig:rq3_fairness} reports the fairness scores obtained by the three \tool versions and the Blackbox and DEMV baseline approaches. From the plot, we observe how the Blackbox approach constantly provides worse bias scores than \tool. On the contrary, DEMV tend to behave consistently with \tool. Specifically, \tool provides statistically better fairness scores than DEMV with a large effect size under at least one fairness definition in 5 out of 9 datasets considered, while all fairness scores are never significantly worse than those obtained by DEMV (see Table \ref{tab:rq3_DEMV}).  

\begin{figure*}[ht!]
    \centering
    \includegraphics[width=.7\linewidth]{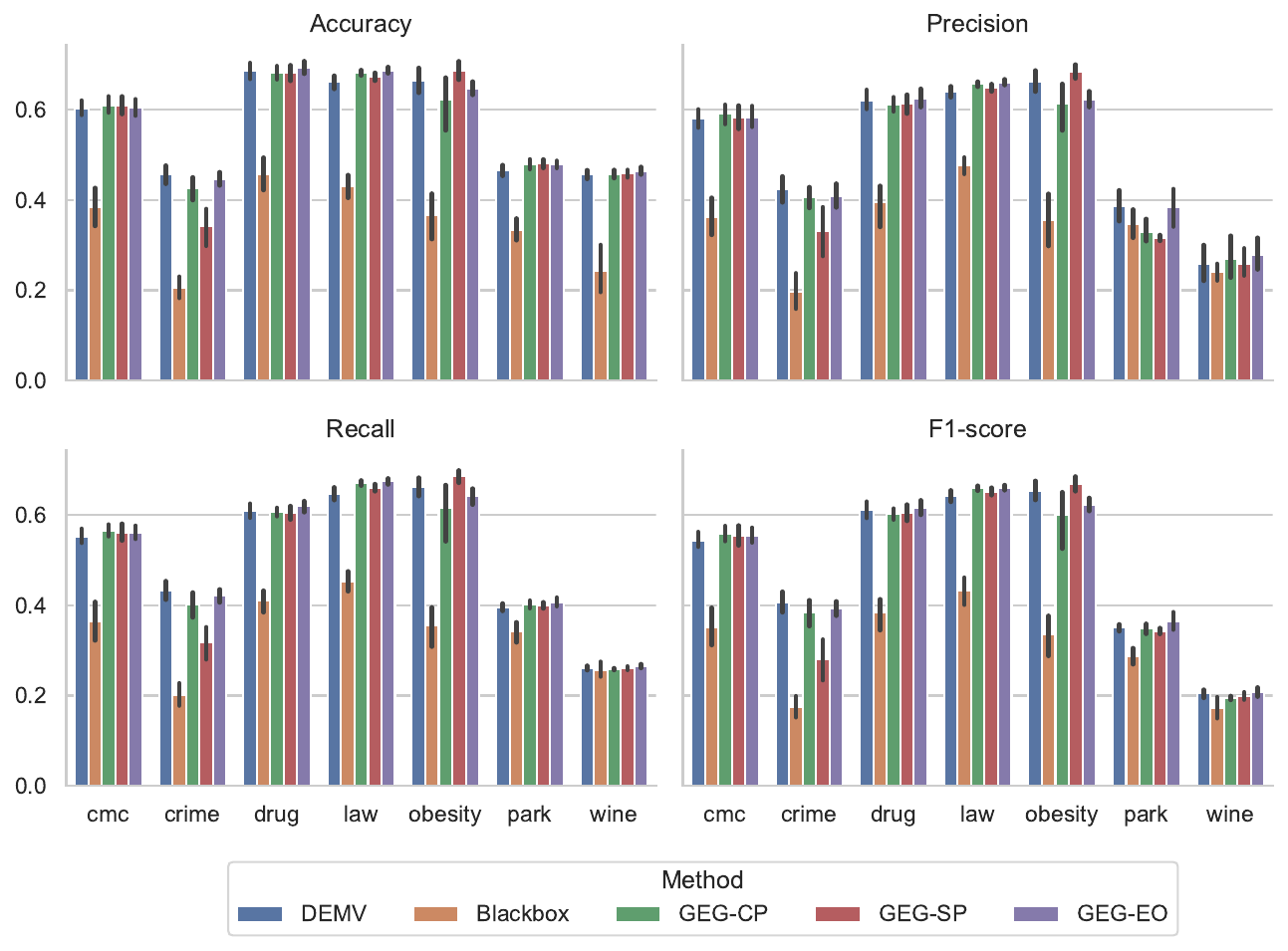}
    \caption{RQ3:Effectiveness scores comparison between DEMV, Blackbox and \tool variants}
    \label{fig:rq3_effectiveness}
\end{figure*}

In terms of effectiveness, as reported in Figure \ref{fig:rq3_effectiveness}, all \tool versions outperform the effectiveness achieved by the Blackbox algorithm. Conversely, DEMV provides a more consistent effectiveness among the different datasets, especially compared to GEG-SP and GEG-CP. GEG-EO instead provides comparable effectiveness to DEMV.

\subsubsection{Trade-offs}

\begin{figure}[ht!]
    \centering
    \includegraphics[width=.8\linewidth]{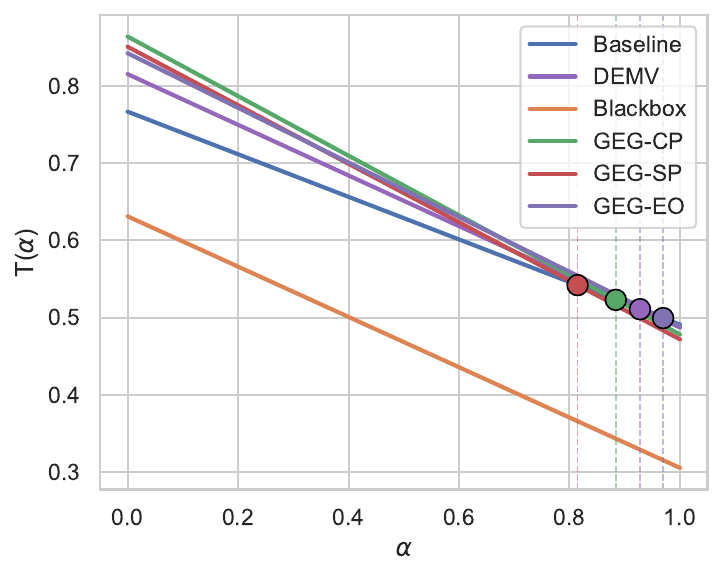}
    \caption{RQ3: Fairness-effectiveness trade-off score at different effectiveness weight $\alpha$}
    \label{fig:rq3_tradeoff}
\end{figure}

Figure \ref{fig:rq3_tradeoff} shows how the Blackbox approach underperforms even the baseline LR model in the fairness/effectiveness trade-off. Conversely, DEMV provides a better fairness/effectiveness trade-off than GEG-SP and GEG-CP, but is outperformed by GEG-EO. By observing the slopes of the curves, DEMV shows a lower overall bias reduction than all \tool versions. However, the bias reduction achieved, especially by GEG-SP and GEG-EO, implies a higher reduction in effectiveness, as well.

\begin{figure*}[ht!]
    \centering
    \begin{subfigure}{\textwidth}
        \includegraphics[width=\textwidth]{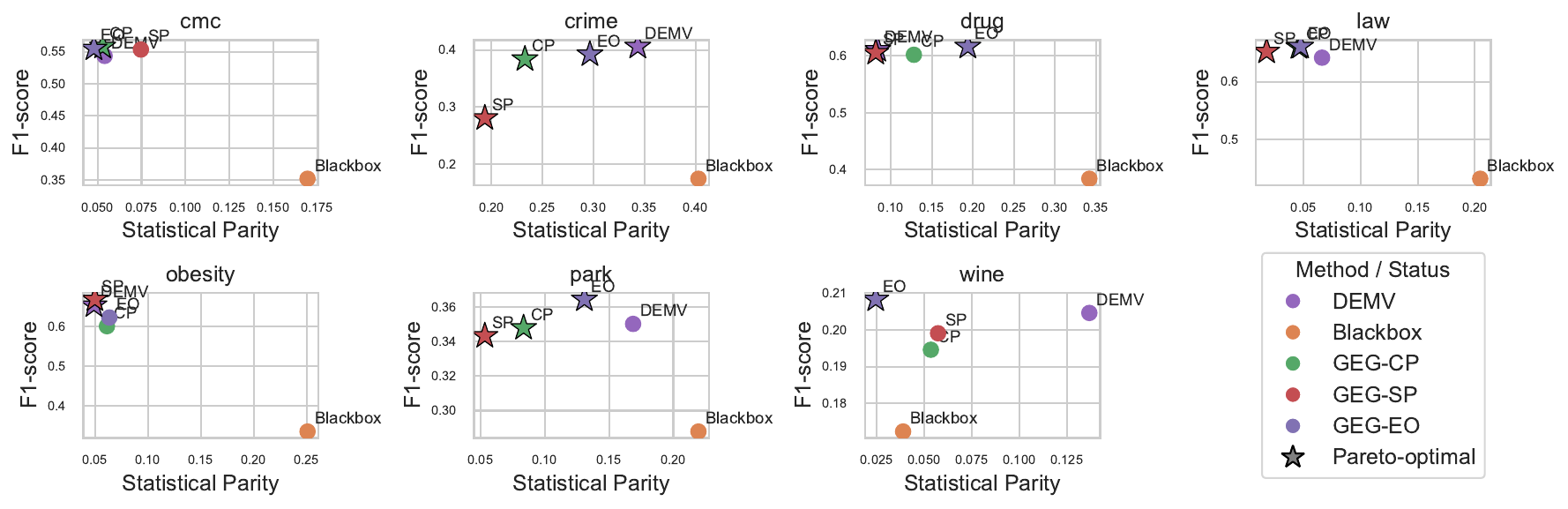}
        \caption{Statistical Parity}
    \end{subfigure}
    \begin{subfigure}{\textwidth}
        \includegraphics[width=\textwidth]{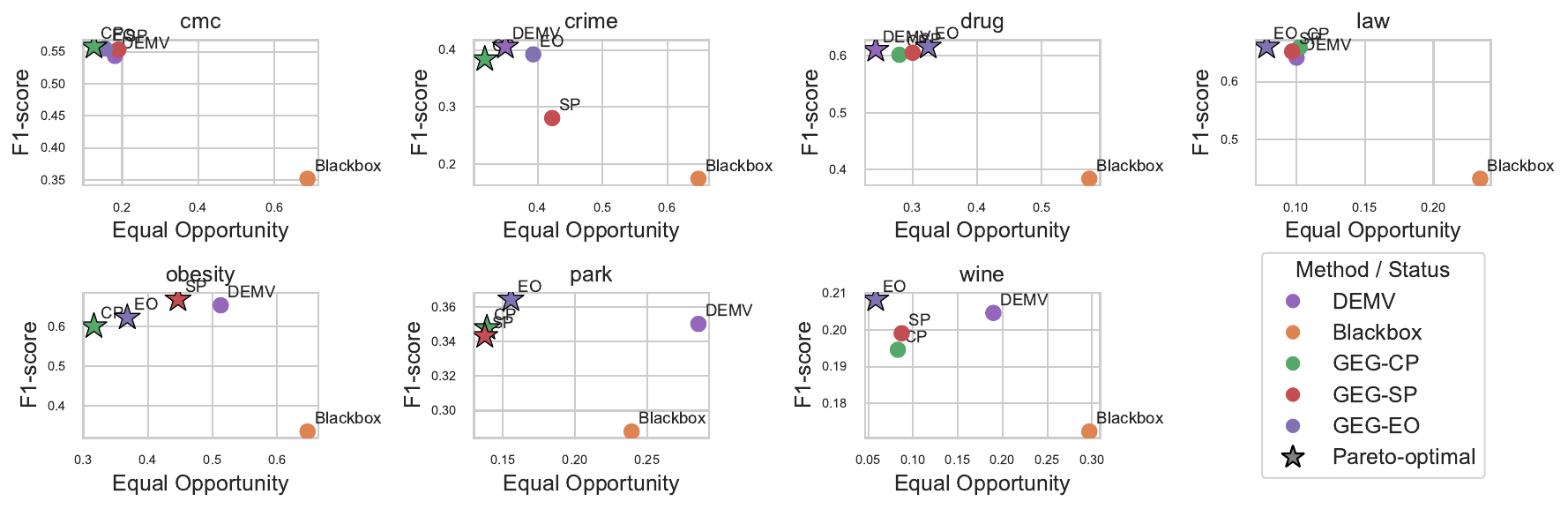}
        \caption{Equal Opportunity}
    \end{subfigure}
    \begin{subfigure}{\textwidth}
        \includegraphics[width=\textwidth]{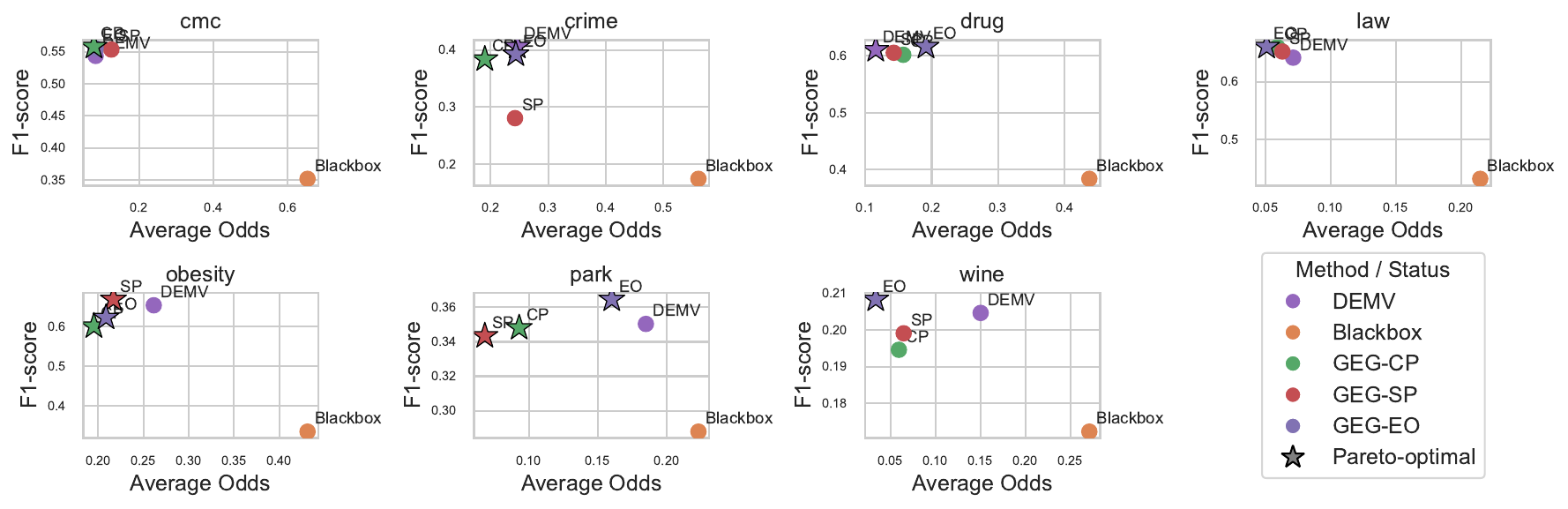}
        \caption{Average Odds}
    \end{subfigure}
    \caption{RQ3: Performance of Blackbox, DEMV, and \tool methods for each dataset in the fairness/effectiveness space. Pareto optimal solutions are marked with a star.}
    \label{fig:rq3_pareto}
\end{figure*}

However, the Pareto optimality results in Figure \ref{fig:rq3_pareto} show that DEMV is dominated by at least one version \tool in most of the datasets and metrics considered (DEMV is dominated in 66.67\% of the cases analysed). Therefore, while DEMV may provide more significant improvements than GEG-SP and GEG-CP in some cases (for instance, in the \textit{Crime} or \textit{Drug} datasets), solutions provided by \tool tend to be more general across datasets and fairness definitions.

\begin{rqanswer}
    \textbf{Answer to RQ$_3$:} All \tool versions overcome Blackbox in terms of Pareto optimality under all effectiveness and fairness metrics combinations. Concerning DEMV, \tool versions outperform it across most datasets and fairness metric combinations, although DEMV provides a better overall trade-off score than GEG-SP and GEG-CP.
\end{rqanswer}

\subsection{RQ$_4$: Sensitivity Analysis}

In the following, we discuss the results of our sensitivity analysis. Note how, for this RQ, to have an overall understanding of \tool behaviour under different hyperparameter values, following previous work \cite{daloisio_debiaser_2023}, we report the variation of the harmonic mean (H-Mean) among all fairness and effectiveness metrics. Metrics whose optimal value is 0 are transformed using $f(x)=1-x$, such that the closer the H-Mean is to 1, the better the model is in terms of fairness and effectiveness.

\begin{figure*}[ht!]
    \centering
    \includegraphics[width=.8\linewidth]{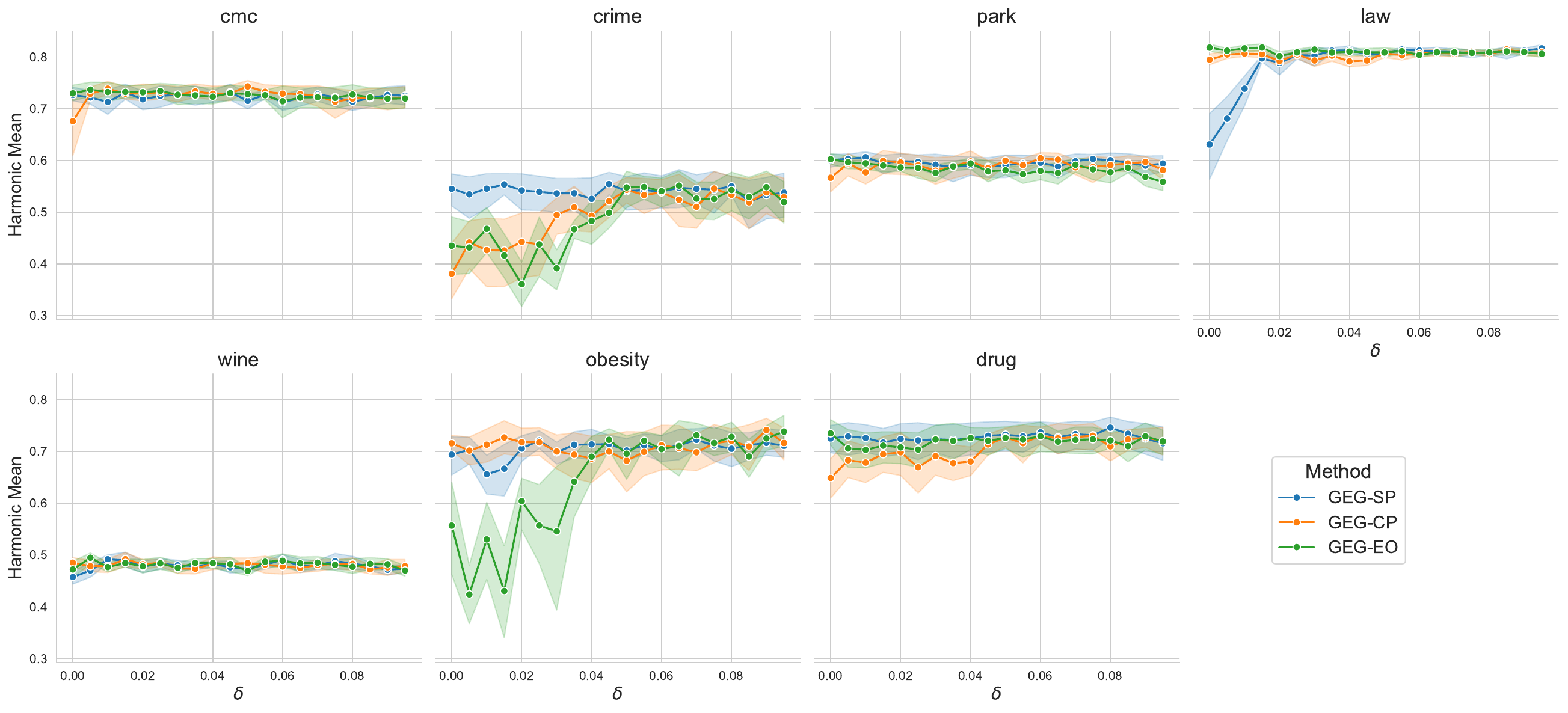}
    \caption{RQ4: Variation of H-Mean scores at different $\delta$ values}
    \label{fig:rq4_delta}
\end{figure*}

Figure \ref{fig:rq4_delta} reports the variation in H-Mean under different $\delta$ scores for each dataset analysed. We observe how all versions of \tool are stable with the majority of datasets considered. The only exceptions are observed in Obesity and Crime, where a $\delta < 0.05$ causes a lower H-Mean, especially with GEG-EO. A possible explanation is that a lower tolerance to bias for these datasets (which are those with the highest number of classes, as reported in Table \ref{tab:dataset}) implies a decrease in the model's effectiveness, especially under the Precision score. A similar pattern is also observed in GEG-SP with the Law dataset. However, all \tool versions become stable in all datasets with $\delta > 0.005$.

\begin{figure*}[ht!]
    \centering
    \includegraphics[width=.8\linewidth]{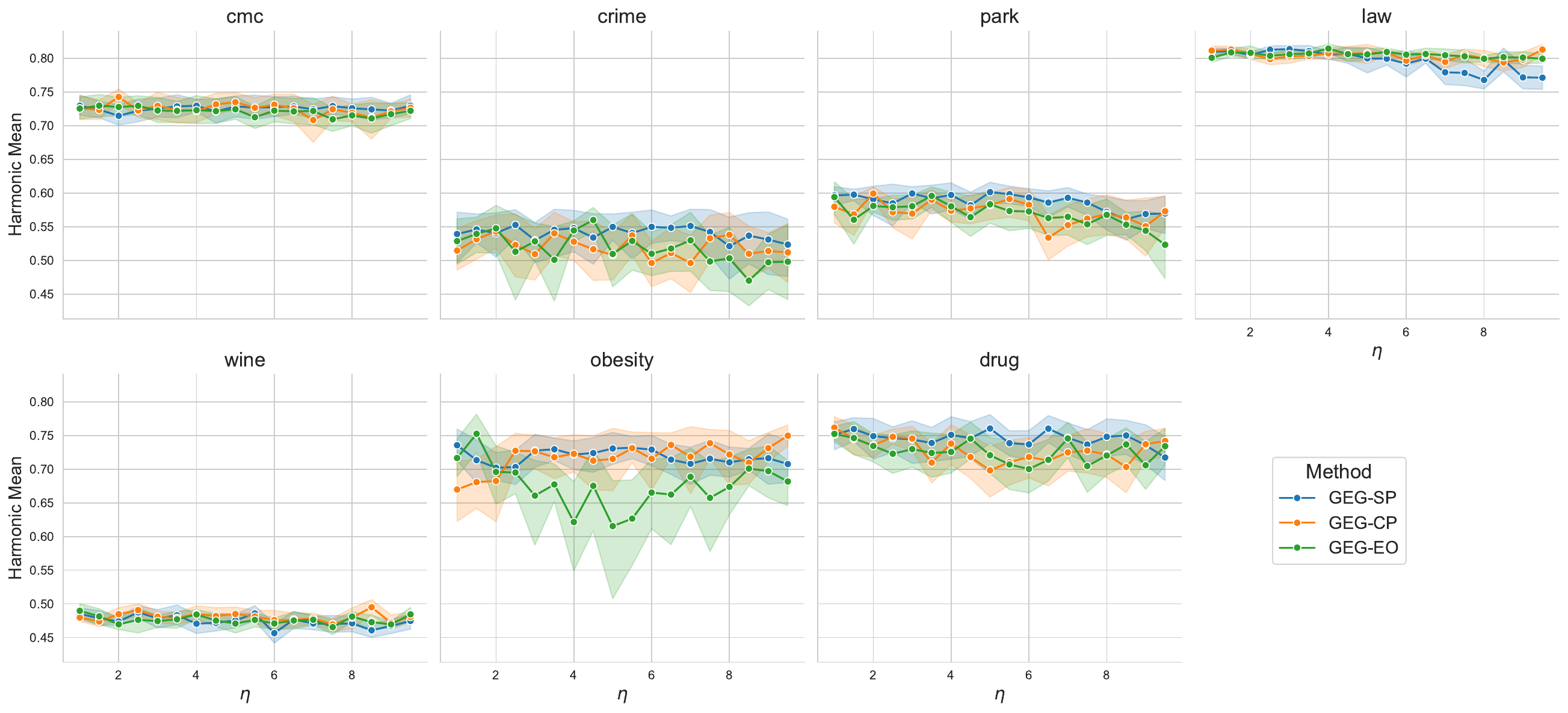}
    \caption{RQ4: Variation of H-Mean scores at different $\eta$ values}
    \label{fig:rq4_eta}
\end{figure*}

Figure \ref{fig:rq4_eta} reports instead the H-Mean variation at different $\eta$ values. We observe a similar pattern compared to the Figure \ref{fig:rq4_delta}, with the highest variability in the Crime and Obesity datasets. This variation could still be explained by the high number of classes in these datasets, where a variation in the learning rate could more significantly influence the model's effectiveness. Nevertheless, the observed variations remain contained, with no large deviations. 

\begin{rqanswer}
    \textbf{Answer to RQ$_4$:} \tool is consistent with $\delta$ and $\eta$ variations, with the largest effects, but still contained, observed in the datasets with the highest number of classes.
\end{rqanswer}

\subsection{RQ$_5$: Different Base Classifiers}

In the following, we report the results obtained by benchmarking the three \tool versions on more complex classifiers. For ease of understanding, we discuss here only the results of the Pareto optimal analysis, while we report the detailed single-score results in \ref{app:detailed}.

\begin{figure}[ht!]
    \centering
    \includegraphics[width=\linewidth]{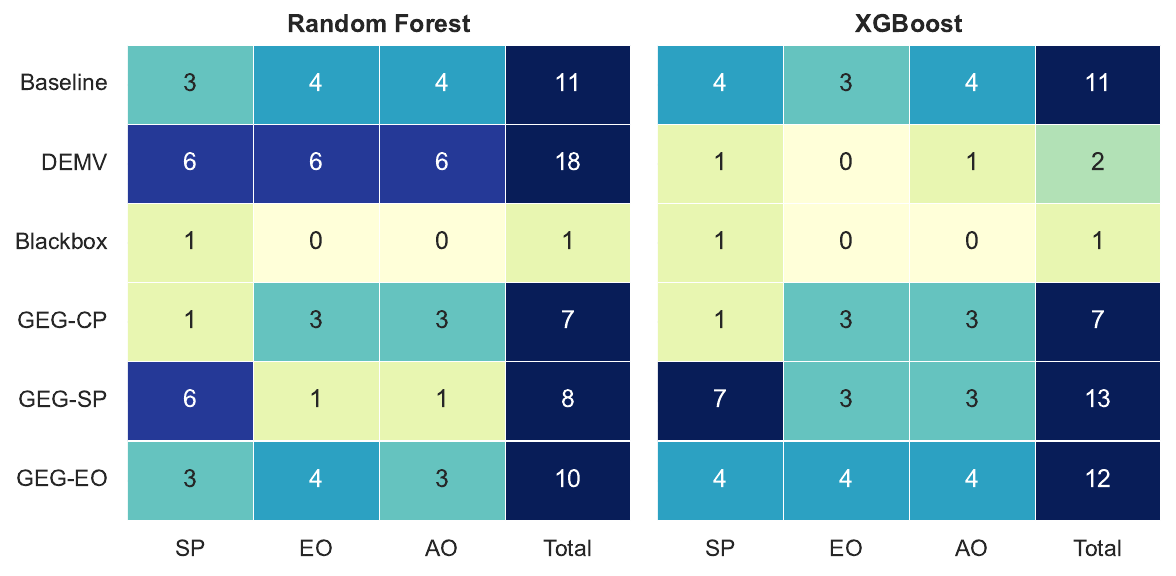}
    \caption{RQ5: Number of Pareto optimal solutions obtained by each approach considering the combination of F1 score and fairness metrics}
    \label{fig:rq5_pareto}
\end{figure}

Figure \ref{fig:rq5_pareto} shows the number of Pareto optimal solutions (obtained considering the combination of F1 score and fairness metrics) obtained by \tool versions and the baselines considered using an RF and an XGB base classifiers. First, we note how both base classifiers achieve a satisfactory trade-off between fairness and effectiveness, as also highlighted by the results reported in Tables \ref{tab:rq4_rf} and \ref{tab:rq4_gb}. Indeed, both base classifiers yield a total of 11 Pareto-optimal solutions across all datasets.  

By examining the results of the individual models, we observe that DEMV emerges as the best approach under the RF base classifier. On the contrary, all \tool versions dominate DEMV in terms of Pareto optimality, with GEG-SP and GEG-EO emerging as the methods that provide the highest number of Pareto-optimal solutions.

This result could be explained by the nature of the RF and XGB classifiers, in which the gradient-driven updates of the XGB model allow the cost-sensitive Lagrangian signal to propagate continuously. On the contrary, the discrete optimisation of the RF model turns the Lagrangian updates into noise that may impact the effectiveness of the prediction more.

Still, all \tool versions yield a large number of Pareto-optimal solutions for both base models, making them a valuable option.

\begin{rqanswer}
    \textbf{Answer to RQ$_5$:} \tool is effective in bias mitigation even when more complex base classifiers are employed. Especially with the XGB base classifier, GEG-SP and GEG-EO yield the most Pareto-optimal solutions across all datasets.
\end{rqanswer}

\section{Discussion}\label{sec:discussion}

In the following, we discuss the theoretical and practical implications derived from our study.

\subsection{Theoretical Implications}

The empirical evaluation has demonstrated that \tool is a valuable solution to improve fairness in multi-class classification contexts. Specifically, the results related to \textbf{RQ$_2$} revealed that incorporating the \textit{Combined Parity} constraint leads to a more consistent reduction of bias across all fairness definitions in binary classification tasks when compared to the original EG-SP and EG-EO approaches.


\revised{Moreover, results achieved by GEG-CP demonstrate how it can simultaneously improve SPD and EOD fairness definitions, compared to base classifiers and baseline bias mitigation methods. 
The reason why \tool achieves this result may be due to the constrained optimisation formulation, where SPD and EOD are optimised under a given tolerance value $\delta$ \cite{corbett2017algorithmic}. Nevertheless, we observe that SPD and EOD values obtained by GEG-CP are generally slightly worse than those obtained by the specific GEG-SP and GEG-EO approaches. This confirms a systematic trade-off when the algorithm optimises multiple definitions simultaneously.} Finally, the sensitivity analysis in the answer to \textbf{RQ$_4$}, demonstrated how GEG-CP is consistent under multiple $\delta$ bounds. 

These contributions demonstrate how \tool is a novel contribution in the field of bias mitigation in binary and multi-class classification tasks.

\subsection{Practical Insights}

From our empirical analysis, we can draw the following main practical insights and recommendations on using \tool:

\begin{itemize}
    
    \item \tool is effective in bias mitigation for multi-class classification regardless of the base classifier employed. 
    
    \item When employing an LR classifier for multi-class classification tasks, adopting \tool in use cases where the number of classes to predict is $\leq 4$ can also increase the effectiveness of the model. This outcome could be explained by the linear nature of the LR classifier, which better lends itself to Lagrangian optimisation.

    \item When adopting more complex classifiers such as RF or XGB for multi-class classification, \tool is still effective in bias mitigation, but it may decrease the prediction's effectiveness. More specifically, if the user can work on the training data, we suggest adopting DEMV to mitigate the bias of the RF classifier. Instead, when adopting XGB, \tool is the best solution to employ.


    \item Concerning binary classification, users can employ GEG-CP in use cases where higher fairness is more relevant than having more positive outcomes predicted \revised{(e.g., use cases protected by specific regulations or ML-models employed in sensitive scenarios~\cite{noauthor_eu_2023})}.

    \item We suggest adopting \tool instead of the pre-processing DEMV \revised{and Blackbox post-processing} approaches to achieve higher fairness with an LR base classifier. Additionally, when applied to datasets with low class imbalance, \tool can achieve higher prediction effectiveness than DEMV.

    \item \revised{As expected, we observed an increase in the training time of ML models while using \tool.\footnote{\revised{On average, on a MacBook Pro M3 with 18GB of RAM, the training time of an LR model with GEG-SP and GEG-EO is $\sim$10 seconds, while with GEG-CP is $\sim$6 seconds, compared to $\sim$2 seconds of the base LR model.}} This increase in the training time is due to the additional constraints imposed on the model. However, since \tool works on the training behaviour of the model, we did not observe an increase in the average inference time of a model trained with \tool.}

    \item \revised{Finally, \tool fills the existing lack of in-processing bias mitigation methods for multi-class classification. Therefore, it could be applied in all use cases where a user cannot change the training dataset but can work on the training behaviour of a model (e.g., regulations protected data \cite{austin2016will}), and it can also be chained with additional post-processing methods in later phases of the development workflow \cite{bellamy_ai_2019}.}

\end{itemize}

%% file: sections/conclusion.tex
In this paper, we addressed the topic of bias mitigation in multi-class classification settings. We first formulate the problem of fair multi-class learning as a multi-objective optimisation problem under multiple linear fairness constraints. Next, we propose \tool, an in-processing bias mitigation method to solve this task. In particular, \tool extends the EG approach from Agarwal et al.~\cite{agarwal_reductions_2018} to the multi-class classification setting. In addition, \tool allows the optimisation of a classifier under multiple fairness constraints simultaneously. 
We perform an extensive evaluation of \tool against six baseline approaches across seven multi-class and three binary datasets, using four effectiveness metrics and three fairness definitions. Our evaluation shows that \tool is successful at mitigating bias without severely impacting the effectiveness of the predictions. Additionally, we draw a set of practical insights for practitioners on using \tool in real-world scenarios.

Future work can extend \tool by including additional \revised{more general fairness constraints in the optimisation process}. Additionally, \tool can be extended to address intersectional fairness scenarios, i.e., where sensitive groups are identified by the combination of two or more sensitive variables~\cite{chen_fairness_2024}. 